\documentclass[twoside,11pt]{article}

\usepackage{blindtext}
\usepackage{jmlr2e}

\usepackage{titletoc}

\usepackage{mathtools,amsfonts,amssymb,dsfont,xcolor}
\usepackage{amsmath}
\usepackage{url}
\usepackage{hyperref}
\hypersetup{
    colorlinks=false,  
    pdfborder={0 0 0}  
}
\usepackage{thm-restate}
\usepackage[sort,compress]{cite}
\usepackage{enumitem}

\DeclareMathOperator{\spand}{sp}

\newcommand{\SD}{\mathrm{SD}}
\newcommand{\AC}{\mathrm{AR}}
\newcommand{\FD}{\mathrm{FD}}

\newcommand{\toup}{%
  \mathrel{\nonscript\mkern-1.2mu\mkern1.2mu{\uparrow}}%
}
\newcommand\OO{\mathcal O}

\newcommand\tildeO{\tilde{\OO}}
\newcommand\reals{\mathds{R}}
\newcommand\integers{\mathds{N}}
\newcommand\EXP{\mathds{E}}
\newcommand\PR{\mathds{P}}

\newtheorem{assumption}{Assumption}

\makeatletter
\newcommand*\doTRANS[2]{\raisebox{\depth}{$\m@th#1\intercal$}}
\makeatother

\usepackage{lastpage}

\firstpageno{1}

\usepackage{lastpage}
\ShortHeadings{Concentration of Reward in MDPs}{Sayedana, Caines, and Mahajan}

\makeatletter

\newcommand{\startappendixtoc}{%
  \let\orig@addcontentsline\addcontentsline
  \renewcommand{\addcontentsline}[3]{%
    \orig@addcontentsline{##1}{##2}{##3}%
    \ifnum\pdfstrcmp{##1}{toc}=0 %
      \ifnum\pdfstrcmp{##2}{section}=0 %
        \orig@addcontentsline{atoc}{##2}{##3}%
      \fi
      \ifnum\pdfstrcmp{##2}{subsection}=0 %
        \orig@addcontentsline{atoc}{##2}{##3}%
      \fi
      \ifnum\pdfstrcmp{##2}{subsubsection}=0 %
        \orig@addcontentsline{atoc}{##2}{##3}%
      \fi
    \fi
  }%
}

\newcommand{\stopappendixtoc}{%
  \let\addcontentsline\orig@addcontentsline
}

\newcommand{\listofappendices}{%
  \section*{Appendix Contents}%
  \@starttoc{atoc}%
}

\makeatother

\begin{document}

\title{Concentration of Cumulative Reward in Markov Decision Processes}

\author{\name Borna Sayedana \email         borna.sayedana@mail.mcgill.ca\\ 
    \addr  Université de Montréal,\\
    and Mila – Quebec AI Institute\\
    Montreal, QC, H2S\,3H1, Canada.
       \AND
       \name Peter E. Caines \email peterc@cim.mcgill.ca
       \\ \addr Department of Electrical and Computer Engineering,  \\
       McGill University,  \\Montreal, QC, H3A\,0E9, Canada.
       \AND
       \name Aditya Mahajan \email aditya.mahajan@mcgill.ca  \\
       \addr Department of Electrical and Computer Engineering,  \\
       McGill University,  \\Montreal, QC, H3A\,0E9, Canada.}

\editor{}

\maketitle

\begin{abstract}
   In this paper, we investigate the concentration properties of cumulative reward in Markov Decision Processes (MDPs), focusing on both asymptotic and non-asymptotic settings. We introduce a unified approach to characterize reward concentration in MDPs, covering both infinite-horizon settings (i.e., average and discounted reward frameworks) and finite-horizon setting.  Our asymptotic results include the law of large numbers, the central limit theorem, and the law of iterated logarithms, while our non-asymptotic bounds include Azuma-Hoeffding-type inequalities and a non-asymptotic version of the law of iterated logarithms. Additionally, we explore two key implications of our results. First, we analyze the sample path behavior of the difference in rewards between any two stationary policies. Second, we show that two alternative definitions of regret for learning policies proposed in the literature are \emph{rate-equivalent}. Our proof techniques rely on a martingale decomposition of cumulative reward, properties of the solution to the policy evaluation fixed-point equation, and both asymptotic and non-asymptotic concentration results for martingale difference sequences.\\
   \begin{keywords}
   Concentration of Rewards, Markov Decision Processes, Reinforcement Learning, Average Reward Infinite-Horizon MDPs
   \end{keywords}
\end{abstract}
\section{Introduction}
Reinforcement learning is a machine learning framework in which an agent learns to make optimal sequential decisions by repeatedly interacting with its environment. This approach is particularly effective for addressing problems with complex dynamic environments. The standard mathematical model for reinforcement learning is Markov Decision Processes (MDPs). In an MDP, the agent takes an action at each time step, receives an instantaneous reward, and transitions to the next state based on a Markovian dynamic that depends on the current state and action.

The existing literature on MDP theory primarily focuses on analyzing and maximizing the \emph{expected} cumulative reward, resulting in methods that emphasize the system's \emph{average} behavior. While this approach is useful in many domains, it may fall short in high-stakes applications, where an agents' decisions may lead to costly consequences. Such scenarios arise in applications such as safety-critical engineering systems and decision-making processes in finance and healthcare. Different approaches have emerged to understand the behavior of MDPs beyond the expected reward. Broadly, these approaches can be categorized as follows: (i) risk-sensitive control, in which the agent aims to identify policies that minimize a specific \emph{risk measure}; (ii) distributional reinforcement learning, in which the distribution of cumulative \emph{discounted} reward is estimated and controlled; (iii) Markov reward processes, in which the \emph{asymptotic} distributional and sample-path properties of the cumulative reward in the infinite-horizon \emph{average} reward framework are investigated. We elaborate on each of these approaches below.

\textbf{Risk-sensitive control.} In the risk-neutral framework the agent's goal is to maximize the expected cumulative reward. In the risk-sensitive framework, the agent's goal is to minimize a risk measure that captures other statistical properties (e.g., variance, tail probability, etc.) of the reward in addition to the mean. There are three primary risk functionals studied in the literature \citep{WangChapman2022RiskAverseAutonomous}: (i) the exponential utility functional, which is an increasing and convex mapping of the cost function. Under certain conditions, this functional can be used to model the mean-variance trade-off in decision-making problems. Risk-averse control using exponential utility functional has been studied in~\citet{HowardMatheson1972,Jacobson1973,Whittle1981,CoraluppiMarcus1999,Borkar2002,BauerleRieder2014}.  (ii) Quantile-based risk functionals such as Value at Risk (VaR) and Conditional Value at Risk (CVaR), which characterize the probability or expectation of the cost exceeds a given threshold, capturing the tail behavior of the cost distribution.  Risk-averse control using quantile-based risk functionals has been studied in~\citet{BauerleOtt2011,ChowTamarMannorPavone2015,MillerYang2021,BauerleGlauner2021,Chapman2022}. (iii) Recursive risk functionals which model the risk at every stage and result in a dynamic programming type solution. Risk-averse control using recursive risk functionals has been studied in~ \citet{ruszczynski2010risk,SinghChowMajumdarPavone2018,BauerleGlauner2022_EJOR}. For a comprehensive survey on risk-sensitive control and RL, please refer to~\citet{WangChapman2022RiskAverseAutonomous,BiswasBorkar2023ErgodicRiskSensitive}.

\textbf{Distributional RL.} The second approach focuses on estimating 
various statistical properties of the discounted cumulative reward in the infinite-horizon discounted reward framework. Early works such as \citet{sobel1982variance,ChungSobel1987} derive a Bellman-type equation to compute the variance of the discounted cumulative reward. More recent works treat the asymptotic discounted cumulative reward as a random variable and use various methods to approximate the distribution or compute its important statistics such as quantiles. The approximate distribution is then used in reinforcement learning algorithms~\citep{MorimuraEtAl2010,MorimuraEtAl2010a,bellemare2017distributional,rowland2018analysis,dabney2018implicit,dabney2018distributional,RowlandEtAl2019,BellemareEtAl2019, LyleCastroBellemare2019,YangEtAl2019,Farahmand2019,DuanEtAl2021,lheritier2021cram,NguyenEtAl2022,rowland2023statistical,rowland2024analysis}.
For a comprehensive review of the algorithms and theoretical results in distributional RL, please refer to~\citet{bdr2023}.

\textbf{Markov Chains.} For a fixed stationary Markov policy, any MDP may be reduced to a Markov reward process, i.e. a Markov chain induced on the state space and an associated reward process. As a result, there is a close connection between the MDP theory and the theory of Markov chains. The asymptotic sample path and distributional behavior of functionals of Markov chains are extensively studied in the literature. For example, the CLT results for Markov chains are established in  \citet{chung1967markov,cogburn1972central,maigret1978theoreme,niemi1982central,kipnis1986central,maxwell2000central,landim2003central,jones2004markov,meyn2012markov,duflo2013random,srikant2025rates}. The rate of convergence in CLT for geometrically ergodic Markov chains is studied in \citet{kontoyiannis2003spectral,kontoyiannis2005large}. In parallel, the asymptotic behavior of martingales are studied in the martingale theory, e.g., in \citet{neveu1975discrete,hall1980martingale,billingsley2013convergence}. As discussed in \citet{meyn2012markov} there is a close connection between the results in Markov chains and their counterparts in martingale theory. As an example, one way to prove the CLT for Markov chains is to use the martingale decomposition arising from the Poisson equation. This approach is used in e.g., \citet{Mandl1971,Mandl1991,hernandez2012further,duflo2013random,maigret1978theoreme}. Recently the martingale approach is used to derive a central limit theorem for the Linear Quadratic Regulation (LQR) problem in \citet{sayedana2024asymptotic}.

\textbf{Our Work.} In this paper, we provide a unified approach for characterizing both asymptotic and non-asymptotic reward concentration in infinite-horizon average reward, infinite-horizon discounted reward, and finite-horizon frameworks. Our results cover asymptotic concentration like LLN, CLT, and LIL, along with non-asymptotic bounds, including Azuma-Hoeffding-type inequalities and a non-asymptotic version of the Law of Iterated Logarithms for the average reward setting. Building upon these concentration results, we explore two of their key implications: (1) the sample path difference of rewards between two policies, and (2) the impact of these findings on the regret analysis of reinforcement learning algorithms. We derive similar non-asymptotic upper-bounds for discounted reward and finite-horizon setups.

\textbf{Comparison.} There are two key distinctions between our work and the existing literature. (i) We establish \emph{non-asymptotic} concentration results for the cumulative reward process. In contrast, studies in distributional RL typically analyze the \emph{asymptotic} behavior of the discounted cumulative reward, whereas in the average-reward and Markov reward process settings, prior works focus on the \emph{asymptotic} distributions of the cumulative reward process. To the best of our knowledge, no \emph{non-asymptotic} concentration result have been reported in the literature for MDPs. (ii) We develop a \emph{unified framework} that enables the derivation of concentration results across the three principle MDP frameworks. In contrast, methods developed in distributional RL that rely on contraction mapping theorems do not extend naturally to the average-reward setting, and techniques based on Markov chain analysis cannot be directly applied to finite-horizon problems.

\textbf{Proof Approach} Our proofs rely on a martingale decomposition similar to the one originating from the Poisson equation in Markov chains (e.g. in ~\citet{meyn2012markov}). Such decomposition enables us to interpret the cumulative reward process both as a martingale and as a functional of an underlying Markov chain. In this paper, we adopt the martingale viewpoint rather than the Markov chain one. We make such a choice since the resulting non-asymptotic bounds solely depend on the statistical properties of the value function. As a result, these bounds can be efficiently computed using existing numerical methods for value function computation.

We also use our results to clarify a nuance in the definition of regret in average reward infinite-horizon reinforcement learning. In this setting, regret is defined as the difference between the \emph{expected} reward obtained by the optimal policy minus the (sample-path) cumulative reward obtained by the learning algorithm as a function of time. The standard results establish that this regret is lower-bounded by $\Omega(\sqrt{D|\mathcal{S}||\mathcal{A}|T})$ and upper bounded by $\tildeO(D|\mathcal{S}|\sqrt{|\mathcal{A}|T})$ \citep{auer2008near}, where $T$ denotes the horizon, $|\mathcal{S}|$ denotes the number of states, $|\mathcal{A}|$ denotes the number of actions, and $D$ denotes the diameter of thr MDP. Various refinements of these results have been considered in the literature 
\citep{auer2006logarithmic, filippi2010optimism,bartlett2012regal,russo2014learning,osband2013more,lakshmanan2015improved,osband2016generalization,ouyang2017learning,theocharous2017posterior,agrawal2017optimistic,talebi2018variance,fruit2018efficient,zhang2019regret,qian2019exploration,fruit2019exploration, zanette2019tighter,fruit2020improved,bourel2020tightening,zhang2023sharper,boone2024achieving}.

There is a more appropriate notion of regret in applications which are driven by an independent exogenous noise process such as inventory management problems where the dynamics are driven by an exogenous demand process and linear quadratic regulation problems where the dynamics are driven by an exogenous disturbance process. In such applications, it is more appropriate to compare the cumulative reward obtained by the optimal policy with cumulative reward obtained by the learning algorithm \emph{under the same realization of the exogenous noise}. For example, in an inventory management problem, one may ask how worse is a learning algorithm compared to the (expected-reward) optimal policy on a specific realization of the demand process. This notion of regret has received significantly less attention in the literature \citep{abbasi2019politex,talebi2018variance}. We show that a consequence of our results is that the two notions of regret are rate-equivalent. A similar result was claimed without a proof in \citet{talebi2018variance}.

\subsection{Contributions}
The contributions of this paper can be summarized as follows:
\begin{enumerate}
    \item We establish the asymptotic concentration of cumulative reward in average reward MDPs, deriving the law of large numbers, the central limit theorem, and the law of iterated logarithm for a class of stationary policies. Compared to the existing asymptotic results in the literature which use Markov chain theory, we provide a simpler proof which leverages a martingale decomposition for the cumulative reward along with the asymptotic concentration of measures for martingale sequences.
    \item We derive policy-dependent and policy-independent non-asymptotic concentration bounds for the cumulative reward in average reward MDPs. These bounds establish an Azuma-Hoeffding-type inequality for the rewards along with a non-asymptotic version of law of iterated logarithm. Although these results apply to a broad subset of stationary policies, we show that for communicating MDPs, these bounds extend to any stationary deterministic policy. We use the established concentration results to characterize the sample path behavior of the performance difference of any two stationary policies. As a corollary of this result, we show that the difference between cumulative reward of any two optimal policies is upper-bounded by $\OO(\sqrt{T})$ with high probability.  
    \item We investigate the difference between two notions of regret in the reinforcement learning literature, cumulative regret and interim cumulative regret. By analyzing the sample path behavior, we establish that both asymptotically and non-asymptotically, this difference is upper-bounded by $\tildeO(\sqrt{T})$. This result implies that, if a reinforcement learning algorithm has a regret upper bound of $\tildeO(\sqrt{T})$ under one definition, the same rate applies to the other, in both of the asymptotic and non-asymptotic frameworks. While this equivalency was claimed in the literature without a proof, our concentration results provide a formal proof for this relation.  
    \item Lastly, we investigate several extensions of our results to other frameworks. In particular, we derive non-asymptotic concentration bounds for the cumulative reward in the infinite-horizon discounted reward and finite-horizon MDP frameworks. These bounds include an Azuma-Hoeffding-type inequality along with a non-asymptotic version of the law of iterated logarithm. Using the vanishing discount analysis, we show that under appropriate conditions, the concentration bounds for discounted reward MDPs approach to the concentration bounds for the average reward MDPs as the discount factor approaches 1. Moreover, we establish the non-asymptotic concentration bounds for models with stochastic reward, i.e., models in which reward is a function of an exogenous process in addition to state and action.
\end{enumerate}

\subsection{Organization}
The rest of this paper is organized as follows. The problem formulation, along with the underlying assumptions, are presented in Sec.~\ref{sec:Problem_formulation}. The main results for the average reward setting are presented in Sec.~\ref{sec:Main_results}. The main results for the discounted reward setting are presented in Sec.~\ref{sec:extension}. The main results for the finite-horizon setting are presented in Sec.~\ref{sec:finite_horizon}. The extension of our results to the case with stochastic reward is presented in Sec.~\ref{sec:stoch_reward}. Our concluding remarks are presented in Sec.~\ref{sec:conclusion}.  Moreover, App.~\ref{sec:solvability_cond} presents a background discussion on Markov chain theory. App.~\ref{sec:Martingales} presents a background discussion on concentration of martingale sequences. 
Proofs of main results are presented in the remaining appendices: App.~\ref{sec:pf_main} for the average reward MDPs, App.~\ref{sec:pf_discount_main} for the discounted reward MDPs, and App.~\ref{sec:pf_finite_horizon} for the finite-horizon MDPs.
\subsection{Notation}\label{sec:notation} 
The symbols $\reals$ and $\integers$ denote the sets of real and natural numbers and $\reals_{+}$ denotes the set of positive real numbers. The notation $\lim_{\gamma \toup 1}$ means the limit as $\gamma$ approaches $1$ from below.  Given a sequence of positive numbers $\{a_{t}\}_{t \geq 0}$ and a function $f \colon \integers \to \reals$, the notation $a_{T} = \OO(f(T))$ means that $\limsup_{T\to\infty} a_{T}/f(T) <\infty$ and $a_{T} = \tildeO(f(T))$ means there exists a finite constant $\alpha$ such that $a_{T} = \OO(\log(T)^{\alpha}f(T))$.

Given a finite set $\mathcal{S}$, $|\mathcal{S}|$ denotes its cardinality and $\Delta(\mathcal S)$ denotes the space of probability measures defined on $\mathcal S$. For a function $V \colon \mathcal S \to \reals$, the span of the function $\spand(V)$ is defined as
\[
\spand(V)\coloneqq \max_{s \in \mathcal{S}}V(s) - \min_{s \in \mathcal{S}}V(s).
\]
Given a probability space $(\Omega,\mathcal{F},\PR)$, the notation $\EXP$ denotes the expectation operator. 
Given a sequence of random variables $\{S_{t}\}_{t \geq 0}$, $S_{0:t}$ is a short hand for $(S_{0},\ldots,S_{t})$ and $\sigma(S_{0:t})$ is the sigma-field generated by random variables $S_{0:t}$. The notation $S \sim \rho$ denotes that the random variable $S$ is sampled from the distribution $\rho$. The standard Gaussian distribution is denoted by $\mathcal{N}(0,1)$. 
Convergence in distribution is denoted by $\xrightarrow[]{(d)}$, almost sure convergence is denoted by $\xrightarrow[]{(a.s.)}$, and convergence in probability is denoted by $\xrightarrow[]{(p)}$.  The phrase almost surely is abbreviated as  $a.s.$ and the phrase infinitely often is abbreviated as $i.o.$ The phrases right hand side and left hand side are abbreviated as RHS and LHS, respectively. 

\section{Problem Formulation}\label{sec:Problem_formulation}
\subsection{System Model}
Consider a Markov Decision Process (MDP) with state space $\mathcal{S}$ and action space $\mathcal{A}$. We assume that $\mathcal{S}$ and $\mathcal{A}$ are finite sets and use $S_{t} \in \mathcal{S}$ and $A_{t} \in \mathcal{A}$ to denote the state and action at time $t$. 
At time $t = 0$, the system starts at an initial state $S_{0}$, which is a random variable with probability mass function~$\rho$. The state evolves in a controlled Markov manner with transition matrix $P$, i.e., for any realizations $s_{0:t+1}$ of $S_{0:t+1}$ and $a_{0:t}$ of $A_{0:t}$, we have:
\[
\PR (S_{t+1} = s_{t+1} | S_{0:t} = s_{0:t}, A_{0:t} = a_{0:t}) = P(s_{t+1}|s_{t},a_{t}).
\]
In the sequel, we will use the notation $\EXP[f(S_{+})|s,a]$ to denote the expectation with respect to $P$, i.e.,
\[
\EXP\big[f(S_{+})|s,a\big] = \sum_{s_{+} \in \mathcal{S}}f(s_{+})P(s_{+}|s,a).
\]
At each time $t$, an agent observes the state of the system $S_{t}$ and chooses the control action as  $A_{t} \sim \pi_{t}(S_{0:t},A_{0:t-1})$, where $\pi_{t}:\mathcal{S}^{t} \times \mathcal{A}^{t-1} \to \Delta(\mathcal{A})$ is the \emph{decision rule} at time $t$. The collection $\pi = (\pi_0 ,\pi_1, \ldots)$ is called a \emph{policy}. We use $\Pi$ to denote the set of all (history dependent and time varying) policies. 

At each time $t$, the system yields a per-step reward $r(S_{t},A_{t})$, where $r:\mathcal{S} \times \mathcal{A} \rightarrow [0,R_{\max}]$. Let $R^{\pi}_{T}$ denote the total reward received by policy $\pi$ until time $T$, i.e. 
\begin{equation*}
R^{\pi}_{T} = \sum_{t=0}^{T-1} r(S_{t},A_{t}),
\quad 
\text{where } A_t \sim \pi(S_{0:t}, A_{0:t-1}).     
\end{equation*}
Note that $R^{\pi}_{T}$ is a random variable and we sometimes use the notation $R^{\pi}_{T}(\omega)$, $\omega \in \Omega$, to indicate its dependence on the sample path. The long-run expected average reward of a policy $\pi \in \Pi$ starting at the state $s\in\mathcal{S}$ is defined as
\[
J^{\pi}(s) = \liminf_{T \to \infty} \frac{1}{T}\EXP^{\pi}\big[R_{T}^{\pi}|S_{0} = s\big],\quad \forall s \in \mathcal{S},
\]
where $\EXP^{\pi}$ is the expectation with respect to the joint distribution of all the system variables induced by $\pi$.
The optimal performance $J^{*}$ starting at state $s\in \mathcal{S}$ is defined as
\[
J^{*}(s) = \sup _{\pi \in \Pi} J^{\pi}(s), \quad \forall s \in \mathcal{S}.
\]
A policy $\pi^{*}$ is called \emph{optimal} if 
\[
J^{\pi^{*}}(s) = J^{*}(s), \quad \forall s \in \mathcal{S}.
\]
\subsection{The Average Reward Planning Setup}
Suppose the system model $\mathcal{M} = (P,r)$ is known. 
\begin{definition}
Given a model $\mathcal{M} = (P,r)$, define $\Pi_{\SD} \subseteq \Pi$ to be the set of all stationary deterministic Markov policies, i.e., for any $\pi = (\pi_{0},\pi_{1},\ldots) \in \Pi_{\SD}$, we have $\pi_{t} : \mathcal{S}\to \mathcal{A}$ (i.e., $A_{t} = \pi_{t}(S_{t})$), and $\pi_{t}$ is the same for all $t$.  
\end{definition}
 With a slight abuse of notation, given a decision rule $\pi:\mathcal{S}\to \mathcal{A}$, we will denote the stationary policy $(\pi,\pi,\pi,\ldots
)$ by $\pi$ and interpret $R^{\pi}_{T}$ and $J^{\pi}$ as $R^{(\pi,\pi,\ldots)}_{T}$ and $J^{(\pi,\pi,\ldots)}$, respectively. 
A stationary policy $\pi \in \Pi_{\SD}$ induces a time-homogeneous Markov chain on $\mathcal{S}$ with transition probability matrix
\[
P^{\pi}(s_{t+1}|s_{t}) \coloneqq P(s_{t+1}|s_{t},\pi(s_{t})), \quad \forall s_{t},s_{t+1} \in \mathcal{S}.
\]

\begin{definition}[AROE Solvability]\label{prop:ACOE}
     A model $\mathcal{M} = (P,r)$ is said to be \textup{AROE} (Average Reward Optimality Equation) solvable if there exists a unique optimal long-term average reward $\lambda^{*} \in \reals$ and an optimal differential value function $V^{*}\colon\mathcal{S} \to \reals$ that is unique up to an additive constant that satisfy:
    \begin{equation}\label{eq:ACOE}
    \lambda^{*} + V^{*}(s) = \max_{a \in \mathcal{A}}\Big[r(s,a) + \EXP\big[V^{*}(S_{+})\big|s,a\big]\Big], \quad \forall s \in \mathcal{S}. \tag{AROE}
    \end{equation}
\end{definition}
\begin{definition}
    Given a model $\mathcal{M} = (P,r)$, a policy $\pi \in \Pi_{\SD}$ is said to satisfy \textup{ARPE} (Average Reward Policy Evaluation equation) if there exists a unique long-term average reward $\lambda^{\pi} \in \reals$ and a differential value function $V^{\pi}\colon\mathcal{S}\to \reals$ that is unique up to an additive constant that satisfy:
    \begin{equation}\label{eq:AOE}
    \lambda^{\pi} + V^{\pi}(s) = r(s,\pi(s)) + \EXP\big[V^{\pi}(S_{+})|s,\pi(s)\big], \quad \forall s \in \mathcal{S}. \tag{ARPE}
    \end{equation} 
\end{definition}
\begin{definition}
Given a model $\mathcal{M}=(P,r)$, define $\Pi_{\AC} \subseteq \Pi_{\SD}$ to be the set of all stationary deterministic policies which satisfy \eqref{eq:AOE}.   
\end{definition}
 The next two propositions follow from standard results in MDP theory. 
\begin{proposition}[{\citet[Prop.~5.2.1.]{bertsekas2012dynamic}}]\label{prop:bertsekas}
    Suppose model $\mathcal{M} = (P,r)$ is \textup{AROE} solvable with a solution $(\lambda^{*},V^{*})$. Then:
    \begin{enumerate}
        \item For all $s \in \mathcal{S}$, $J^{*}(s) = \lambda^{*}$.
        \item\label{itm:pi-star} Let $\pi^{*} \in \Pi_{\SD}$ be any policy such that $\pi^{*}(s)$ is an argmax of the RHS of \eqref{eq:ACOE}. Then $\pi^{*}$ is optimal, i.e., for all $s \in \mathcal{S}$, $J^{\pi^{*}}(s) = J^{*}(s) = \lambda^{*}$.
        \item The policy $\pi^{*}$ in item~\ref{itm:pi-star} belongs to $\Pi_{\AC}$. In particular, it satisfies \eqref{eq:AOE} with a solution $(\lambda^{*},V^{*})$.
    \end{enumerate}
\end{proposition}

\begin{proposition}[{\citet[Prop.~5.2.2]{bertsekas2012dynamic}}]\label{prop:ACFP}
    For any policy $\pi \in \Pi_{\AC}$, we have $J^{\pi}(s) = \lambda^{\pi}$, for all $s \in \mathcal{S}$. 
\end{proposition}
We assume that model $\mathcal M$ satisfies the following assumption.
\begin{assumption}\label{ass:exitence_optimal}
    The model $\mathcal M = (P,r)$ is AROE solvable. Hence, there exists an optimal policy $\pi^* \in \Pi_{\AC}$. 
\end{assumption}
Proposition~\ref{prop:bertsekas} implies that under Assumption~\ref{ass:exitence_optimal}, $J^{*}(s)$ is constant. In the rest of this section we assume that Assumption~\ref{ass:exitence_optimal} always holds and denote $J^{*}(s)$ by $J^{*}$.

\subsection{Classification of MDPs}
We present the main results of this paper for the policy class $\Pi_{\AC}$ under Assumption~\ref{ass:exitence_optimal}. However, by imposing further assumptions on $\mathcal{M}$, we can provide a finer characterization of the set $\Pi_{\AC}$ and provide sufficient conditions to guarantee Assumption~\ref{ass:exitence_optimal}. We recall definitions of different classes of MDPs. Depending on the properties of states following the policies in $\Pi_{\SD}$, we can classify MDPs to various classes.

\begin{definition}[\citet{kallenberg2002classification}]
    We say that $\mathcal{M}$ is 
    \begin{enumerate} 
        \item \textbf{Recurrent (or ergodic)} if for \emph{every} policy $\pi \in \Pi_{\SD}$, the transition matrix $P^{\pi}$
        consists of a single recurrent class. 
        \item \textbf{Unichain} if for \emph{every} policy $\pi \in \Pi_{\SD}$, the transition matrix $P^{\pi}$
        is unichain, i.e., it consists of a single recurrent class plus a possibly empty set of transient states. 
        \item \textbf{Communicating} if, for every pair of states $s,s' \in \mathcal{S}$, there exists a policy $\pi \in \Pi_{\SD}$ under which $s'$ is accessible from $s$.
        \item \textbf{Weakly Communicating} if there exists a closed set of states $\mathcal{S}_{c}$ such that (i)~for every two states $s,s' \in \mathcal{S}_{c}$, there exists a policy $\pi \in \Pi_{\SD}$ under which $s'$ is accessible from $s$; (ii)~all states in $\mathcal {S} \setminus \mathcal S_c$ are transient under \emph{every} policy.
    \end{enumerate}
\end{definition}
See App.~\ref{sec:solvability_cond} for the details related to the definitions of Markov chains.
The following proposition shows the connections between the MDP classes defined above. 
\begin{proposition}[{\citet[Figure~8.3.1.]{puterman2014markov}}]\label{prop:relation_of_mdps}
The following statements hold: 
    \begin{enumerate}
        \item If $\mathcal{M}$ is recurrent then it is also unichain.
        \item If $\mathcal{M}$ is unichain then it is also weakly communicating.
        \item If $\mathcal{M}$ is communicating then it is also weakly communicating. 
    \end{enumerate}
\end{proposition}
By definition, we know that $\Pi_{\AC} \subseteq \Pi_{\SD}$. However, providing a finer characterization of the set $\Pi_{\AC}$ requires further assumptions on the model $\mathcal{M}$. The following proposition presents a sufficient condition for $\mathcal{M}$ under which $\Pi_{\AC} = \Pi_{\SD}$, as well as conditions guaranteeing that $\Pi_{\AC}$ is non-empty, showing the existence of an optimal policy $\pi^{*} \in \Pi_{\AC}$.

\begin{proposition}[{\citet[Table~8.3.1.]{puterman2014markov}}]\label{prop:unichain_communicating}
    The following properties hold: 
    \begin{enumerate}
        \item If $\mathcal{M}$ is recurrent or unichain, then $\Pi_{\SD} = \Pi_{\AC}$.
        \item If $\mathcal{M}$ is recurrent, unichain, communicating, or weakly communicating, then there exists an optimal policy $\pi^{*} \in \Pi_{\AC}$. Hence $\Pi_{\AC}$ is non-empty. 
    \end{enumerate}
\end{proposition}

\subsection{The Average Reward Learning Setup}
We now consider the case where the system model $\mathcal{M}=(P,r)$ is not known. In this case, an agent must use a history dependent policy belonging to $\Pi$ to \emph{learn} how to act. To differentiate from the planning setting, we denote such a policy by $\mu$ and refer to it as a \emph{learning policy}.
The quality of a learning policy $\mu \in \Pi$ is quantified by the regret with respect to the optimal policy $\pi^*$. There are two notions of regret in the literature, which we state below.

\begin{enumerate}
    \item \textbf{Interim cumulative regret\footnote{In the stochastic bandit literature, this definition is sometimes being refereed to as the pseudo regret}  of policy $\mu$ at time $T$}, denoted by $\bar{\mathcal{R}}^{\mu}_{T}(\omega)$, is the difference between the \emph{average} cumulative reward (i.e., $TJ^{*}$) and the cumulative reward of the learning policy, i.e.,
    \begin{equation}\label{eq:reg_def2}
        \bar{\mathcal{R}}_{T}^{\mu}(\omega) \coloneqq  TJ^{*} - R^{\mu}_T(\omega) .
    \end{equation}
     \item \textbf{Cumulative regret of policy $\mu$ at time $T$}, denoted by $\mathcal{R}^{\mu}_{T}(\omega)$, is the difference between the cumulative reward of the optimal policy and the cumulative reward of the learning policy along the \emph{same sample trajectory}, i.e.,
    \begin{equation}\label{eq:reg_def1}
       \mathcal{R}_{T}^{\mu}(\omega) \coloneqq  R^{\pi^{*}}_T(\omega) - R^{\mu}_T(\omega). 
    \end{equation}
\end{enumerate}
Cumulative regret compares the sample path performance of the learning policy with the sample path performance of the optimal policy \emph{on the same sample path}, while the interim cumulative regret compares the sample path performance of the learning policy with the \emph{average} performance of the optimal policy. 

In this paper, we characterize  probabilistic upper-bounds on the difference between the regret and the interim regret and establish that up to $\tilde{\OO}(\sqrt{T})$, these two definitions are rate-equivalent under suitable assumptions.

Let $\mathcal{D}^{\mu}_{T}(\omega)$ denote the difference between the cumulative regret and the interim cumulative regret, i.e., $\mathcal{D}^{\mu}_{T}(\omega) \coloneqq \mathcal{R}_{T}^{\mu}(\omega) - \bar{\mathcal{R}}_{T}^{\mu}(\omega)$. It follows from~\eqref{eq:reg_def2}--\eqref{eq:reg_def1} that
\begin{equation}\label{eq:D(T)}
    \mathcal{D}^{\mu}_{T}(\omega) = \mathcal{R}^{\pi^{*}}_{T}(\omega) - TJ^{*}\!\!,   
\end{equation}
which implies that $\mathcal{D}^{\mu}_T(\omega)$ is not a function of the learning policy $\mu$ and it only depends on the cumulative reward received by the optimal policy. Therefore, we drop the dependence on $\mu$ in our notation and denote the difference between the cumulative regret and the interim cumulative regret by $\mathcal{D}_{T}(\omega)$. In this paper, we characterize asymptotic and non-asymptotic guarantees for the random sequence $\{\mathcal{D}_T(\omega)\}_{T\geq1}$.  
\begin{remark}
    Let $\Pi^* \subset \Pi_{AR}$ denote the set of all optimal policies that satisfy AROE. Assumption~\ref{ass:exitence_optimal} implies that $\Pi^{*} \not= \emptyset$ but in general, $|\Pi^{*}|$ may be greater than $1$. If that is the case, our results are applicable to all optimal policies in $\Pi^{*}$.
\end{remark}

\section{Main Results for the Average Reward Setup}\label{sec:Main_results}
We first define statistical properties of the differential value function which is induced by any policy $\pi\in\Pi_{\AC}$.
\subsection{Statistical Definitions}
 For any policy $\pi \in \Pi_{\AC}$, define the following properties of the value function $V^\pi$.
\begin{enumerate}
    \item Span $H^\pi$, which is given by 
    \begin{equation}\label{eq:H}
        H^{\pi} \coloneqq \spand(V^{\pi}) = \max_{s\in \mathcal{S}}V^{\pi}(s) - \min_{s\in \mathcal{S}}V^{\pi}(s) .
    \end{equation}
    \item Conditional standard deviation $\sigma^{\pi}(s)$, which is given by 
        \[
        \sigma^{\pi}(s) \coloneqq \Big[\EXP\big[\big(V^{\pi}(S_{+}) - \EXP\big[V^{\pi}(S_{+})|s,\pi(s)\big] \big)^2\big|s,\pi(s)\big]\Big]^{1/2}.
        \]
    \item Maximum absolute deviation $K^\pi$, which is given by
\begin{equation}\label{eq:k}
    K^{\pi} \coloneqq \max_{s,s_{+} \in \mathcal{S}}\!\Big|V^{\pi}(s_{+}) - \EXP\big[V^{\pi}(S_{+})\big|s,\pi(s)\big]\Big|.
\end{equation}
\end{enumerate}
For any optimal policy $\pi^* \in \Pi_{\AC}$, we denote the corresponding quantities by $H^*$, $\sigma^{*}(s)$, and $K^*$. 
\begin{remark}\label{re:shift_invariant}
    As mentioned earlier, the solution of \eqref{eq:AOE} is unique only up to an additive constant. Adding a constant to $V^{\pi}$ does not change the values of $H^{\pi}, K^{\pi}$, and $\sigma^{\pi}$. Therefore it does not matter which specific solution of \eqref{eq:AOE} is used to compute $H^{\pi}, K^{\pi}$, and $\sigma^{\pi}$.
\end{remark}

\begin{definition}[\citet{bartlett2009regal}]
    Let the expected number of steps to transition from state $s$ to state $s'$ under a policy $\pi \in \Pi_{\SD}$ be denoted by $T^{\pi}\!(s, s')$. For any policy $\pi \in \Pi_{\SD}$, the diameter of the policy $D^{\pi}$ is defined as
    \[
    D^{\pi} \coloneqq \max_{s,s' \in \mathcal{S}} T^{\pi}(s,s').
    \]
    For the model $\mathcal{M}$, the diameter $D$ and the worst case diameter $D_{w}$ are defined as
    \begin{align}
    D  \coloneqq \max_{\substack{s,s' \in \mathcal{S}\\s\not=s'}}  \min_{\pi 
    \in \Pi_{\SD}}T^{\pi}(s, s'),\\    
    D_{w}  \coloneqq \max_{\pi 
    \in \Pi_{\SD}} \max_{\substack{s,s' \in \mathcal{S}\\s\not=s'}}  T^{\pi}(s, s').
    \end{align}
\end{definition}
\begin{lemma}\label{lem:inequalities}
    Following relationships hold between the quantities $H^{\pi},K^{\pi}$, and $\sigma^{\pi}$:
    \begin{enumerate}
        \item For any policy $\pi \in \Pi_{\AC}$, we have 
            \begin{equation}\label{eq:dispersion_1}
                \sigma^{\pi}(s) \leq K^{\pi} \leq H^{\pi} < \infty, \quad \forall s \in \mathcal{S}.
            \end{equation}
        \item If $\mathcal{M}$ is communicating, then for any policy $\pi \in \Pi_{\AC}$, we have $H^{\pi} \le D^{\pi}R_{\max}$. Therefore, 
    \begin{equation}\label{eq:dispersion_2}
    \sigma^{\pi}(s) \leq K^{\pi} \leq H^{\pi} \leq D^{\pi}R_{\max}\leq D_{w}R_{\max}, \quad \forall s \in \mathcal{S}.
    \end{equation}
        \item If $\mathcal{M}$ is weakly communicating, then for any optimal policy $\pi^{*} \in \Pi_{\AC}$, we have $H^{*} \le DR_{\max}$. Therefore,
    \begin{equation}\label{eq:dispersion_3}
    \sigma^{*}(s) \leq K^{*} \leq H^{*} \leq DR_{\max}, \quad \forall s \in \mathcal{S}.
    \end{equation} 
    \end{enumerate}
\end{lemma}
The proof is presented in App.~\ref{subsec:pf_inequalities}.

This section presents three families of results. In Sec.~\ref{sec:any_policy}, we present a set of sample path properties for $R^{\pi}_{T}(\omega)$ for any policy $\pi \in \Pi_{\AC}$, depicting both asymptotic and non-asymptotic concentration of $R^{\pi}_{T}(\omega)$ around its ergodic mean. In Sec.~\ref{sec:common_J}, we apply these concentration results to characterize the sample path behavior of the difference between any two policies belonging to $\Pi_{\AC}$, while in Sec.~\ref{sec:implication_regret}, we apply these results to the optimal policy $\pi^{*}$ to derive the properties of the difference between the cumulative regret and the interim cumulative regret $\mathcal{D}_{T}(\omega)$. 
\subsection{Sample Path Characteristics Of Any Policy}\label{sec:any_policy}
In this section, we derive asymptotic and non-asymptotic sample path properties of $R^{\pi}_{T}(\omega)$ for any policy $\pi \in \Pi_{\AC}$. The following theorem characterizes the asymptotic concentration rates of $R^{\pi}_{T}(\omega)$, establishing LLN, CLT and LIL.
\begin{definition}
    Let $\{\Sigma^{\pi}_{t}\}_{t\geq0}$ denote the random process defined as
    \[
    \Sigma_{0}^{\pi} = 0, \quad \Sigma^{\pi}_{t} = \sum_{\tau=0}^{t-1}\sigma^{\pi}\!(S_{\tau})^2\!.
    \]
    Corresponding to this process, define the set $\Omega_{0}^{\pi}$ as 
    \[
    \Omega_{0}^{\pi} \coloneqq \Big\{\omega \in \Omega\hspace{0.5ex}:\hspace{0.5ex} \lim_{t\to\infty}\Sigma_{t}^{\pi}(\omega)= \infty\Big\}.
    \]
\end{definition}
\begin{theorem}\label{thm:asymptotic}
For any policy $\pi \in \Pi_{\AC}$ and any initial state $s_{0} \in \mathcal{S}$, we have following asymptotic characteristics: 
\begin{enumerate}
    \item (Law of Large Numbers) The empirical average of the cumulative reward converges almost surely to $J^{\pi}\!,$ i.e.,
    \begin{equation}\label{eq:lln}
    \lim_{T \to \infty}\frac{R^{\pi}_{T}(\omega)}{T} = J^{\pi}, \quad a.s.    
    \end{equation}
    \item (Central Limit Theorem) Assume that $\PR(\Omega_{0}^{\pi}) = 1$. Let the stopping time $\nu_{t}$ be defined as
    \(
    \nu_{t} \coloneqq \min \Big\{T \geq 1 \hspace{0.5ex}:\hspace{0.5ex} \Sigma_{T}^{\pi} \geq t \Big\}.
    \)
    Then
    \begin{equation}\label{eq:clt}
    \lim_{T \to\infty} \dfrac{R^{\pi}_{\nu_{T}}(\omega) - \nu_{T}J^{\pi}}{\sqrt{T}} \xrightarrow[]{(d)} \mathcal{N}(0,1).    
    \end{equation}
    \item (Law of Iterated Logarithm) 
    For almost all $\omega \in \Omega_0^{\pi}$, we have
    \begin{align}\label{eq:lil}
    \liminf_{T \to \infty} \frac{R_{T}^{\pi}(\omega)-TJ^{\pi}}{\sqrt{2\Sigma^{\pi}_{T}\log\log \Sigma^{\pi}_{T}}} = -1, \quad
    \limsup_{T \to \infty} \frac{R_{T}^{\pi}(\omega)-TJ^{\pi}}{\sqrt{2\Sigma^{\pi}_{T}\log\log \Sigma^{\pi}_{T}}} = 1. 
    \end{align}
\end{enumerate}
\end{theorem}
The proof is presented in App.~\ref{sec:pf_thm:asymptotic}.
  
\begin{corollary}\label{cor:asymptotic_optimal}
    For any optimal policy $\pi^{*} \in \Pi^{*}$,
    the cumulative reward $R^{\pi^{*}}_{T}(\omega)$ satisfies the asymptotic concentration rates in \eqref{eq:lln}--\eqref{eq:lil}, where in the LHS, $J^{\pi}$ is replaced with $J^{*}$.
\end{corollary}

\begin{proof}
    Since $\pi^{*}$ is in $\Pi_{\AC}$, by Theorem~\ref{thm:asymptotic}, the optimal policy should satisfy the asymptotic concentration rates in \eqref{eq:lln}--\eqref{eq:lil}.
\end{proof}
The proof of Theorem~\ref{thm:asymptotic} relies on the finiteness of $K^{\pi}$. However, due to the asymptotic nature of this result, the exact sample complexity dependence of these bounds on properties of the differential value function $V^{\pi}$ is not evident. The following theorem establishes the concentration of cumulative reward around the quantity $TJ^{\pi} - \big(V^{\pi}(S_{T}) - V^{\pi}(S_{0})\big)$.

\begin{theorem}\label{thm:finite_return_meta_new}
    For any policy $\pi \in \Pi_{\AC}$, the following upper-bounds hold: 
    \begin{enumerate} 
        \item For any $\delta \in (0,1)$, with probability at least $1-\delta$, we have
        \begin{equation}\label{eq:clt_finite_new}
        \big|R^{\pi}_{T} - TJ^{\pi} - \big(V^{\pi}(S_{0})-V^{\pi}(S_{T})\big)\big|  \leq K^{\pi}\sqrt{2T\log\frac{2}{\delta}}.    
        \end{equation}
        \item For any $\delta \in (0,1)$, for all $T \geq T^\pi_{0}(\delta) \coloneqq \Bigl\lceil \dfrac{173}{K^\pi}\log\dfrac{4}{\delta}\Big\rceil $, with probability at least $1-\delta$, we have
        \begin{equation}\label{eq:lil_finite_new}
        \big|R^{\pi}_{T} - TJ^{\pi} - \big(V^{\pi}(S_{0})-V^{\pi}(S_{T})\big)\big| \leq \max\Big\{K^{\pi}\sqrt{3T\Big(2\log\log\frac{3T}{2} + \log\frac{2}{\delta}\Big)}, (K^{\pi})^2\Big\}.
        \end{equation}
    \end{enumerate}
\end{theorem}
The proof is presented in App.~\ref{sec:pf_finite_return_meta_new}.

Theorem~\ref{thm:finite_return_meta_new} establishes a sample path dependent concentration result. The following theorem establishes a sample path independent finite-time concentration of $R^{\pi}_{T}(\omega)$ as a function of the statistical properties of $V^{\pi}$.

\begin{theorem}\label{thm:finite_return_meta}
    For any policy $\pi \in \Pi_{\AC}$, following upper-bounds hold:
    \begin{enumerate} 
        \item For any $\delta \in (0,1)$, with probability at least $1-\delta$, we have
        \begin{equation}\label{eq:clt_finite}
        \big|R^{\pi}_{T} - T\!J^{\pi}\big| \leq K^{\pi}\sqrt{2T\log\frac{2}{\delta}} + H^{\pi}\!.    
        \end{equation}
        \item For any $\delta \in (0,1)$, for all $T \geq T^\pi_{0}(\delta) \coloneqq \Bigl\lceil \dfrac{173}{K^\pi}\log\dfrac{4}{\delta}\Big\rceil $, with probability at least $1-\delta$, we have
        \begin{equation}\label{eq:lil_finite}
        \big|R^{\pi}_{T} - T\!J^{\pi}\big| \leq \max\Big\{K^{\pi}\sqrt{3T\Big(2\log\log\frac{3T}{2} + \log\frac{2}{\delta}\Big)}, (K^{\pi})^2\Big\} + H^{\pi}.
        \end{equation}
    \end{enumerate}
\end{theorem}
The proof is presented in App.~\ref{sec:pf_finite_return_meta}.

\begin{corollary}\label{cor:non_asymptotic_optimal}
    For any optimal policy $\pi^{*} \in \Pi^{*}$,
    the cumulative reward $R^{\pi^{*}}_{T}(\omega)$ satisfies the non-asymptotic concentration rates in \eqref{eq:clt_finite}--\eqref{eq:lil_finite}, where in the LHS, $J^{\pi}$ is replaced with $J^{*}$ and in the statement and RHS, $(K^{\pi},H^{\pi})$ are replaced with $(K^{*},H^{*})$.
\end{corollary}

\begin{proof}
    Since $\pi^{*}$ is in $\Pi_{\AC}$, by Theorem~\ref{thm:finite_return_meta}, the optimal policy should satisfy the non-asymptotic concentration rates in \eqref{eq:clt_finite}--\eqref{eq:lil_finite}.
\end{proof}
\begin{corollary}
    If $\mathcal{M}$ is unichain or recurrent, then any policy $\pi \in \Pi_{\SD}$ satisfies asymptotic concentration rates in \eqref{eq:lln}--\eqref{eq:lil} and non-asymptotic concentration rates in \eqref{eq:clt_finite}--\eqref{eq:lil_finite}.
\end{corollary}
\begin{proof}
    By Prop.~\ref{prop:unichain_communicating}, for the unichain or recurrent model $\mathcal{M}$, we have $\Pi_{\AC} = \Pi_{\SD}$. As a result, any policy $\pi$ which belongs to $\Pi_{\SD}$ also belongs to $\Pi_{\AC}$. Therefore, by Theorem~\ref{thm:asymptotic}, the asymptotic concentration rates in \eqref{eq:lln}--\eqref{eq:lil} hold for the policy $\pi$ and by Theorem~\ref{thm:finite_return_meta}, the non-asymptotic rates in \eqref{eq:clt_finite}--\eqref{eq:lil_finite} hold for the policy $\pi$.
\end{proof}

\begin{corollary}\label{cor:all_cases_optimal}
    If $\mathcal{M}$ is recurrent, unichain, communicating, or weakly communicating, then every optimal policy $\pi^{*} \in \Pi^{*}$ satisfies asymptotic concentration rates in \eqref{eq:lln}--\eqref{eq:lil} and non-asymptotic concentration rates in \eqref{eq:clt_finite}--\eqref{eq:lil_finite}. (Prop.~\ref{prop:unichain_communicating} shows that there exists at least one such policy.)
\end{corollary}
\begin{proof}
    By Prop.~\ref{prop:unichain_communicating}, for any model $\mathcal{M}$ which is recurrent, unichain, communicating, or weakly communicating, there exists an optimal policy $\pi^{*}$ belonging to $\Pi_{\AC}$. As a result, by Corollary~\ref{cor:asymptotic_optimal}, the asymptotic concentration rates in \eqref{eq:lln}--\eqref{eq:lil} hold for every optimal policy $\pi^{*} \in \Pi_{\AC}$. Furthermore, by Corollary~\ref{cor:non_asymptotic_optimal}, the non-asymptotic concentration rates in \eqref{eq:clt_finite}--\eqref{eq:lil_finite} hold for every optimal policy $\pi^{*} \in \Pi_{\AC}$.
\end{proof}

In Theorem~\ref{thm:finite_return_meta}, the upper-bounds are established in terms of $K^{\pi}$ and $H^{\pi}$. To compute $K^{\pi}$ and $H^{\pi}$, one must solve the corresponding \eqref{eq:AOE} equation. At the cost of loosening these bounds, we derive upper-bounds which are in terms of the diameter of the policy $D^{\pi}$ and the maximum reward $R_{\max}$. As a result, these upper-bounds only depend on the properties of the Markov chain induced by $\pi$ and $R_{\max}$.

\begin{corollary}\label{cor:finite_return_meta}
    Suppose $\mathcal{M}$ is communicating. For any policy $\pi \in \Pi_{\AC}$, following upper-bounds hold:
    \begin{enumerate} 
        \item For any $\delta \in (0,1)$, with probability at least $1-\delta$, we have
        \begin{equation}\label{eq:policy_indp_1}
            \big|R^{\pi}_{T} - T\!J^{\pi}\big| \leq D^{\pi}R_{\max}\sqrt{2T\log\frac{2}{\delta}} + D^{\pi}R_{\max}.
        \end{equation}
            \item For any $\delta \in (0,1)$, for all $T\geq T_{0}(\delta) \coloneqq \Bigl\lceil \dfrac{173}{D^{\pi}R_{\max}}\log\dfrac{4}{\delta}\Big\rceil$, with probability at least $1-\delta$, we have
        \begin{equation}\label{eq:policy_indp_2}
            \big|R^{\pi}_{T} - T\!J^{\pi}\big| \leq  \max\Big\{D^{\pi}R_{\max}\sqrt{3T\Big(2\log\log\frac{3T}{2} + \log\frac{2}{\delta}\Big)}, (D^{\pi}R_{\max})^2\Big\} + D^{\pi}R_{\max}.
        \end{equation}
    \end{enumerate}
\end{corollary}
The proof is presented in App.~\ref{app:pf_finite_return_meta}.

\begin{corollary} \label{cor:com_optimal}
    If $\mathcal{M}$ is communicating or weakly communicating, then for any optimal policy $\pi^{*} \in \Pi^{*}$,
    the cumulative reward $R^{\pi^{*}}_{T}(\omega)$ satisfies the non-asymptotic concentration rates in \eqref{eq:policy_indp_1}--\eqref{eq:policy_indp_2}, where in the LHS, $J^{\pi}$ is replaced with $J^{*}$ and in the RHS, $D^{\pi}$ is replaced with $D$.
\end{corollary}
The proof is presented in App.~\ref{sec:pf_cor_com_optimal}.

In the Corollary~\ref{cor:finite_return_meta}, the dependence of upper-bounds on the parameters of $\mathcal{M}$ are reflected through $D^{\pi}R_{\max}$. This implies that if the diameter of the policy $D^{\pi}$ or maximum reward $R_{\max}$ increases, these upper-bounds loosen with a linear rate.

\begin{remark}
    The upper-bounds derived in Corollary~\ref{cor:finite_return_meta} depend on the diameter of the policy $D^{\pi}$ and are therefore policy-dependent. If $\mathcal{M}$ is communicating or weakly communicating, by Lemma~\ref{lem:inequalities}, Part 2, we can replace the diameter of the policy $D^{\pi}$ with the worst case diameter $D_{w}$ to get policy-independent upper-bounds. For brevity, we omit this result. 
\end{remark}

\subsection{Sample Path Behavior of the Performance Difference of Two Stationary Policies} \label{sec:common_J}
As an implication of the results presented in the Sec.~\ref{sec:any_policy}, we characterize the sample path behavior of the difference in cumulative rewards between any two stationary policies. As a consequence, we derive the non-asymptotic concentration of the difference in rewards between any two optimal policies. These concentration bounds are presented in the following two corollaries.    

\begin{corollary}\label{cor:two_policies}
    Consider two policies $\pi_{1},\pi_{2} \in \Pi_{\AC}$. The following upper-bounds hold for the difference between the cumulative reward received by the two policies. 
    \begin{enumerate}
        \item For any $\delta \in (0,1)$, with probability at least $1-\delta$, we have
    \begin{align}\label{eq:two_policy_1}
    \Big|\big|R^{\pi_{1}}_{T} - R^{\pi_{2}}_{T}\big| - \big|TJ^{\pi_{1}}-TJ^{\pi_{2}}\big|\Big| &\leq K^{\pi_{1}}\sqrt{2T\log\frac{4}{\delta}} + H^{\pi_{1}} + K^{\pi_{2}}\sqrt{2T\log\frac{4}{\delta}} + H^{\pi_{2}}.  
    \end{align}
        \item For any $\delta \in (0,1)$, for all $T \geq T^\pi_{0}(\delta) \coloneqq \max\Big\{\Bigl\lceil \dfrac{173}{K^{\pi_1}}\log\dfrac{8}{\delta}\Big\rceil,\Bigl\lceil \dfrac{173}{K^{\pi_{2}}}\log\dfrac{8}{\delta}\Big\rceil\Big\} $, with probability at least $1-\delta$, we have
\begin{align}\label{eq:two_policy_2} 
    \Big|\big|R^{\pi_{1}}_{T} - R^{\pi_{2}}_{T}\big| - \big|TJ^{\pi_{1}}-TJ^{\pi_{2}}\big|\Big| 
    \leq &\max\Big\{K^{\pi_{1}}\sqrt{3T\Big(2\log\log\frac{3T}{2} + \log\frac{4}{\delta}\Big)}, (K^{\pi_{1}})^2\Big\} + H^{\pi_{1}}\notag\\
    +&\max\Big\{K^{\pi_{2}}\sqrt{3T\Big(2\log\log\frac{3T}{2} + \log\frac{4}{\delta}\Big)}, (K^{\pi_{2}})^2\Big\} + H^{\pi_{2}}.
\end{align} 
    \end{enumerate}
\end{corollary}
The proof is presented in App.~\ref{sec:pf_cor_two_policies}.

\begin{corollary}\label{cor:two_optimal}
    Consider two optimal policies $\pi^{*}_{1},\pi^{*}_{2}\in \Pi^{*}$. Then for the difference between cumulative rewards received by the two optimal policies $\big|R^{\pi^{*}_{1}}_{T} - R^{\pi^{*}_{2}}_{T} \big|$, we have 
        \begin{enumerate}
        \item For any $\delta \in (0,1)$, with probability at least $1-\delta$, we have
        \begin{equation}\label{eq:two_policy_3}
        \big| R_{T}^{\pi^{*}_{1}} - R_{T}^{\pi^{*}_{2}} \big| \leq 2\Big(K^{*}\sqrt{2T\log\frac{4}{\delta}} + H^{*}\Big).
        \end{equation}
        \item For any $\delta \in (0,1)$, for all $T \geq T^{\pi^{*}}_{0}(\delta) \coloneqq \Bigl\lceil \dfrac{173}{K^{*}}\log\dfrac{8}{\delta}\Big\rceil $, with probability at least $1-\delta$, we have
        \begin{equation}\label{eq:two_policy_4}
        \big| R_{T}^{\pi^{*}_{1}} - R_{T}^{\pi^{*}_{2}} \big| \leq 2\Big(\max\Big\{K^{*}\sqrt{3T\Big(2\log\log\frac{3T}{2} + \log\frac{4}{\delta}\Big)}, (K^{*})^2\Big\} + H^{*}\Big).   
        \end{equation}
    \end{enumerate}
\end{corollary}

\begin{proof}
    Since both policies $\pi^{*}_{1}, \pi^{*}_{2} \in \Pi_{\AC}$ are optimal policies, by the definition, we have $J^{\pi^{*}_{1}} = J^{\pi^{*}_{2}} = J^{*}$ and therefore, $T \big|J^{\pi^{*}_{1}}-J^{\pi^{*}_{2}} \big| = 0$. As a result, by Corollary~\ref{cor:two_policies}, the difference $\big|R^{\pi^{*}_{1}}_{T} - R^{\pi^{*}_{2}}_{T} \big|$ satisfies the non-asymptotic concentration rates in Corollary~\ref{cor:two_policies} with the RHS of \eqref{eq:two_policy_1}--\eqref{eq:two_policy_2} being simplified to RHS of \eqref{eq:two_policy_3}--\eqref{eq:two_policy_4}.
\end{proof}

\begin{remark}
    Similar to the Corollary~\ref{cor:finite_return_meta}, by imposing the assumption that $\mathcal{M}$ is communicating or weakly communicating, we can derive the counterpart of \eqref{eq:two_policy_1}--\eqref{eq:two_policy_2} and \eqref{eq:two_policy_3}--\eqref{eq:two_policy_4}  in terms of $D^{\pi}R_{\max}$ respectively. For brevity, we omit this result.
\end{remark}

\subsection{Implication for Learning}\label{sec:implication_regret}

In this section, we present the consequences of our results on the regret of learning algorithms. We characterize the asymptotic and non-asymptotic sample path behavior of the difference between cumulative regret and interim cumulative regret. Recall that for any learning policy $\mu$, this difference is defined as $\mathcal{D}_{T}(\omega) = \bar{\mathcal{R}}^{\mu}_{T}(\omega) - \mathcal{R}^{\mu}_{T}(\omega) $. Similar to Theorem~\ref{thm:asymptotic}, we characterize the asymptotic concentration
rates of $\{\mathcal{D}_{T}(\omega)\}_{T\geq1}$, establishing LLN, CLT and LIL.

\begin{definition}
    Let $\{\Sigma^{*}_{t}\}_{t\geq0}$ denote the random process defined as
    \[
    \Sigma^{*}_0 = 0, \quad \Sigma^{*}_{t} = \sum_{\tau=0}^{t-1}\sigma^{*}(S_{\tau})^2.
    \]
    Corresponding to this process, we define the set $\Omega_{0}^{*}$ as 
    \[
    \Omega_{0}^{*} \coloneqq \Big\{\omega \in \Omega\hspace{0.5ex}:\hspace{0.5ex} \lim_{t\to\infty}\Sigma^{*}_{t}(\omega)= \infty\Big\}.
    \]
\end{definition}

\begin{theorem}\label{thm:asymp_D}
    For any learning policy $\mu$, the difference $\mathcal{D}_{T}(\omega)$ of cumulative regret and interim cumulative regret satisfies following properties.
    \begin{enumerate}
    \item (Law of Large Numbers) The difference almost surely grows sub-linearly, i.e.
    \[
    \lim_{T\to\infty} \frac{\mathcal{D}_T(\omega)}{T} =   0, \quad a.s.
    \]
    \item (Central Limit Theorem) Assume that $\PR(\Omega_{0}^{*}) = 1$. Let stopping time $\nu_{t}$ be defined as
    \(
    \nu_{t} \coloneqq \min \Big\{T \geq 1 \hspace{0.5ex}:\hspace{0.5ex} \Sigma_{T}^{*} \geq t \Big\}.
    \)
    Then  
    \[
    \lim_{T\to\infty} \frac{\mathcal{D}_{\nu_{T}}(\omega)}{\sqrt{T}}  \xrightarrow[]{(d)} \mathcal{N}(0,1).
    \]
    \item (Law of Iterated Logarithm) 
    For almost all $\omega \in \Omega_0^{*}$, we have
    \begin{align}
    \liminf_{T \to \infty} \frac{\mathcal{D}_T(\omega)}{\sqrt{2\Sigma^{*}_{T}\log\log \Sigma^{*}_{T}}} = -1, \quad  
    \limsup_{T \to \infty} \frac{\mathcal{D}_T(\omega)}{\sqrt{2\Sigma^{*}_{T}\log\log \Sigma^{*}_{T}}} = 1.
    \end{align}
\end{enumerate}
\end{theorem}
Proof is presented in App.~\ref{sec:pf_symp_D}.

In addition to the asymptotic results presented in Theorem~\ref{thm:asymp_D}, we present non-asymptotic guarantees for the sequence $\{\mathcal{D}_{T}(\omega)\}_{T\geq1}$. Similar to Theorem~\ref{thm:finite_return_meta}, we characterize the non-asymptotic 
concentration of $\mathcal{D}_{T}(\omega)$ as a function of statistical properties of $V^{*}$ (i.e., $K^{*}$ and $H^{*}$).

\begin{theorem}\label{thm:D_finite}
    The difference of cumulative regret and interim cumulative regret $\mathcal{D}_{T}(\omega)$ satisfies:
    \begin{enumerate} 
        \item For any $\delta \in (0,1)$, with probability at least $1-\delta$, we have
        \[
        \big|\mathcal{D}_{T}(\omega)\big| \leq K^{*}\sqrt{2T\log\!\frac{2}{\delta}} + H^{*}.
        \]
        \item For any $\delta \in (0,1)$, for all $T\geq T^{*}_{0}(\delta) \coloneqq \Bigl\lceil \dfrac{173}{K^*}\log\dfrac{4}{\delta}\Big\rceil $, with probability at least $1-\delta$, we have
        \[
        \big|\mathcal{D}_{T}(\omega)\big| \leq \max\Big\{K^{*}\sqrt{3T\Big(2\log\log\frac{3T}{2} + \log\frac{2}{\delta}\Big)}, (K^{*})^2\Big\} + H^{*}.
        \]
    \end{enumerate}
\end{theorem}
Proof is presented in App.~\ref{sec:pf_D_finite}. As mentioned earlier, the difference $\mathcal{D}_T(\omega)$ does not depend on the learning policy $\mu$. Therefore, the results of Theorem~\ref{thm:D_finite} do not depend on the choice of the learning policy either.

In Theorem~\ref{thm:D_finite}, the upper-bounds are established in terms of $K^{*}$ and $H^{*}$. Similar to Corollary~\ref{cor:finite_return_meta}, we can derive upper-bounds in terms of model parameters $D$ and $R_{\max}$ at the cost of loosening the upper-bounds. These bounds are presented in the following Corollary. 
\begin{corollary}\label{cor:D_finite}
    Suppose $\mathcal{M}$ is recurrent, unichain, communicating, or weakly communicating, then $\mathcal{D}_{T}(\omega)$ satisfies following properties.
    \begin{enumerate} 
        \item For any $\delta \in (0,1)$, with probability at least $1-\delta$, we have
        \begin{align*}
            \big|\mathcal{D}_{T}(\omega)\big| \leq DR_{\max}\sqrt{2T\log\!\frac{2}{\delta}} + DR_{\max}.
        \end{align*}
        \item For any $\delta \in (0,1)$, for all $T\geq T_{0}(\delta) \coloneqq \Bigl\lceil \dfrac{173}{DR_{\max}}\log\dfrac{4}{\delta}\Big\rceil$, with  probability at least $1-\delta$, we have
        \begin{align*}
            \big|\mathcal{D}_{T}(\omega)\big| \leq \max\Big\{DR_{\max}\sqrt{3T\Big(2\log\log\frac{3T}{2} + \log\frac{2}{\delta}\Big)}, (DR_{\max})^2\Big\} + DR_{\max}.
        \end{align*}
    \end{enumerate}
\end{corollary}
Proof is presented in App.~\ref{sec:pf_cor_D_finite}.
\begin{remark}
    Notice that conditions of Corollary~\ref{cor:D_finite} are weaker than the conditions of Corollary~\ref{cor:finite_return_meta}.  As a result, Corollary~\ref{cor:D_finite} can be applied to broader classes of $\mathcal{M}$. This difference originates from the difference between items~(2) and (3) in Lemma~\ref{lem:inequalities}.
\end{remark}

In this section, we established probabilistic upper-bounds for the difference between cumulative regret and interim cumulative regret. We showed, asymptotically and non-asymptotically, the growth rate of this difference is upper-bounded by $\tildeO(\sqrt{T})$. This implies that if we establish a regret rate of $\tildeO(\sqrt{T})$ for a learning algorithm $\mu$ using either of the definitions, similar regret rate hold for the algorithm $\mu$ using the other definition. This result is presented in the following theorem.

\begin{theorem}\label{thm:rate_equiv}
For any learning policy $\mu$ we have:
    \begin{enumerate}
        \item The following statements are equivalent.
        \begin{enumerate}
            \item $R^\mu_T(\omega) \leq  \tilde {\mathcal{O}}(\sqrt{T})$, a.s.
            \item $\bar R^\mu_T(\omega) \leq \tilde {\mathcal{O}}(\sqrt{T})$, a.s.
        \end{enumerate}
        \item The following statements are true.
        \begin{enumerate}
            \item Suppose for a learning algorithm $\mu$ and any $\delta \in (0,1)$, there exists a $T_{0}(\delta)$ such that for all $T\geq T_{0}(\delta)$, with probability at least $1-\delta$, we have $R^{\mu}_{T}(\omega) \leq \tilde {\mathcal{O}}(\sqrt{T})$, where $\tildeO(\cdot)$ notation functionally depends upon constants related to $\mathcal{M}$ and $\delta$. Then for any $\delta \in (0,1)$,
            there exists $T_{1}(\delta)$ such that for all $T\geq T_{1}(\delta)$, with probability at least $1-\delta$, we have $\bar{R}^{\mu}_{T}(\omega) \leq \tilde {\mathcal{O}}(\sqrt{T})$.
            \item Suppose for a learning algorithm $\mu$ and any $\delta \in (0,1)$, there exists a $T_{0}(\delta)$ such that for all $T\geq T_{0}(\delta)$, with probability at least $1-\delta$, we have $\bar{R}^{\mu}_{T}(\omega) \leq \tilde {\mathcal{O}}(\sqrt{T})$, where $\tildeO(\cdot)$ notation functionally depends upon constants related to $\mathcal{M}$ and $\delta$. Then for any $\delta \in (0,1)$,
            there exists $T_{1}(\delta)$ such that for all $T\geq T_{1}(\delta)$, with probability at least $1-\delta$, we have $R^{\mu}_{T}(\omega) \leq \tilde {\mathcal{O}}(\sqrt{T})$.
        \end{enumerate}
    \end{enumerate}    
\end{theorem}
Proof is presented in App.~\ref{sec:pf_thm_rate_equiv}.

\section{Main Results for the Discounted Reward Setup} \label{sec:extension}
In this section, we extend the non-asymptotic concentration results that we established for the average reward setup to the discounted reward setup.
\subsection{System Model}
Consider a discounted reward MDP with state space $\mathcal{S}$ and action space $\mathcal{A}$. Similar to Sec.~\ref{sec:Problem_formulation}, we assume that $\mathcal{S}$ and $\mathcal{A}$ are finite sets. The state evolves in a controlled Markov manner with transition matrix $P$ and at each time $t$, the system yields a per-step reward $r(S_{t},A_{t}) \in [0,R_{\max}]$. Let $\gamma \in (0,1)$ denote the discount factor of the model. The definitions of policies and policy sets $\Pi$ and $\Pi_{\SD}$ are similar to Sec.~\ref{sec:Problem_formulation}. The discounted cumulative reward received by any policy $\pi$ is given by
\[
R^{\pi,\gamma}_{T}\!(\omega) \coloneqq \sum_{t=0}^{T-1}\gamma^{t}r(S_{t},A_{t}), \quad \text{where, } A_{t} = \pi(S_{0:t},A_{0:t-1}), \quad \omega \in \Omega.
\]
Note that $R^{\pi,\gamma}_{T}\!(\omega)$ is a random variable. For this model, the long-run expected discounted reward of policy $\pi \in \Pi_{\SD}$ starting at the state $s \in \mathcal{S}$ is defined as 
\[
V^{\pi}_{\gamma}(s) \coloneqq \EXP^{\pi}\Big[\lim_{T\to\infty}R^{\pi,\gamma}_{T}\bigm|S_{0} = s\Big], \quad \forall s \in \mathcal{S},
\]
where $\EXP^{\pi}$ is the expectation with respect to the joint distribution of all the system variables induced by $\pi$.
We refer to the function $V_{\gamma}^{\pi}$ as the discounted value function corresponding to the policy $\pi$. The optimal performance $V^{*}_{\gamma}$ starting at state $s\in \mathcal{S}$ is defined as
\[
V^{*}_{\gamma}(s) = \sup_{\pi \in \Pi}V^{\pi}_{\gamma}(s), \quad \forall s \in \mathcal{S}.
\]
A policy $\pi^{*}$ is called optimal if 
\[
V^{\pi^{*}}_{\gamma}(s) = V^{*}_{\gamma}(s),\quad \forall s \in \mathcal{S}.
\]
\begin{definition}
    A discounted model $\mathcal{M}$ is said to satisfy DROE (Discounted Reward Optimality Equation) if there exists an optimal discounted value function $V^{*}_{\gamma} : \mathcal{S} \to \reals$ that satisfies:
    \begin{equation}\label{eq:bellman_opt}
    V^{*}_{\gamma}(s) = \max_{a\in \mathcal{A}}\Big[r(s,a) + \gamma\EXP\big[V^{*}_{\gamma}(S_{+})\bigm| s,a\big]\Big], \quad \forall s \in \mathcal{S}.\tag{DROE}   
    \end{equation}
\end{definition}
\begin{definition}
    Given a discounted model $\mathcal{M}$, a policy $\pi \in \Pi_{\SD}$ is said to satisfy DRPE (Discounted Reward Policy Evaluation equation) if there exists a discounted value function $V_{\gamma}^{\pi} : \mathcal{S} \to \reals$ that satisfies:
    \begin{equation}\label{eq:bellman}
    V^{\pi}_{\gamma}(s) = r(s,\pi(s)) + \gamma\EXP\big[V^{\pi}_{\gamma}(S_{+})\bigm|s,\pi(s)\big], \quad \forall s \in \mathcal{S}.\tag{DRPE}   
    \end{equation}
\end{definition}
\begin{proposition}[{\citet[Prop.~1.2.3--1.2.5]{bertsekas2012dynamic}}]\label{prop:bertsekas_discounted}
    For a discounted model $\mathcal{M}$, following statements hold:
    \begin{enumerate}
        \item Any policy $\pi \in \Pi_{\SD}$ satisfies \eqref{eq:bellman}.
        \item\label{itm:pi-star_discount} Let $\pi^{*}$ be any policy such that $\pi^{*}(s)$ is an argmax of the RHS of \eqref{eq:bellman_opt}. Then $\pi^{*}$ is optimal, i.e., for all $s \in \mathcal{S}$, $V^{\pi^{*}}_{\gamma}(s) = V^{*}_{\gamma}(s)$.
        \item The policy $\pi^{*}$ in step~\ref{itm:pi-star} belongs to $\Pi_{\SD}$. In particular, it satisfies \eqref{eq:bellman} with a solution $V^{*}_{\gamma}$.
    \end{enumerate}
\end{proposition}
\subsection{Sample Path Characteristics of Any Policy}\label{sec:dis_any_policy}
For any policy $\pi \in \Pi_{\SD}$, we define following statistical properties of the discounted value function $V^{\pi}_{\gamma}$.
\begin{enumerate}
    \item Span of the discounted value function $V^{\pi}_{\gamma}$ given by
    \begin{equation}\label{eq:H_discount}
        H^{\pi,\gamma} \coloneqq \spand(V^{\pi}_{\gamma}) = \max_{s\in \mathcal{S}}V^{\pi}_{\gamma}(s) - \min_{s\in \mathcal{S}}V^{\pi}_{\gamma}(s) .
    \end{equation}
    \item Maximum absolute deviation of the discounted value function $V^{\pi}_{\gamma}$ is given by
        \begin{equation}\label{eq:k_discount}
        K^{\pi,\gamma} \coloneqq \max_{s,s_{+} \in \mathcal{S}}\!\Big|V^{\pi}_{\gamma}(s_{+}) - \EXP\big[V^{\pi}_{\gamma}(S_{+})\bigm|s,\pi(s)\big]\Big|.
        \end{equation} 
\end{enumerate}
For any optimal policy $\pi^* \in \Pi_{\SD}$, we denote these corresponding quantities by $H^{*,\gamma}$, and  $K^{*,\gamma}$. Similar to the results in Theorem~\ref{thm:finite_return_meta_new} for the average reward setup, we can derive non-asymptotic concentration results for the discounted reward setup. These results are presented in the following theorem. To simplify the notation, let 
\[
 f^{\gamma}(T) \coloneqq \sum_{t=1}^{T}\gamma^{2t} =  \frac{\gamma^{2}-\gamma^{2T+2}}{1-\gamma^2}.
\]
An immediate implication of the definitions of $R^{\pi,\gamma}_{T}$ and $V_{\gamma}^{\pi}(s)$ is that 
\[
\EXP\Big[R^{\pi,\gamma}_{T} + \gamma^{T}V_{\gamma}^{\pi}(S_{T}) - V_{\gamma}^{\pi}(S_{0})\Big] = 0.
\]
In this section, we show that with high-probability $R^{\pi,\gamma}_{T}$ concentrates around $V_{\gamma}^{\pi}(S_{0}) - \gamma^{T}V_{\gamma}^{\pi}(S_{T})$ and characterize the concentration rate.   
\begin{theorem}\label{thm:discounted_non_asymp}
    For any policy $\pi \in  \Pi_{\SD}$ and any $s \in \mathcal{S}$, we have:
    \begin{enumerate} 
        \item For any $\delta \in (0,1)$, with probability at least $1-\delta$, we have
        \begin{align}\label{eq:discounted_finite1}
        \Big|R^{\pi,\gamma}_{T} - \big(V^{\pi}_{\gamma}(S_{0})-\gamma^{T}V^{\pi}_{\gamma}(S_{T})\big)\Big| \leq K^{\pi,\gamma} \sqrt{2f^{\gamma}(T)\log\frac{2}{\delta}}. 
        \end{align}
        \item For any $\delta \in (0,1)$, if $\lim_{T' \to \infty}f^{\gamma}(T') > \dfrac{173}{K^{\pi,\gamma}}\log\dfrac{4}{\delta}$, then for all $T\geq T_{0}(\delta)\coloneqq \min\Big\{T' \geq 1 : f^{\gamma}(T') > \dfrac{173}{K^{\pi,\gamma}}\log\dfrac{4}{\delta}\Big\}$, with probability at least $1-\delta$, we have
        \begin{align}\label{eq:discounted_finite2}
        \Big|R^{\pi,\gamma}_{T}&-\big(V^{\pi}_{\gamma}(S_{0})-\gamma^{T}V^{\pi}_{\gamma}(S_{T})\big)\Big| \notag  \\ &\leq \max\bigg\{ K^{\pi,\gamma} \sqrt{3f^{\gamma}(T)\Big(2\log\log \big(\frac{3}{2}f^{\gamma}(T)\big) +\log\frac{2}{\delta}\Big)},(K^{\pi,\gamma})^{2}\bigg\}.     
        \end{align}
    \end{enumerate}
\end{theorem}
The proof is presented in App.~\ref{app:pf_discount}.
\begin{corollary}\label{cor:discounted_non_asymp2}
    For any policy $\pi \in  \Pi_{\SD}$ and any $s \in \mathcal{S}$, we have:
    \begin{enumerate} 
        \item For any $\delta \in (0,1)$, with probability at least $1-\delta$, we have
        \begin{align}\label{eq:discounted_finite3}
        \Big|R^{\pi,\gamma}_{T} - V^{\pi}_{\gamma}(S_{0})\Big| \leq K^{\pi,\gamma} \sqrt{2f^{\gamma}(T)\log\frac{2}{\delta}} + \frac{\gamma^{T}}{1-\gamma}R_{\max}. 
        \end{align}
        \item For any $\delta \in (0,1)$, if $\lim_{T' \to \infty}f^{\gamma}(T') > \dfrac{173}{K^{\pi,\gamma}}\log\dfrac{4}{\delta}$, then for all $T\geq T_{0}(\delta)\coloneqq \min\Big\{T' \geq 1 : f^{\gamma}(T') > \dfrac{173}{K^{\pi,\gamma}}\log\dfrac{4}{\delta}\Big\}$, with probability at least $1-\delta$, we have
        \begin{align}\label{eq:discounted_finite4}
        \Big|R^{\pi,\gamma}_{T}&-V^{\pi}_{\gamma}(S_{0})\Big| \notag  \\ &\leq \max\bigg\{ K^{\pi,\gamma} \sqrt{3f^{\gamma}(T)\big(2\log\log (\frac{3}{2}f^{\gamma}(T)) +\log\frac{2}{\delta}\big)},(K^{\pi,\gamma})^{2}\bigg\} + \frac{\gamma^{T}}{1-\gamma}R_{\max}.
        \end{align}
    \end{enumerate}  
\end{corollary}
The proof is presented in App.~\ref{app:pf_discount_2}.
\begin{corollary}\label{cor:discounted_non_asymp}
    For any optimal policy $\pi^{*} \in \Pi_{\SD}$, the discounted cumulative reward $R^{\pi^{*}\!\!,\gamma}_{T}(\omega)$ satisfies the non-asymptotic concentration rates in \eqref{eq:discounted_finite1}--\eqref{eq:discounted_finite4}, where in the LHS, $V^{\pi}_{\gamma}(s)$ is replaced with $V^{*}_{\gamma}(s)$ and in the statement and RHS, $K^{\pi,\gamma}$ is replaced with $K^{*,\gamma}$.
\end{corollary}
\begin{proof}
    Since $\pi^{*}$ is in $\Pi_{\SD}$, by Theorem~\ref{thm:discounted_non_asymp} and Corollary~\ref{cor:discounted_non_asymp2}, the optimal policy satisfies the non-asymptotic concentration rates in \eqref{eq:discounted_finite1}--\eqref{eq:discounted_finite4}. 
\end{proof}
\subsection{Sample Path Behavior of Performance Difference of Two Stationary Policies}
As an implication of the results presented in the Sec.~\ref{sec:dis_any_policy}, we characterize the sample path
behavior of the difference in discounted cumulative rewards between any two stationary policies. As
a consequence, we derive the non-asymptotic concentration of the difference in rewards
between any two optimal policies. These concentration bounds are presented in the following
two corollaries. 

\begin{corollary}\label{cor:two_policies_discounted}
    Consider two policies $\pi_{1},\pi_{2} \in \Pi_{\SD}$. Let $\{S_{t}^{\pi_{1}}\}_{t\geq0}$ and $\{S_{t}^{\pi_{2}}\}_{t\geq0}$ denote the random sequences of the states encountered by policy $\pi_{1}$ and $\pi_{2}$ respectively. Following upper-bounds hold for the difference between the discounted cumulative reward received by the two policies.  
    \begin{enumerate}
        \item For any $\delta \in (0,1)$, with probability at least $1-\delta$, we have
        \begin{align}\label{eq:two_policy_1_dis}
         &\phantom{\leq}\Big|\big|R^{\pi_{1},\gamma}_{T}-R^{\pi_{2},\gamma}_{T}\big| - \big|\big[V^{\pi_{1}}_{\gamma}(S_{0}^{\pi_1}) - \gamma^{T}V^{\pi_{1}}_{\gamma}(S_{T}^{\pi_1}) \big]-\big[V^{\pi_{2}}_{\gamma}(S_{0}^{\pi_2}) - \gamma^{T}V^{\pi_{2}}_{\gamma}(S_{T}^{\pi_2}) \big] \big| \Big|\notag \\&\leq K^{\pi_{1},\gamma}\sqrt{2f^{\gamma}(T)\log\frac{4}{\delta}} + K^{\pi_{2},\gamma}\sqrt{2f^{\gamma}(T)\log\frac{4}{\delta}}.
        \end{align}
        \item For any $\delta \in (0,1)$, if $\lim_{T' \to \infty}f^{\gamma}(T') > \dfrac{173}{K^{\pi,\gamma}}\log\dfrac{4}{\delta}$, define $T_{0}^{\pi_{i}}(\frac{\delta}{2})$ as
        \begin{equation}\label{eq:def_T_pi}
        T_{0}^{\pi_{i}}(\frac{\delta}{2})\coloneqq \min\Big\{T'\geq 1 : f^{\gamma}(T') > \dfrac{173}{K^{\pi_{i},\gamma}}\log\dfrac{8}{\delta}\Big\}, \quad  i \in \{1,2\}.  
        \end{equation}
         Then, for all $T \geq T^\pi_{0}(\delta) \coloneqq  \max\Big\{T_{0}^{\pi_{1}}(\frac{\delta}{2}),T_{0}^{\pi_{2}}(\frac{\delta}{2})\Big\} $, with probability at least $1-\delta$, we have
        \begin{align}\label{eq:two_policy_2_dis}
             &\Big|\big|R^{\pi_{1},\gamma}_{T}-R^{\pi_{2},\gamma}_{T}\big| - \big|\big[V^{\pi_{1}}_{\gamma}(S_{0}^{\pi_1}) - \gamma^{T}V^{\pi_{1}}_{\gamma}(S_{T}^{\pi_1}) \big]-\big[V^{\pi_{2}}_{\gamma}(S_{0}^{\pi_2}) - \gamma^{T}V^{\pi_{2}}_{\gamma}(S_{T}^{\pi_2}) \big] \big| \Big| \notag  \notag\\ \leq
            &\max\Big\{K^{\pi_{1},\gamma}\sqrt{3f^{\gamma}(T)\Big(2\log\log\frac{3}{2}f^{\gamma}(T) + \log\frac{4}{\delta}\Big)}, (K^{\pi_{1},\gamma})^2\Big\} \notag\\ 
            + &\max\Big\{K^{\pi_{2},\gamma}\sqrt{3f^{\gamma}(T)\Big(2\log\log\frac{3}{2}f^{\gamma}(T) + \log\frac{4}{\delta}\Big)}, (K^{\pi_{2},\gamma})^2\Big\}.
        \end{align}
    \end{enumerate}
\end{corollary}
The proof if presented in App.~\ref{sec:pf_cor_two_policies_discounted}.

\begin{corollary}\label{cor:dis_two_optimal}
    Consider two optimal policies $\pi^{*}_{1},\pi^{*}_{2}\in \Pi_{\SD}$. Let $\{S_{t}^{\pi_{1}^{*}}\}_{t\geq0}$ and $\{S_{t}^{\pi_{2}^{*}}\}_{t\geq0}$ denote the random sequences of states encountered by optimal policies $\pi_{1}^{*}$ and $\pi_{2}^{*}$. To simplify the expression, we assume the system starts at a fixed initial state, i.e., $S_{0}^{\pi_{1}^{*}} = S_{0}^{\pi_{2}^{*}}$. Then for the difference between discounted cumulative rewards received by the two optimal policies $\big|R^{\pi^{*}_{1},\gamma}_{T} - R^{\pi^{*}_{2},\gamma}_{T} \big|$, we have:
    \begin{enumerate}
        \item For any $\delta \in (0,1)$, with probability at least $1-\delta$, we have
        \begin{align}\label{eq:two_policy_3_dis}
        \Big|\big|R^{\pi^{*}_{1},\gamma}_{T}-R^{\pi^{*}_{2},\gamma}_{T} \big| - \gamma^{T}\big|V_{\gamma}^{*}(S_{T}^{\pi_{2}^{*}}) - V_{\gamma}^{*}(S_{T}^{\pi_{1}^{*}}) \big|\Big| \leq  2\Big(K^{*,\gamma}\sqrt{2f^{\gamma}(T)\log\frac{4}{\delta}} \Big).
        \end{align}
        \item Consider $T_{0}^{\pi^{*}}\!(\frac{\delta}{2})$ defined in \eqref{eq:def_T_pi}. For any $\delta \in (0,1)$, for all $T\geq T_{0}^{\pi^{*}}\!(\frac{\delta}{2})$, with probability at least $1-\delta$, we have
        \begin{align}\label{eq:two_policy_4_dis}
        &\Big|\big|R^{\pi^{*}_{1},\gamma}_{T}-R^{\pi^{*}_{2},\gamma}_{T} \big| - \gamma^{T}\big|V_{\gamma}^{*}(S_{T}^{\pi_{2}^{*}}) - V_{\gamma}^{*}(S_{T}^{\pi_{1}^{*}}) \big|\Big| \notag \\ &\leq 
        2\Big(\max\Big\{K^{*,\gamma}\sqrt{3f^{\gamma}(T)\Big(2\log\log\big(\frac{3}{2}f^{\gamma}(T)\big) + \log\frac{4}{\delta}\Big)}, (K^{*,\gamma})^2\Big\} \Big).
        \end{align}
    \end{enumerate}
\end{corollary}
\begin{proof}
    Since both policies $\pi^{*}_{1}, \pi^{*}_{2} \in \Pi_{\SD}$ are optimal policies, by the definition, we have
    \[
    V^{\pi^{*}_{1}}_{\gamma}(s) = V^{\pi^{*}_{2}}_{\gamma}(s) = V^{*}_{\gamma}(s), \quad \forall s \in \mathcal{S}, \hspace{2.0mm} \forall \gamma \in (0,1).
    \]
    As a result, by the assumption that $S_{0}^{\pi_{1}^{*}} = S_{0}^{\pi_{2}^{*}}$ we have
    \[
    \Big| V^{*}_{\gamma}(S_{0}^{\pi_{1}^{*}}) - V^{*}_{\gamma}(S_{0}^{\pi_{2}^{*}}) \Big| = 0. 
    \]
    In addition, we have
    \[
    K^{\pi_{1}^{*},\gamma} = K^{\pi_{2}^{*},\gamma} = K^{*,\gamma}, \quad  \forall \gamma \in (0,1).
    \] 
    As a result, by Corollary~\ref{cor:two_policies_discounted}, the difference $\big|R^{\pi^{*}_{1},\gamma}_{T} - R^{\pi^{*}_{2},\gamma}_{T} \big|$ satisfies the non-asymptotic concentration rates in Corollary~\ref{cor:two_policies_discounted} with the RHS of \eqref{eq:two_policy_1_dis} and \eqref{eq:two_policy_2_dis} being simplified to RHS of \eqref{eq:two_policy_3_dis}--\eqref{eq:two_policy_4_dis}.  
\end{proof}
\subsection{Vanishing Discount Analysis}
In order to observe the connection between the upper-bounds established in  Theorem~\ref{thm:finite_return_meta_new} and Theorem~\ref{thm:discounted_non_asymp},
we investigate the asymptotic behavior of these two upper-bounds as the discount factor $\gamma$ goes to $1$ from below (i.e., $\gamma \toup 1$). This characterization is stated in the following Corollary.  
\begin{corollary}\label{cor:beta_asymptotic}
    For any policy $\pi \in \Pi_{\AC}$,
    we have the following asymptotic relations between the bounds in Theorem~\ref{thm:finite_return_meta_new} and Theorem~\ref{thm:discounted_non_asymp}.
    \begin{enumerate}
        \item As $\gamma$ goes to $1$ from below, the quantity in the LHS of \eqref{eq:discounted_finite1}--\eqref{eq:discounted_finite2} converges to the LHS of \eqref{eq:clt_finite_new}, i.e.,
        \[
        \lim_{\gamma \toup 1} \Big|R^{\pi,\gamma}_{T} - \big(V^{\pi}_{\gamma}(S_{0})-\gamma^{T}V^{\pi}_{\gamma}(S_{T})\big)\Big| = \Big|R^{\pi}_{T} - TJ^{\pi} + \big(V^{\pi}(S_{0})-V^{\pi}(S_{T})\big) \Big|.
        \]
        \item As $\gamma$ goes to $1$ from below, the RHS in \eqref{eq:discounted_finite1} converges to the RHS in \eqref{eq:clt_finite_new}, i.e.,
        \[
        \lim_{\gamma \toup 1} \bigg[K^{\pi,\gamma}\sqrt{2f^{\gamma}(T)\log\frac{2}{\delta}}\bigg] =  K^{\pi}\sqrt{2T\log\frac{2}{\delta}} .
        \]
        \item As $\gamma$ goes to $1$ from below, the RHS in \eqref{eq:discounted_finite2} converges to the RHS in \eqref{eq:lil_finite_new}, i.e.,
        \begin{align*}
        \lim_{\gamma \toup 1} \bigg[&\max\bigg\{ K^{\pi,\gamma} \sqrt{3f^{\gamma}(T)\big(2\log\log (\frac{3}{2}f^{\gamma}(T)) +\log\frac{2}{\delta}\big)},(K^{\pi,\gamma})^{2}\bigg\}\bigg] \\
        = &\max\Bigg\{K^{\pi}\sqrt{3T\Big(2\log\log\frac{3T}{2} + \log\frac{2}{\delta}\Big)}, (K^{\pi})^2\Bigg\}  .
        \end{align*}
    \end{enumerate}
\end{corollary}
Proof is presented in App.~\ref{app:pf_cor_beta_asymptotic}.
\begin{remark}
The non-asymptotic characterizations are established in Theorem~\ref{thm:discounted_non_asymp}.
Since the discounted cumulative reward $R^{\pi,\gamma}_{T}$ is finite for $\mathcal{M}$, we cannot provide any asymptotic characterization for this quantity. 
However, Corollary~\ref{cor:beta_asymptotic} shows that as the discount factor $\gamma$ goes to $1$ from below, the non-asymptotic concentration behavior of $R^{\pi,\gamma}_{T}$ resembles the non-asymptotic concentration of $R^{\pi}_{T}$. This gives a complete picture of concentration rate of $R^{\pi,\gamma}_{T}$ and $R^{\pi}_{T}$.
\end{remark}
\section{Main Results for the Finite-Horizon Setup}\label{sec:finite_horizon}
In this section, we extend the non-asymptotic concentration results that we established for the average reward and discounted reward setups to the case of finite-horizon setup.
\subsection{System Model}
Consider an MDP with state space $\mathcal{S}$ and action space $\mathcal{A}$. Similar to Sec.~\ref{sec:Problem_formulation}, we assume that $\mathcal{S}$ and $\mathcal{A}$ are finite sets. The state evolves in a controlled Markov manner with transition matrix $P$ and at each time $t$, the system yields a per-step reward $r(S_{t},A_{t}) \in [0,R_{\max}]$. Let $h \in \reals$ denote the horizon of the problem. The definitions of policy and policy set $\Pi$ are similar to Sec.~\ref{sec:Problem_formulation}. 
\begin{definition}
Given a model $\mathcal{M} = (P,r,h)$, define $\Pi_{\FD}$ to be the set of finite-horizon deterministic policies, i.e., for any $\pi = (\pi_{0},\pi_{1},\ldots,\pi_{h}) \in \Pi_{\FD}$, we have $\pi_{t}:\mathcal{S}\to \mathcal{A}$ $($i.e., $A_{t} = \pi_{t}(S_{t}))$, but $\pi_{t}$ may depend upon $t$.    
\end{definition}
The cumulative reward received by any policy $\pi \in \Pi$ up to time $T$ ($T$ is not necessarily equal to $h$) is given by 
\[
R_{T}^{\pi,h}(\omega) \coloneqq \sum_{t=0}^{T-1}r(S_{t},A_{t}), \quad \text{where, } A_{t} = \pi(S_{0:t},A_{0:t-1}), \quad \omega \in \Omega, \quad T \leq h+1.
\]
Note that $R_{T}^{\pi,h}(\omega)$ is a random variable. For this model, the expected total reward of any policy $\pi \in \Pi$ starting at the state $s\in \mathcal{S}$ is defined as 
\[
J^{\pi,h}(s) \coloneqq \EXP^{\pi}\Big[R_{h+1}^{\pi,h} \bigm| S_{0} = s\Big], \quad\forall s \in \mathcal{S},
\]
where $\EXP^{\pi}$ is the expectation with respect to the joint distribution of all the system variables induced by $\pi$. The optimal performance $J^{*,h}(s)$ starting at state $s \in \mathcal{S}$ is defined as
\[
J^{*,h}(s) = \sup_{\pi \in \Pi} J^{\pi,h}(s), \quad \forall s \in \mathcal{S}. 
\]
A policy $\pi^{*}$ is called optimal if 
\[
J^{\pi^{*}\!,h}(s) = J^{*,h}(s), \quad \forall s \in \mathcal{S}.
\]
\begin{definition}
    The sequence of finite-horizon optimal value functions $\big\{V_{t}^{*,h}\big\}_{t=0}^{h+1} : \mathcal{S} \to \reals $  is defined as follows
    \[
    V_{h+1}^{*,h}(s) = 0, \quad \forall s \in \mathcal{S},
    \]
    and for $t \in \{h,h-1,\ldots,0\}$, recursively define $V_{t}^{*,h}(s)$ based on the FHDP (Finite-Horizon Dynamic Programming equation) given by 
    \begin{equation}\label{eq:bellman_opt_finite}
    V_{t}^{*,h}(s) = \max_{a \in \mathcal{A}} \Big[r(s,a) + \EXP \big[V_{t+1}^{*,h}(S_{+})\bigm|s,a\big]\Big], \quad \forall s \in \mathcal{S}. \tag{FHDP}  
    \end{equation}
\end{definition}

\begin{definition}
    Given a policy $\pi\in \Pi_{\FD}$, the sequence of finite-horizon value functions $\big\{V_{t}^{\pi,h}\big\}_{t=0}^{h+1} : \mathcal{S} \to \reals $ corresponding to the policy $\pi$ is defined as follows 
    \[
    V_{h+1}^{\pi,h}(s) = 0, \quad \forall s \in \mathcal{S},
    \]
    and for $t \in \{h,h-1,\ldots,0\}$, recursively define $V^{\pi,h}_{t}(s)$ based on the FHPE (Finite-Horizon Policy Evaluation equation) given by
    \begin{equation}\label{eq:finite-h-PE}
        V^{\pi,h}_{t}(s) = r(s,\pi_{t}(s)) + \EXP\big[V_{t+1}^{\pi,h}(S_{+})\bigm|s,\pi_{t}(s)\big], \quad \forall s \in \mathcal{S}.\tag{FHPE}  
    \end{equation}
\end{definition}
\begin{proposition}[\citet{bertsekas2012dynamic1}]
    Let $\pi^{*}=(\pi_{0}^{*},\pi_{1}^{*},\ldots,\pi_{h}^{*}) \in \Pi_{\FD}$ be a policy such that $\pi^{*}_{t}(s_{t})$ denote the argmax of \eqref{eq:bellman_opt_finite} at stage $t$. Then the policy $\pi^{*}$ is optimal, i.e., for all $s\in \mathcal{S}$, $J^{\pi^{*}\!\!,h}(s) = J^{{*},h}(s)$.
\end{proposition}

\subsection{Sample Path Characteristics of Any Policy}\label{sec:finite_h_any_policy}
For any policy $\pi \in \Pi_{\FD}$, we define following statistical properties of the sequence of finite-horizon value functions $\{V^{\pi,h}_{t}\}_{t=0}^{h+1}$.
\begin{enumerate}
    \item Span of the finite-horizon value function $V^{\pi,h}_{t}$ is given by 
    \begin{equation}
        H^{\pi,h}_{t}\coloneqq \spand(V^{\pi,h}_{t}), \quad \forall t\in \{0,1,\ldots,h\}.
    \end{equation}
    \item Maximum absolute deviation of the finite-horizon value function $V^{\pi,h}_{t}$ is given by 
    \begin{equation}\label{eq:def_k_h_finite}
        K^{\pi,h}_{t} \coloneqq \max_{s,s_{+}} \Big|V_{t}^{\pi,h}(s_{+}) - \EXP\big[V_{t}^{\pi,h}(S_{+})\bigm|s,\pi_{t}(s)\big] \Big|, \quad \forall t\in \{0,1,\ldots,h\}.
    \end{equation}
\end{enumerate} 
Similar to the results in Theorem~\ref{thm:finite_return_meta} and Theorem~\ref{thm:discounted_non_asymp} for the average reward and discounted reward setups, we derive non-asymptotic concentration results for the finite-horizon setup. These results are presented in the following theorem. To simplify the notation, let
\begin{equation}\label{eq:def_max_k}
    \bar{K}_{T}^{\pi,h} = \max_{0 \leq  t \leq T}K_{t}^{\pi,h}, \quad \bar{H}_{T}^{\pi,h} = \max_{0 \leq  t \leq T}H_{t}^{\pi,h},
\end{equation}
and let
\begin{equation}\label{eq:def_g}
    g^{\pi,h}(T) \coloneqq\frac{\sum_{t=1}^{T}(K_{t}^{\pi,h})^2}{\big(\bar{K}_{T}^{\pi,h}\big)^2}.
\end{equation}
For any optimal policy $\pi^* \in \Pi_{\FD}$, we denote these corresponding quantities by $H^{*,h}_{t}$, $K^{*,h}_{t}$, $\bar{H}^{*,h}_{T}$, $\bar{K}^{*,h}_{T}$, and $g^{*,h}(T)$.
An immediate implication of the definitions of $R^{\pi,h}_{T}$ and $V_{T}^{\pi,h}(s)$ is that 
\[
\EXP\Big[R^{\pi,h}_{T} + V_{T}^{\pi,h}(S_{T}) - V_{0}^{\pi,h}(S_{0})\Big] = 0.
\]
In this section, we show that with high-probability $R^{\pi,h}_{T}$ concentrates around $V_{0}^{\pi,h}(S_{0}) - V_{T}^{\pi,h}(S_{T})$ and characterize the concentration rate. Following theorem is analogous to the concentration bounds in average reward setup given in
Theorem~\ref{thm:finite_return_meta_new} and concentration bounds in discounted reward setup given in Theorem~\ref{thm:discounted_non_asymp}.
\begin{theorem}\label{thm:finite-h-non-asym}
    For any policy $\pi \in \Pi_{\FD}$, we have:
    \begin{enumerate}
        \item For any $\delta \in (0,1)$, with probability at least $1-\delta$, we have
        \[
        \Big|R^{\pi,h}_{T} - \big(V_{0}^{\pi,h}(S_{0}) - V_{T}^{\pi,h}(S_{T})\big) \Big| \leq \bar{K}_{T}^{\pi,h}\sqrt{2g^{\pi,h}(T)\log\frac{2}{\delta}}.
        \]
        \item For any $\delta \in (0,1)$, if $g^{\pi,h}(h)\geq 173\log\frac{4}{\delta}$, define $T_{0}^{\pi,h}(\delta)$ to be 
        \begin{equation}\label{eq:finite-h-def-T0}
        T_{0}^{\pi,h}(\delta) \coloneqq \min\Big\{T'\geq1 : g^{\pi,h}(T') \geq 173\log\frac{4}{\delta}\Big\}.    
        \end{equation}
        Then with probability at least $1-\delta$, for all $T_{0}^{\pi,h}(\delta) \leq T \leq h+1$, we have
        \begin{align}
        \Big|R^{\pi,h}_{T} &- \big(V_{0}^{\pi,h}(S_{0}) - V_{T}^{\pi,h}(S_{T})\big) \Big| \notag  \\ &\leq \max\bigg\{ \bar{K}^{\pi,h}_{T} \sqrt{3g^{\pi,h}(T)\Big(2\log\log (\frac{3}{2}g^{\pi,h}(T)) +\log\frac{2}{\delta}\Big)},(\bar{K}^{\pi,h}_{T})^{2}\bigg\}.
        \end{align}
    \end{enumerate}
\end{theorem}
The proof is presented in App.~\ref{sec:pf_thm:finite-h-non-asym}.

Following Corollary establishes the finite-time concentration of $R^{\pi,h}_{T}$ around the quantity $V_{0}^{\pi,h}(S_{0})$. This results is analogous to the concentration bounds in the average reward setup given in Theorem~\ref{thm:finite_return_meta} and concentration bounds in the discounted reward setup given in Corollary~\ref{cor:discounted_non_asymp2}.
\begin{corollary}\label{cor:finite-h-non-asym}
    For any policy $\pi \in \Pi_{\FD}$, we have:
    \begin{enumerate}
        \item For any $\delta \in (0,1)$, with probability at least $1-\delta$, we have 
        \[
        \Big|R^{\pi,h}_{T} - V_{0}^{\pi,h}(S_{0}) \Big| \leq \bar{K}_{T}^{\pi,h}\sqrt{2T\log\frac{2}{\delta}} + \bar{H}^{\pi,h}_{T} .
        \]
        \item For any $\delta \in (0,1)$, if $g^{\pi,h}(h)\geq 173\log\frac{4}{\delta}$, define $T_{0}^{\pi,h}(\delta)$ as specified in  \eqref{eq:finite-h-def-T0}. Then with probability at least $1-\delta$, for all $T_{0}^{\pi,h}(\delta) \leq T \leq h+1$, we have
        \[
        \Big|R^{\pi,h}_{T}-V_{0}^{\pi,h}(S_{0}) \Big| \leq \max\bigg\{ \bar{K}^{\pi,h}_{T} \sqrt{3T\big(2\log\log \big(\dfrac{3T}{2}\big) +\log\frac{2}{\delta}\big)},(\bar{K}^{\pi,h}_{T})^{2}\bigg\} + \bar{H}_{T}^{\pi,h}.
        \]
    \end{enumerate}
\end{corollary}
The proof is presented in App.~\ref{pf:cor_finite-h-non-asym}.
\subsection{Sample Path Behavior of Performance Difference of Two Policies}
As an implication of the results presented in Sec.~\ref{sec:finite_h_any_policy}, we characterize the sample path behavior of the difference in cumulative rewards between any two policies. As a consequence, we derive the non-asymptotic concentration of the difference in rewards between any two optimal policies. These concentration bounds are presented in the following two corollaries.
\begin{corollary}\label{cor:finite-h-two-policies}
Consider two policies $\pi_{1},\pi_{2}\in \Pi_{\FD}$. Let $\{S_{t}^{\pi_{1}}\}_{t=0}^{h}$ and $\{S_{t}^{\pi_{2}}\}_{t=0}^{h}$ denote the random sequences of the states encountered by policies $\pi_{1}$ and $\pi_{2}$ respectively. Following upper-bounds hold for the difference between the cumulative reward received by the two policies $\big|R^{\pi_{1},h}_{T}-R^{\pi_{2},h}_{T}\big|$. 
\begin{enumerate}
    \item For any $\delta \in (0,1)$, with probability at least $1-\delta$, we have
    \begin{align}\label{eq:two_policy_finite_h_1}
    &\phantom{\leq}\Big|\big|R^{\pi_{1},h}_{T}-R^{\pi_{2},h}_{T}\big| - \big|\big[V^{\pi_{1},h}_{0}(S_{0}^{\pi_1}) - V^{\pi_{1},h}_{T}(S_{T}^{\pi_1}) \big]-\big[V^{\pi_{2},h}_{0}(S_{0}^{\pi_2}) - V^{\pi_{2},h}_{T}(S_{T}^{\pi_2}) \big] \big| \Big| \notag\\&\leq \bar{K}_{T}^{\pi_{1},h}\sqrt{2g^{\pi_{1},h}(T)\log\frac{4}{\delta}} + \bar{K}_{T}^{\pi_{2},h}\sqrt{2g^{\pi_{2},h}(T)\log\frac{4}{\delta}}.
    \end{align}
    \item For any $\delta \in (0,1)$, if $\min\big\{g^{\pi_{1},h}(h),g^{\pi_{2},h}(h) \big\}\geq 173\log\frac{8}{\delta}$, define $T_{0}^{\pi,h}(\delta)$ as specified in \eqref{eq:finite-h-def-T0} and let
    \[
    T_{0}^{h}(\delta) \coloneqq \max \Big\{T_{0}^{\pi_{1},h}(\frac{\delta}{2}), T_{0}^{\pi_{2},h}(\frac{\delta}{2}) \Big\}.
    \]
    Then, with probability at least $1-\delta$, for all $T_{0}^{h}(\delta) \leq T \leq h+1$, we have
    \begin{align}\label{eq:two_policy_finite_h_2}
    &\phantom{\leq} \Big|\big|R^{\pi_{1},h}_{T}-R^{\pi_{2},h}_{T}\big| - \big|\big[V^{\pi_{1},h}_{0}(S_{0}^{\pi_1}) - V^{\pi_{1},h}_{T}(S_{T}^{\pi_1}) \big]-\big[V^{\pi_{2},h}_{0}(S_{0}^{\pi_2}) - V^{\pi_{2},h}_{T}(S_{T}^{\pi_2}) \big] \big| \Big| \nonumber \\&\leq 
    \max\bigg\{ \bar{K}^{\pi_{1},h}_{T} \sqrt{3g^{\pi_{1},h}(T)\Big(2\log\log (\frac{3}{2}g^{\pi_{1},h}(T)) +\log\frac{4}{\delta}\Big)},(\bar{K}^{\pi_{1},h}_{T})^{2}\bigg\} \nonumber\\
    &+ \max\bigg\{ \bar{K}^{\pi_{2},h}_{T} \sqrt{3g^{\pi_{2},h}(T)\Big(2\log\log (\frac{3}{2}g^{\pi_{2},h}(T)) +\log\frac{4}{\delta}\Big)},(\bar{K}^{\pi_{2},h}_{T})^{2}\bigg\}.
\end{align}
\end{enumerate}
\end{corollary}
The proof is presented in App.~\ref{app:pf_cor:finite-h-two-policies}.
\begin{corollary}\label{cor:finite-h-two-optimal}
    Consider two optimal policies $\pi_{1}^{*},\pi_{2}^{*} \in \Pi_{\FD}$. Let $\{S_{t}^{\pi_{1}^{*}}\}_{t=0}^{h}$ and $\{S_{t}^{\pi_{2}^{*}}\}_{t=0}^{h}$ denote the random sequences of states encountered by optimal policies $\pi_{1}^{*}$ and $\pi_{2}^{*}$. To simplify the expression, we assume the system starts at a fixed initial state, i.e., $S_{0}^{\pi_{1}^{*}} = S_{0}^{\pi_{2}^{*}}$. Then for the difference between the cumulative rewards received by the two optimal policies $\big|R_{T}^{\pi_{1}^{*},h} - R_{T}^{\pi_{2}^{*},h} \big|$, we have:
    \begin{enumerate}
        \item For any $\delta \in (0,1)$, with probability at least $1-\delta$, we have
        \begin{equation}\label{eq:two_policy_finite_h_3}
        \Big|\big|R^{\pi^{*}_{1},h}_{T}-R^{\pi^{*}_{2},h}_{T} \big| - \big|V_{T}^{*,h}(S_{T}^{\pi_{2}^{*}}) - V_{T}^{*,h}(S_{T}^{\pi_{1}^{*}}) \big|\Big| \leq  2\Big(\bar{K}^{*,h}_{T}\sqrt{2g^{*,h}(T)\log\frac{4}{\delta}} \Big).
        \end{equation}
        \item For any $\delta \in (0,1)$, if $g^{*,h}(h)\geq 173\log\frac{4}{\delta}$, define $T_{0}^{\pi^{*}\!\!,h}(\delta)$ as specified in  \eqref{eq:finite-h-def-T0}. Then with probability at least $1-\delta$, for all $T_{0}^{\pi^{*}\!\!,h}(\delta) \leq T \leq h+1$, we have
        \begin{align}\label{eq:two_policy_finite_h_4}
        &\Big|\big|R^{\pi^{*}_{1},h}_{T}-R^{\pi^{*}_{2},h}_{T} \big| - \big|V_{T}^{*,h}(S_{T}^{\pi_{2}^{*}}) - V_{T}^{*,h}(S_{T}^{\pi_{1}^{*}}) \big|\Big| \nonumber\\ \leq  &2\Big(\max\bigg\{ \bar{K}^{*,h}_{T} \sqrt{3g^{*,h}(T)\Big(2\log\log (\frac{3}{2}g^{*,h}(T)) +\log\frac{4}{\delta}\Big)},(\bar{K}^{*,h}_{T})^{2}\bigg\}\Big).
        \end{align}
    \end{enumerate}
\end{corollary}
\begin{proof}
    Since both policies $\pi^{*}_{1},\pi^{*}_{2} \in \Pi_{\FD}$ are optimal policies, by the definition, we have
    \[
    V_{t}^{\pi_{1}^{*},h}(s) = V_{t}^{\pi_{2}^{*},h}(s) = V_{t}^{*,h}(s), \quad \forall s \in \mathcal{S},\quad \forall t \in \{0,1,\ldots,h+1\}.
    \]
    As a result, by the assumption that $S_{0}^{\pi_{1}^{*}} = S_{0}^{\pi_{2}^{*}}$, we have 
    \[
    \Big|V^{*,h}_{0}(S_{0}^{\pi_{1}^{*}}) - V^{*,h}_{0}(S_{0}^{\pi_{2}^{*}})  \Big| = 0.
    \]
    In addition, we have 
    \[
    \bar{K}_{T}^{\pi^{*}_{1},h} = \bar{K}_{T}^{\pi^{*}_{2},h} = \bar{K}_{T}^{*,h} \quad \text{and} \quad g^{\pi^{*}_{1},h}(T) = g^{\pi^{*}_{2},h}(T) = g^{*,h}(T).
    \]
    As a result, by Corollary~\ref{cor:finite-h-two-policies}, the difference $\big|R_{T}^{\pi_{1}^{*},h} - R_{T}^{\pi_{2}^{*},h} \big|$ satisfies the non-asymptotic concentration rates in Corollary~\ref{cor:finite-h-two-policies} with the RHS of \eqref{eq:two_policy_finite_h_1}--\eqref{eq:two_policy_finite_h_2} being simplified to RHS of \eqref{eq:two_policy_finite_h_3}--\eqref{eq:two_policy_finite_h_4}.
\end{proof}

\section{Extension of the Results to Random Reward}\label{sec:stoch_reward}
In this section, we extend the results of Sec.~\ref{sec:Problem_formulation} to settings where the reward at time $t$ depends not only on the state-action pair $(s,a)$ but also on an exogenous process.
Consider an average reward MDP with state space $\mathcal{S}$ and action space $\mathcal{A}$. Similar to Sec.~\ref{sec:Problem_formulation}, we assume that $\mathcal{S}$ and $\mathcal{A}$ are finite sets. The state evolves in a controlled Markov manner with transition matrix $P$. Let $\{E_{t}\}_{t\geq0}$ denote an exogenous process which satisfies the following property 
\begin{equation} \label{eq:E_condition}
    \EXP\big[E_{t}\big|S_{0:t},A_{0:t}\big] = \EXP\big[E_{t}\big|S_{t},A_{t}\big], \quad \forall t\geq0.   
\end{equation}At each time $t$, the system yields a per-step random reward $\tilde{r}(S_{t},A_{t},E_{t})$, where $\tilde{r}: \mathcal{S} \times \mathcal{A} \times \mathcal{E} \to [0,R_{\max}]$. We use the notation $\tilde{r}$ to denote the new definition of reward function distinguishing it from $r$ in Sec.~\ref{sec:Problem_formulation}, and denote this model by $\tilde{M} = (P,\tilde{r})$. Let $\tilde{R}^{\pi}_{T}$ denote the total reward received by policy $\pi$ until time $T$, i.e., 
\[
\tilde{R}^{\pi}_{T} = \sum_{t=0}^{T-1}\tilde{r}(S_{t},A_{t},E_{t}), \quad A_{t} \sim \pi(S_{0:t},A_{0:t-1}).
\]
Compared to the model in Sec.~\ref{sec:Problem_formulation}, the only change in this model is the use of $\tilde r(S_t, A_t, E_t)$ as the per-step reward, where the process $\{E_t\}_{t\geq0}$ satisfies the conditional independence property in \eqref{eq:E_condition}. As a result, by defining reward function $r(s,a)$ as 
\begin{equation}\label{eq:r_reduction}
    r(s,a) = \EXP\big[ \tilde{r} (S_t, A_t, E_t) \big| S_t = s, A_t = a\big],   
\end{equation}
model $\tilde{\mathcal{M}} = (P,\tilde{r})$ reduces to the model $ \mathcal{M} = (P,r)$ in Sec.~\ref{sec:Problem_formulation}. Therefore by the results in Sec.~\ref{sec:Main_results}, we can establish the concentration behavior of the process $R^{\pi}_T$, where 
\[
    R^{\pi}_T = \sum_{t=0}^{T-1} r(S_t, A_t) = \sum_{t=0}^{T-1} \EXP\big[\tilde{r}(S_t, A_t,E_{t})\big|S_{t},A_{t}=\pi(S_{t})\big].
\]

However, in this setting, we are interested in the concentration of cumulative reward process 
\(
\tilde R^{\pi}_T = \sum_{t=0}^{T-1} \tilde r(S_t, A_t, E_t).
\)
To simplify the analysis, we define the quantities associated with the model $\tilde{\mathcal{M}}$ based on the reduced model $\mathcal{M}$.
\begin{definition}
    Let $\mathcal{M} = (P,r)$ with $r(s,a)$ defined in \eqref{eq:r_reduction} be the reduced model of $\tilde{\mathcal{M}} = (P,\tilde{r})$. For the model $\tilde{\mathcal{M}}$,
    we define policies $\pi \in \tilde{\Pi}$, policy sets $\tilde{\Pi}_{\AC}$ and $\tilde{\Pi}_{\SD}$, long-run expected reward function $\tilde{J}^{\pi}$, differential value function $\tilde{V}^{\pi}$, optimal performance $\tilde{J}^{*}$, and optimal policy $\tilde{\pi}^{*}$ as the corresponding quantities of the reduced model $\mathcal{M}$ in Sec.~\ref{sec:Problem_formulation}.    
\end{definition}
For any policy $\pi \in \tilde{\Pi}_{\AC}$, we define the maximum absolute deviation of value function  $\tilde{K}^{\pi}$ similar to Sec.~\ref{sec:Main_results}, i.e., 
    \begin{align}
        \tilde{K}^{\pi} &\coloneqq \max_{s,s_{+} \in \mathcal{S}}\!\Big|\tilde{V}^{\pi}(s_{+}) - \EXP\big[\tilde{V}^{\pi}(S_{+})\big|s,\pi(s)\big]\Big|.\label{eq:K_tilde}
    \end{align}
In this section, we define a similar quantity for the reward function $\tilde{r}(S_{t},A_{t},E_{t})$.
\begin{definition}\label{def:reduced_model}
    For any policy $\pi \in \tilde{\Pi}_{\AC}$, we define the maximum absolute deviation of reward function $\tilde{K}_{r}^{\pi}$ as follows
    \begin{equation}\label{eq:k_r_tilde}
    \tilde{K}^{\pi}_{r} \coloneqq \max_{\substack{s,s' \in \mathcal{S}\\a'\in \mathcal{A}\\e' \in \mathcal{E}}} \Big|\tilde{r}(s',a',e') - \EXP\big[\tilde{r}(S,A,E)\big|S=s,A=\pi(s)\big] \Big|.   
    \end{equation}
\end{definition}
    The following theorem establishes the concentration of cumulative reward $\tilde{R}^{\pi}_{T}$ around the quantity $T\tilde{J}^{\pi} - \big(\tilde{V}^{\pi}(S_{T})-\tilde{V}^{\pi}(S_{0})\big)$. Compared to the results in Sec.~\ref{sec:Main_results}, the concentration of cumulative reward is characterized by a weaker bound due to the added stochasticity in the cumulative reward process.
\begin{theorem}\label{thm:non-asymp_stoch_reward}
    For any policy $\pi \in \tilde{\Pi}_{\AC}$, the following upper-bounds hold:
    \begin{enumerate}
        \item For any $\delta \in (0,1)$, with probability at least $1-\delta$, we have
        \begin{equation}
            \Big|\tilde{R}^{\pi}_{T} - T\tilde{J}^{\pi} - \big(\tilde{V}^{\pi}(S_{0}) - \tilde{V}^{\pi}(S_{T}) \big) \Big| \leq \tilde{K}^{\pi}\sqrt{2T\log\frac{4}{\delta}} + \tilde{K}^{\pi}_{r}\sqrt{2T\log\frac{4}{\delta}}.
        \end{equation}
        \item For any $\delta \in (0,1)$, for all $T\geq T_{0}(\delta) \coloneqq \max\Big\{\Bigl\lceil \dfrac{173}{\tilde{K}^{\pi}}\log\dfrac{8}{\delta}\Big\rceil , \Bigl\lceil\dfrac{173}{\tilde{K}^{\pi}_{r}}\log\dfrac{8}{\delta}\Big\rceil \Big\}$, with probability at least $1-\delta$, we have
        \begin{align}
            \Big| \tilde{R}^{\pi}_{T} - T\tilde{J}^{\pi} &- \big(\tilde{V}^{\pi}(S_{0}) - \tilde{V}^{\pi}(S_{T}) \big) \Big| 
            \\&\leq 
            \max\Big\{\tilde{K}^{\pi}\sqrt{3T\Big(2\log\log\frac{3T}{2} + \log\frac{4}{\delta}\Big)}, (\tilde{K}^{\pi})^2\Big\}\\
            &+\max\Big\{\tilde{K}^{\pi}_{r}\sqrt{3T\Big(2\log\log\frac{3T}{2} + \log\frac{4}{\delta}\Big)}, (\tilde{K}^{\pi}_{r})^2\Big\}.
        \end{align}
    \end{enumerate}
\end{theorem}
The proof is presented in App.~\ref{sec:pf_thm:non-asymp_stoch_reward}.
\begin{remark}
    By following arguments analogous to those used in this section, similar concentration bounds can be established for the cumulative reward process under any stationary randomized policy. For brevity, we omit these results.  
\end{remark}

\section{Conclusion}\label{sec:conclusion}
In this paper, we investigated the sample path behavior of cumulative reward in Markov decision processes. In particular, we established the asymptotic concentration of rewards, including the law of large numbers, the central limit theorem, and the law of iterated logarithm. Moreover, non-asymptotic concentrations of rewards were obtained, including an Azuma-Hoeffding-type inequality and a non-asymptotic version of the law of iterated logarithm, all applicable to a general class of stationary policies. Using these results, we characterized the relationship between two notions of regret in the literature, cumulative regret and interim cumulative regret. We showed that, in both the asymptotic and non-asymptotic settings, the two definitions are \emph{rate equivalent} as long as either of the regrets is upper-bounded by $\tildeO(\sqrt{T})$. Moreover, we extended our results to three different frameworks: (i) the infinite-horizon discounted reward setting, where we established non-asymptotic concentration of the cumulative discounted reward; (ii) the finite-horizon setting, where we established non-asymptotic concentration of the cumulative total reward; (iii) the infinite-horizon average setting with stochastic reward, where we established the non-asymptotic concentration of the cumulative reward process. Our proof technique in the third extension may also be applied to the discounted reward and finite-horizon frameworks; however, these extensions are omitted for brevity.

The contributions of this work are twofold: (i) It unifies two sets of literature, showing that if an algorithm achieves a regret of $\tildeO(\sqrt{T})$ under one definition, the same rate applies to the other, thereby resolving an existing gap in our theoretical understanding. (ii) The asymptotic and non-asymptotic concentration bounds found in this work can be used to evaluate the probabilistic performance of a policy, allowing for the assessment of risk and safety in the MDP setup. Such assessments may be used in decision making problems involving high-stakes costs including safety-critical engineering systems, finance and healthcare. As a result, we believe our analysis paves the way for risk-aware decision making and reinforcement learning in such applications.

\section{Disclosure of Funding} 
This research was supported in part by Fonds de Recherche du Québec, Nature et Technologies (FRQNT)
Grant 316558 (B. Sayedana), Air Force OSR Grant FA9550-23-1-0015 (P.E. Caines), NSERC Grants RGPIN-2019-0533 (P.E. Caines) and RGPIN-2021-03511 (A. Mahajan), and
Alliance Grant ALLRP 592356 (A. Mahajan).

\section{Acknowledgment}
The authors thank Reza Alvandi for helpful feedback and comments.

\bibliography{IEEEabrv,mybibfile}

\newpage
\appendix
\appendix
\startappendixtoc      
\listofappendices 

\newpage
\section{Background on Markov Chain Theory}\label{sec:solvability_cond}
Consider a time-homogeneous Markov chain defined on a finite state space $\mathcal{S}$. Let $P$ denote the state transition probability and $P^{k}$ denote the $k$-step state transition probability. Then we use the following terminology.
\begin{itemize}
    \item Given $s,s' \in \mathcal{S}$, state $s'$ is said to be \emph{accessible from} $s$, if there exists a finite time $k\geq0$ such that $P^{k}(s'|s) > 0$.
    \item States $s$ and $s'$ in $\mathcal{S}$ are said to \emph{communicate} if $s$ is accessible from $s'$ and $s'$ is accessible from $s$.
    \item Communication relation is reflexive, symmetric, and transitive. Therefore, communication relation is an equivalence relation, and it generates a partition of the state space $\mathcal{S}$ into disjoint equivalence classes called \emph{communication classes} \citep{bremaud2013markov}.
    \item Let $T_{s}$ denote the hitting time of state $s$. State $s$ is called \emph{recurrent} if 
    \[
    \PR\big(T_{s} < \infty \bigm| S_{0} = s\big) = 1,
    \]
    and otherwise it is called \emph{transient}. 
    \item A \emph{recurrent class} is a communication class where every state within the class is recurrent.
    \item A \emph{transient class} is a communication class where every state within the class is transient.
\end{itemize}

\section{Background on Martingales}\label{sec:Martingales}
Let $(\Omega, \mathcal F, \PR)$ be a probability space. A \emph{filtration} $\{\mathcal F_t\}_{t \ge 0}$ is a non-decreasing family of sub-sigma fields of $\mathcal F$. 
A random sequence $\{X_{t}\}_{t\geq0}$ is called \emph{integrable} if  $\EXP[|X_{t}|] < \infty$ for all $t \ge 0$. A random sequence $\{X_{t}\}_{t\geq0}$ is called \emph{adapted} to the filtration $\{\mathcal{F}_{t}\}_{t\geq0}$ if $X_{t}$ is $\mathcal{F}_{t}$-measurable for all $t\geq0$. 
\begin{definition}
An integrable  sequence $\{X_{t}\}_{t\geq0}$ adapted to the filtration $\{\mathcal{F}_{t}\}_{t\geq0}$ is called a martingale if 
\[
\EXP[X_{t+1}|\mathcal{F}_{t}] = X_{t}, \quad a.s. \quad \forall t \ge 0.
\]    
\end{definition}

\begin{definition}
    Let $\{c_{t}\}_{t\geq1}$ be a sequence of real numbers and $C$ be a positive real number. 
    A real integrable sequence $\{Y_{t}\}_{t\geq1}$ adapted to the filtration $\{\mathcal{F}_{t}\}_{t\geq0}$ is called:
    \begin{enumerate}
        \item Martingale Difference Sequence (MDS) if 
        \[
        \EXP[Y_{t}|\mathcal{F}_{t-1}] = 0, \quad a.s. \quad \forall t \ge 1.
        \]
        \item Sequentially bounded MDS with respect to the sequence $\{c_{t}\}_{t\geq1}$ if  it is an MDS and 
        \[
        |Y_{t}| \leq c_{t},\quad a.s. \quad \forall t \ge 1.
        \]
        \item Uniformly bounded MDS with respect to the constant $C$ if it is an MDS and 
        \[
        |Y_{t}| \leq C, \quad a.s.\quad \forall t \ge 0.
        \]
    \end{enumerate}
\end{definition}

There is a unique MDS corresponding to a martingale and vise versa. In particular, 
given a martingale $\{X_{t}\}_{t\geq0}$, the corresponding MDS $\{Y_{t}\}_{t\geq1}$ is defined as
\[
Y_{t} \coloneqq X_{t}-X_{t-1}, \quad \forall t\geq1.
\]
Moreover, given an MDS $\{Y_{t}\}_{t\geq1}$, the corresponding martingale sequence $\{X_{t}\}_{t\geq0}$ is defined as
\[
X_{0} = 0 ,\quad X_{T} =\sum_{t=1}^{T}Y_{t}, \quad \forall T \ge 1.  
\]
Consider a martingale $\{X_t\}_{t \ge 0}$ such that $\{X_t^2\}_{t \ge 0}$ is integrable.  The \emph{increasing process} $\{A_{t}\}_{t\geq1}$ associated with the sequence $\{X^{2}_{t}\}_{t\geq0}$ is defined as
\[
A_{1} = \EXP[X^2_{1}|\mathcal{F}_{0}]-X^2_{1}, \quad A_{t}  = \EXP[X^2_{t}|\mathcal{F}_{t-1}]-X^2_{t-1} + A_{t-1},\quad \forall t\geq2.
\]
Let $\{Y_t\}_{t \ge 1}$ be the MDS corresponding to $\{X_t\}_{t \ge 0}$. Then, we can express $\{A_t\}_{t \ge 1}$ in terms of $\{Y^2_t\}_{t \ge 1}$. In particular, we have
\begin{align*}
    A_{t} &=  \EXP[X^2_{t}|\mathcal{F}_{t-1}]-X^2_{t-1} + A_{t-1}  \\ &=  \EXP[X^2_{t-1}|\mathcal{F}_{t-1}] + 2\EXP[Y_{t}|\mathcal{F}_{t-1}]X_{t-1} + \EXP[Y^2_{t}|\mathcal{F}_{t-1}] - X^2_{t-1} + A_{t-1} \\ &=  \EXP[Y^2_{t}|\mathcal{F}_{t-1}] + A_{t-1}.    
\end{align*}
As a result, we have
\[
A_{T} = \sum_{t=1}^{T}\EXP[Y^2_{t}|\mathcal{F}_{t-1}], \quad \forall T\geq1.
\]
Therefore, we sometimes say that $\{A_t\}_{t \ge 1}$ is the increasing sequence associated with $\{Y_t^2\}_{t \ge 1}$. 

Martingale sequences are an important class of stochastic processes. Both asymptotic and non-asymptotic concentration of martingale sequences have been well studied. In Sec.~\ref{sec:asym_mart} and \ref{sec:non_asym_mart}, we present the asymptotic and non-asymptotic concentration characteristics of martingales with bounded MDS. 

\subsection{Asymptotic Concentration}\label{sec:asym_mart}

\subsubsection{Strong Law of Large Numbers}
The first asymptotic results presented in this section is a version of Strong Law of Large numbers for martingale difference sequences.
\begin{theorem}[{see~\citep[Theorem~3.3.1]{stout1974almost}}]\label{thm:SLLN_MDS}
    Let $\{Y_{t}\}_{t\geq1}$ be an MDS and $\{a_{t}\}_{t\geq1}$ be a sequence of positive and  $\mathcal{F}_{t-1}$-measurable real numbers such that $\lim\limits_{t \to \infty}a_{t} = \infty$.
    If for some $0 < p \leq 2$, we have: 
    \[
        \sum_{t=1}^{\infty}\frac{\EXP\big(|Y_{t}|^{p}|\mathcal{F}_{t-1}\big)}{a_{t}^{p}} < \infty.
    \]
    Then:
    \[
    \frac{\sum_{t=t}^{T}Y_{t}}{T} \to 0, \quad \text{a.s.}
    \]
\end{theorem}
\subsubsection{Central Limit Theorem}
Following theorem characterizes a version of Central Limit Theorem for martingale sequences with corresponding bounded MDS.
\begin{theorem}[{see~\citep[Theorem~35.11]{billingsley2013convergence}}]\label{thm:CLT_MDS} 
    Let $\{Y_{t}\}_{t\geq1}$ be a sequentially bounded MDS with respect to the sequence $\{c_{t}\}_{t\geq1}$. Let
   $\{A_{t}\}_{t\geq1}$ be the increasing process associated with $\{Y_{t}^{2}\}_{t\geq1}$, i.e.
   \[
    A_{T} = \sum_{t=1}^{T}\EXP[Y^2_{t}|\mathcal{F}_{t-1}], \quad \forall T\geq1.
    \]
    Define the stopping time $\nu_t$ as
    \[
    \nu_{t} \coloneqq \min \big\{T\geq1: A_{T} \geq t \big\}.
    \]
    Let $ \Omega_{0} = \{\omega \in \Omega : \lim_{T \to \infty} A_{T} = \infty\}$.
    If $\PR(\Omega_{0}) = 1$, then
    \[
    \frac{1}{\sqrt{T}} 
    \sum_{t=1}^{\nu_{T}}Y_{t} \xrightarrow[]{(d)} \mathcal{N}(0,1).
    \]
\end{theorem}
\subsubsection{Law of Iterated Logarithm}
Following theorem characterizes a version of Law of Iterated Logarithm for uniformly bounded MDS.
\begin{theorem}[{see~\citep[Proposition~VII-2-7]{neveu1975discrete}}]\label{thm:LIL_MDS}
   Let $\{Y_{t}\}_{t\geq1}$ be a uniformly bounded MDS with respect to the constant $C$.
   Furthermore, let $\{A_{t}\}_{t\geq1}$ and $\Omega_{0}$ be as defined in Theorem~\ref{thm:CLT_MDS}. Then, for almost all $\omega \in \Omega_{0}$, we have
   \begin{align*}
   \liminf_{T \to \infty} \frac{\sum_{t=1}^{T}Y_{t}}{\sqrt{2A_{T}\log\log A_{T}}} = -1,\quad
   \limsup_{T \to \infty} \frac{\sum_{t=1}^{T}Y_{t}}{\sqrt{2A_{T}\log\log A_{T}}} = 1.
   \end{align*}
\end{theorem}

Non-asymptotic high-probability bounds with similar functional dependence on the horizon $T$ also exist for martingales. These bounds are presented in Sec.~\ref{sec:non_asym_mart}.

\subsection{Non-Asymptotic Concentration}\label{sec:non_asym_mart}
\subsubsection{Azuma-Hoeffding Inequality}
A famous non-asymptotic concentration for martingale sequences is Azuma-Hoeffding inequality.
\begin{theorem}[{see~\citep[Theorem~2.2.1]{raginsky2013concentration}}]\label{thm:Azuma} 
Let $\{Y_{t}\}_{t\geq1}$ be a sequentially bounded MDS with respect to the sequence $\{c_{t}\}_{t\geq1}$.
Then for all $T\geq1$ and for all $\epsilon > 0$, we have
\[
\PR \bigg(\bigg|\sum_{t=1}^{T}Y_{t}\bigg|\geq \epsilon\Bigg) \leq 2\exp\bigg(\frac{-\epsilon^2}{ 2\sum_{t=1}^{T}c_{t}^{2}}\bigg).
\]
\end{theorem}
By rewriting the statement of Theorem~\ref{thm:Azuma}, we get following equivalent form of Azuma-Hoeffding inequality.
\begin{corollary}\label{cor:azuma}
We have following statements
\begin{enumerate}
    \item Let $\{Y_{t}\}_{t\geq1}$ be a sequentially bounded MDS with respect to the sequence $\{c_{t}\}_{t\geq1}$. For any $\delta \in (0,1)$, with probability at least $1-\delta$, we have
    \[
    \bigg|\sum_{t=1}^{T}Y_{t}\bigg| \leq \sqrt{2\sum_{t=1}^{T}c_{t}^{2}\log\!\dfrac{2}{\delta}}.
    \]
    \item Let $\{Y_{t}\}_{t\geq1}$ be a uniformly bounded MDS with respect to the constant $C$. For any $\delta \in (0,1)$, with probability at least $1-\delta$, we have
    \[
    \bigg|\sum_{t=1}^{T}Y_{t}\bigg| \leq C\sqrt{2T\log\!\dfrac{2}{\delta}}.
    \]
\end{enumerate}
\end{corollary}
The proof of Part 1 follows by equating the RHS of Theorem~\ref{thm:Azuma} to $\delta$ and solving for $\epsilon$. The proof of Part 2 follows by substituting the sequence $\{c_{t}\}_{t\geq1}$ with the constant $C$ in the RHS of Part 1. 

\subsubsection{Non-Asymptotic Law of Iterated Logarithm}
Following result is a finite-time analogue of  Law of Iterated Logarithm. This result shows that for a large enough horizon $T$, the growth rate of a martingale sequence is of the order $\OO\Big(\sqrt{T\log\log(T)} \Big)$ with high probability. 
\begin{theorem}[{see~\citep[Theorem~4]{balsubramani2014sharp}}]\label{thm:non_asym_LIL}
Let $\{Y_t\}_{t\geq1}$ be a sequentially bounded MDS with respect to the sequence $\{c_{t}\}_{t\geq1}$.
For any $\delta \in (0,1)$, for all $T\geq T_{0}(\delta)\coloneqq \min\Big\{T\hspace{0.5ex}:\hspace{0.5ex} \sum_{t=1}^{T}c_{t}^{2}\geq 173\log \dfrac{4}{\delta}\Big\}$, with probability at least $1-\delta$, we have
\begin{equation}\label{eq:mid_thm_non1}
    \Big|\sum_{t=1}^{T}Y_{t}\Big| \leq \sqrt{3\Big(\sum_{t=1}^{T}\!c_{t}^{2} \Big) \Big(2\log\log\!\dfrac{3\sum_{t=1}^{T}c_{t}^{2}}{2\big|\sum_{t=1}^{T}Y_{t}\big|}  +\log\!\dfrac{2}{\delta} \Big)}.
\end{equation}
\end{theorem}
For the simplicity of the analysis, we state a slightly simplified version of this theorem in the following corollary. 
\begin{corollary}\label{cor:finite_lil}
Let $\{Y_{t}\}_{t\geq1}$ be a uniformly bounded MDS with respect to the constant $C$. For any $\delta \in (0,1)$, for all $T\geq T_{0}(\delta)\coloneqq \Bigl\lceil\dfrac{173}{C}\log\dfrac{4}{\delta}\Bigl\rceil$, with probability at least $1-\delta$, we have
    \begin{equation}\label{eq:mid_thm_non2}
        \Big|\sum_{t=1}^{T}Y_{t}\Big| \leq C\max\bigg\{\sqrt{3T \! \Big(2\log\log\!\frac{3T}{2} \!+\! \log\!\frac{2}{\delta}}\Big),C\bigg\}.    
    \end{equation}
\end{corollary}
\begin{proof}
    This corollary follows from Theorem~\ref{thm:non_asym_LIL}, by substituting the sequence $\{c_{t}\}_{t\geq1}$ with the constant $C$ on the RHS of \eqref{eq:mid_thm_non1}. There are two cases: either $\Big|\sum_{t=1}^{T}Y_{t}\Big| \leq C^{2}$ or $\Big|\sum_{t=1}^{T}Y_{t}\Big| \geq C^{2}$.
    If $\Big|\sum_{t=1}^{T}Y_{t}\Big| \geq C^{2}$, by Theorem~\ref{thm:non_asym_LIL}, with probability at least $1-\delta$, we get:
    \begin{align*}
        \Big|\sum_{t=1}^{T}Y_{t}\Big| \leq C\sqrt{3T \Big(2\log\log\!\dfrac{3TC^{2}}{2\big|\sum_{t=1}^{T}Y_{t}\big|}  +\log\!\dfrac{2}{\delta} \Big)} \leq C\sqrt{3T \! \Big(\!2\log\log\!\frac{3T}{2} \!+\! \log\!\frac{2}{\delta}\Big)}.
    \end{align*}
    Otherwise, we have $\Big|\sum_{t=1}^{T}Y_{t}\Big| \leq C^{2}$. As a result, we can summarize these two cases and get that with probability at least $1-\delta$, we have \begin{equation}
        \Big|\sum_{t=1}^{T}Y_{t}\Big| \leq \max\bigg\{C\sqrt{3T \! \Big(2\log\log\!\frac{3T}{2} \!+\! \log\!\frac{2}{\delta}}\Big),C^2\bigg\}.    
    \end{equation}
\end{proof}

\section{Proof of Main Results for the Average Reward Setup}\label{sec:pf_main}
\subsection{Preliminary Results}
\subsubsection{Martingale Decomposition}
We first present a few preliminary lemmas. 
To simplify the notation, we define following martingale difference sequence. 
\begin{definition}
    Let filtration $\mathcal{F} = \{\mathcal{F}_{t}\}_{t\geq0}$ be defined as $\mathcal{F}_{t} \coloneqq \sigma(S_{0:t},A_{0:t})$.
    For any policy $\pi \in \Pi_{\AC}$, let  $V^{\pi}$ denote the corresponding differential value function. 
    We define the sequence $\{M^{\pi}_{t}\}_{t\geq1}$ as follows
    \begin{equation}\label{eq:def_M}
         M^{\pi}_{t} \coloneqq V^{\pi}(S_{t})-\EXP\big[V^{\pi}(S_{t})\bigm|S_{t-1},\pi(S_{t-1})\big], \quad \forall t\geq1,
    \end{equation}
    where $\{S_t\}_{t \ge 0}$ denotes the random sequence of states encountered along the current sample path.
\end{definition}
\begin{lemma}\label{lem:M_be_MDS}
   Sequence $\{M^{\pi}_t\}_{t \ge 1}$ is an MDS. 
\end{lemma}
\begin{proof}
    By the definition of $\{\mathcal{F}_{t}\}_{t\geq0}$, we have that $S_{t-1}$ is $\mathcal{F}_{t-1}$-measurable. As a result, we have
    \begin{align*}
    \EXP\Big[M_{t}^{\pi} \bigm| \mathcal{F}_{t-1}\Big]& = \EXP\Big[V^{\pi}(S_{t})-\EXP\big[V^{\pi}(S_{t})\bigm|S_{t-1},\pi(S_{t-1})\big] \bigm| \mathcal{F}_{t-1}\Big] \\& = \EXP\Big[V^{\pi}(S_{t}) \bigm| \mathcal{F}_{t-1}\Big]-\EXP\Big[V^{\pi}(S_{t}) \bigm| S_{t-1},\pi(S_{t-1})\Big] = 0,
    \end{align*}
    which shows that $\{M_{t}^{\pi}\}_{t\geq0}$ is an MDS with respect to the filtration $\{\mathcal{F}_{t}\}_{t\geq0}$.  
\end{proof}
We now present a martingale decomposition of the cumulative reward $R^{\pi}_{T}(\omega)$. 
\begin{lemma}\label{lem:return_decom}
Given any policy $\pi \in \Pi_{\AC}$, we can rewrite the cumulative reward $R^{\pi}_{T}$ as follows 
\begin{equation}
    R^{\pi}_{T} = TJ^{\pi} + \sum_{t=1}^{T}M^{\pi}_{t} + V^{\pi}(S_{0})-V^{\pi}(S_{T}).
\end{equation} 
\end{lemma}

\begin{proof}
    Since $\pi \in \Pi_{\AC}$, \eqref{eq:AOE} implies that along the trajectory of states $\{S_{t}\}_{t=0}^{T}$ induced by the policy $\pi$, we have 
    \[
    r\big(S_{t},\pi(S_{t})\big) = J^{\pi} + V^{\pi}(S_{t}) - \EXP\Big[V^{\pi}(S_{t+1})\Bigm|S_{t},\pi(S_{t})\Big], \quad \forall t\geq1.
    \]
    As a result, we have
    \begin{align*}
        R^{\pi}_{T} =\phantom{.} &TJ^{\pi} +\sum_{t=0}^{T-1}\Big[V^{\pi}(S_{t}) - \EXP\big[V^{\pi}(S_{t+1})\bigm|S_{t},\pi(S_{t})\big] \Big]\\ \stackrel{(a)}{=} \phantom{.}
        &TJ^{\pi} +\sum_{t=0}^{T-1}\Big[V^{\pi}(S_{t}) - \EXP\big[V^{\pi}(S_{t+1})\bigm|S_{t},\pi(S_{t})\big] \Big] + V^{\pi}(S_{T})-V^{\pi}(S_{T}) \\ \stackrel{(b)}{=} \phantom{.}
        &TJ^{\pi} + \sum_{t=0}^{T-1}\Big[V^{\pi}(S_{t+1}) - \EXP\big[V^{\pi}(S_{t+1})\bigm|S_{t},\pi(S_{t}) \big]\Big] + V^{\pi}(S_{0})-V^{\pi}(S_{T}) \\ \stackrel{(c)}{=} \phantom{.} 
        &TJ^{\pi} + \sum_{t=1}^{T}M^{\pi}_{t} + V^{\pi}(S_{0})-V^{\pi}(S_{T}),
    \end{align*}
    where $(a)$ follows from adding and subtracting $V^{\pi}(S_{T})$, $(b)$ follows from re-arranging the terms in the summation,  and $(c)$ follows from the definition of $\{M_{t}^{\pi}\}_{t\geq0}$ in \eqref{eq:def_M}. 
 
\end{proof}
\subsubsection{A Consequence of The Union Bound}
\begin{lemma}\label{lem:delta}
    Suppose for any $\delta_{1} \in (0,1)$, for all $T\geq T_{1}(\delta_{1})$, with probability at least $1-\delta_{1}$, the random sequence $\{X_{T}\}_{T\geq0}$ satisfies
    \begin{equation*}
    |X_{T}| \leq h_{1}(T,\delta_{1}). \end{equation*}
    Moreover, suppose for any $\delta_{2} \in (0,1)$, for all $T\geq T_{2}(\delta_{2})$, with probability at least $1-\delta_{2}$, the random sequence $\{Y_{T}\}_{T\geq0}$ satisfies
    \begin{equation*}
    |Y_{T}| \leq h_{2}(T,\delta_{2}).    
    \end{equation*}
    Then for any $\delta \in (0,1)$, for all $T\geq T_{0}(\delta) \coloneqq \max\{T_{1}(\frac{\delta}{2}),T_{2}(\frac{\delta}{2})\}$, with probability at least $1-\delta$, the random sequence $\{X_{T}+Y_{T}\}_{T\geq0}$ satisfies
    \[
    |X_{T}+Y_{T}| \leq h_{1}(T,\delta/2) + h_{2}(T,\delta/2).
    \]
\end{lemma}
\begin{proof}
    For a given $\delta \in (0,1)$, by the lemma's assumption, for all $T\geq T_{1}(\delta/2)$, we have
    \begin{equation}\label{eq:mid_delta_1}
    \PR\Big(|X_{T}| > h_{1}(T,\delta/2)\Big) < \frac{\delta}{2}.       
    \end{equation}
    Similarly, we have that for all $T\geq T_{2}(\delta/2)$, we have
    \begin{equation}\label{eq:mid_delta_2}
    \PR\Big(|Y_{T}| > h_{2}(T,\delta/2)\Big) < \frac{\delta}{2}.
    \end{equation} 
    Now $|X_{T}+Y_{T}| \geq h_{1}(T,\delta/2) + h_{2}(T,\delta/2)$ implies that  $|X_{T}| > h_{1}(T,\delta/2)$ or $|Y_{T}| > h_{2}(T,\delta/2)$. As a result, by applying the union bound and \eqref{eq:mid_delta_1}--\eqref{eq:mid_delta_2}, we get 
    \[
    \PR\Big(|X_{T}+Y_{T}| \geq h_{1}(T,\delta/2) + h_{2}(T,\delta/2)\Big) \leq \delta.
    \]
\end{proof}
\subsubsection{Proof of Lemma~\ref{lem:inequalities}}\label{subsec:pf_inequalities}
\paragraph{Proof of Part 1} Recall that for any policy $\pi \in \Pi_{\AC}$, the claim is the following chain of inequalities
\begin{equation}
    \sigma_{\pi}(s) \stackrel{(a)}{\leq} K^{\pi} \stackrel{(b)}{\leq} H^{\pi} \stackrel{(c)}{\leq} \infty, \quad \forall s \in \mathcal{S}.
\end{equation}
\textbf{Proof of Part 1-$\mathbf{(a)}$}:
\noindent By the definition of $K^{\pi}$ in Eq.~\eqref{eq:k}, we have
\[
\Big| V^{\pi}(S_{+})-\EXP\big[V^{\pi}(S_{+}) \bigm| s,\pi(s)\big]\Big| \leq K^{\pi}, \quad \forall s \in \mathcal{S}, \quad \textit{a.s.}
\] 
As a result, we have
\begin{align*}
    &\phantom{=}\EXP\Big[\big(V^{\pi}(S_{+})-\EXP[V^{\pi}(S_{+})\bigm|s,\pi(s)]\big)^2 \bigm | s,\pi(s)\Big] \\ &=
    \sum_{s' \in \mathcal{S}}\big(V^{\pi}(s')-\EXP\big[V^{\pi}(S_{+})\bigm|s,\pi(s))\big]\big)^2P(s'|s,\pi(s)) \leq (K^{\pi})^2, \quad \forall s \in \mathcal{S}.
\end{align*}
\textbf{Proof of Part 1-$\mathbf{(b)}$}: By the definition of expectation operator, we have
\[
\min_{s\in \mathcal{S}} V^{\pi}(s) \leq \EXP[V^{\pi}(S_{+})|s,\pi(s)] \leq  \max_{s\in \mathcal{S}} V^{\pi}(s).
\]
As a result, we have
\begin{equation}\label{eq:mid_lemma_ineq_1}
    V^{\pi}(s) - \EXP[V^{\pi}(S_{+})|s,\pi(s)] \leq V^{\pi}(s) - \min_{s \in \mathcal{S}}V^{\pi}(s) \leq \max_{s\in \mathcal{S}}V^{\pi}(s) - \min_{s \in \mathcal{S}}V^{\pi}(s), \quad \forall s \in \mathcal{S}.    
\end{equation}
Similarly, we have
\begin{equation}\label{eq:mid_lemma_ineq_2}
    \EXP[V^{\pi}(S_{+})|s,\pi(s)] -V^{\pi}(s) \leq \max_{s \in \mathcal{S}}V^{\pi}(s) - V^{\pi}(s) \leq \max_{s\in \mathcal{S}}V^{\pi}(s) - \min_{s \in \mathcal{S}}V^{\pi}(s), \quad \forall s \in \mathcal{S}.   
\end{equation}
Therefore \eqref{eq:mid_lemma_ineq_1}--\eqref{eq:mid_lemma_ineq_2} imply that 
\[
\big| V^{\pi}(S_{+})-\EXP[V^{\pi}(S_{+})\bigm|s,\pi(s)]\big| \leq \spand(V^{\pi}) = H^{\pi}.
\]
\textbf{Proof of Part 1-$\mathbf{(c)}$}: Since policy $\pi \in \Pi_{\AC}$, by \eqref{eq:AOE}, we know $V^{\pi}:\mathcal{S} \to \reals$ is a real-valued function and therefore, $H^{\pi} < \infty$.\\
\paragraph{Proof of Part 2}
We prove that if $\mathcal{M}$ is communicating, then for any policy $\pi \in \Pi_{\AC}$, we have $H^{\pi} \leq D^{\pi}R_{\max}$.
Consider $s,s' \in \mathcal{S}$ where $s \not= s'$. By \citet{puterman2014markov}, we have:
\begin{equation}\label{eq:diam_mid_proof}
    V^{\pi}(s) = \EXP\Big[\sum_{t=0}^{\infty}[r(S_{t},A_{t})-J^{\pi}] \Bigm| S_{0} = s \Big].
\end{equation}
Now consider the stopping time $\tau_{0}$ where $S=s'$ for the first time. We can rewrite $V^{\pi}(s)$ as follows
\begin{align*}
    V^{\pi}(s) &\stackrel{(a)}{=} \EXP\Big[\sum_{t=0}^{\tau_{0}-1}[r(S_{t},A_{t})-J^{\pi}] + \sum_{t=\tau_{0}}^{\infty}[r(S_{t},A_{t})-J^{\pi}] \bigm| S_{0} = s \Big].\\ 
    &\stackrel{(b)}{=} \EXP\Big[\sum_{t=0}^{\tau_{0}-1}[r(S_{t},A_{t})-J^{\pi}] \bigm| S_{0} = s\Big] + \EXP\Big[\sum_{t=\tau_{0}}^{\infty}[r(S_{t},A_{t})-J^{\pi}] \bigm|S_{0} = s \Big] \\
    &\stackrel{(c)}{=} \EXP\Big[\sum_{t=0}^{\tau_{0}-1}[r(S_{t},A_{t})-J^{\pi}] \bigm| S_{0} = s\Big] + \EXP\Big[\sum_{t=\tau_{0}}^{\infty}[r(S_{t},A_{t})-J^{\pi}] \bigm|S_{\tau_{0}} = s' \Big] \\
    &\stackrel{(d)}{=} \EXP\Big[\sum_{t=0}^{\tau_{0}-1}[r(S_{t},A_{t})-J^{\pi}] \bigm| S_{0} = s \Big] + V^{\pi}(s'),
\end{align*}
where $(a)$ follows from splitting the summation with the stopping time $\tau_{0}$; $(b)$ follows from linearity of expectation and the fact that first and second term of RHS of $(b)$ are finite; $(c)$ follows from the strong Markov property and $(d)$ follows from definition of $V^{\pi}(s')$. Therefore, we have
\begin{align*}
    V^{\pi}(s)-V^{\pi}(s')  &= \EXP\Big[\sum_{t=0}^{\tau_{0}-1}[r(S_{t},A_{t})-J^{\pi}]\Big] \leq \EXP\Big[\sum_{t=0}^{\tau_{0}-1}[r(S_{t},A_{t})]\Big]\\&\stackrel{(e)}{\leq} T^{\pi}(s, s')R_{\max} \stackrel{(f)}{\leq} D^{\pi}R_{\max} \stackrel{(g)}{\leq} D_{w}R_{\max} \stackrel{(h)}{<} \infty, 
\end{align*}
where $(e)$ follows from the definition of $T^{\pi}(s, s')$, $(f)$ follows from the definition of $D^{\pi}$, $(g)$ follows from the definition of $D_{w}$, and $(h)$ follows by the fact that $\mathcal{M}$ is communicating. 
Since one can repeat the same argument with any two pairs of $(s,s')$, it implies that $H^{\pi} \leq D^{\pi}R_{\max} \leq D_{w}R_{\max} < \infty $. 

\paragraph{Proof of Part 3}
The result of this part follows from \citet[Theorem~4]{bartlett2012regal}, where it is shown that for weakly communicating $\mathcal{M}$, we have $H^{*} \leq DR_{\max}$.

\subsection{Proof of Theorem~\ref{thm:asymptotic}}\label{sec:pf_thm:asymptotic}
\subsubsection{Proof of Part 1}
By Lemma~\ref{lem:return_decom}, for any policy $\pi \in \Pi_{\AC}$, we can rewrite the cumulative reward $R^{\pi}_{T}$ as follows 
\begin{equation*}
    R^{\pi}_{T} = TJ^{\pi} + \sum_{t=1}^{T}M_{t}^{\pi} + V^{\pi}(S_{0})-V^{\pi}(S_{T}).
\end{equation*}
By \eqref{eq:k} and Lemma~\ref{lem:inequalities}, we have  
\[
\big|M_{t}^{\pi}\big| \leq K^{\pi} < \infty, \quad \forall t\geq1.
\]
Therefore
\[
\sum_{t=1}^{\infty}\frac{(M^{\pi}_{t})^2}{t^2} \leq K^{\pi} \sum_{t=1}^{\infty}\frac{1}{t^2} < \infty.
\]
As a result by choosing $p =2$ and $a_{t} = t$ in Theorem~\ref{thm:SLLN_MDS}, we have
\[
\lim_{T\to \infty}\frac{\sum_{t=1}^{T}M_{t}^{\pi}}{T} = 0, \quad a.s.
\]
Furthermore, Lemma~\ref{lem:inequalities} implies that random variable $V^{\pi}(S_{t})$ has bounded support, therefore, 
\[
\lim_{T \to \infty}\dfrac{V^{\pi}(S_{0})-V^{\pi}(S_{T})}{T} = 0, \quad a.s.
\]
As a result, we have
\[
\lim_{T\to\infty}\frac{R^{\pi}_T}{T} =\lim_{T\to\infty}\frac{\sum_{t=1}^{T}M_{t}^{\pi} + V^{\pi}(S_{0})-V^{\pi}(S_{T}) + TJ^{\pi} }{T} =  J^{\pi}, \quad a.s.
\]
\subsubsection{Proof of Part 2}
To prove this part, we verify the conditions of Theorem~\ref{thm:CLT_MDS} for the MDS $\{M_{t}^{\pi}\}_{t\geq0}$.
By Lemma~\ref{lem:inequalities}, we have 
\[
\big| M_{t}^{\pi} \big| \leq K^{\pi} < \infty, \quad \forall t\geq1.
\]
As a result, the MDS $\{M_{t}^{\pi}\}_{t\geq0}$ is a uniformly bounded MDS with respect to the constant $K^{\pi}$. 
By the theorem's assumption we have $\PR(\Omega_{0}^{\pi}) = 1$, as a result, 
\[
\sum_{t=1}^{\infty}\EXP\Big[\big(M_{t}^{\pi}\big)^2\bigm|\mathcal{F}_{t-1}\Big] = \infty, \quad a.s.
\]
Therefore, for the stopping time $\{\nu_{t}\}_{t\geq0}$ defined in Theorem~\ref{thm:asymptotic}, we have
\begin{equation}\label{eq:mds-clt-mid-1}
    \dfrac{\sum_{t=1}^{\nu_{T}}M_{t}^{\pi}}{\sqrt{T}} \xrightarrow[]{(d)} \mathcal{N}(0,1).
\end{equation}
Since by Lemma~\ref{lem:inequalities}, $V^{\pi}(S_{t})$ has bounded support for all $t\geq1$, we get  
\begin{equation}\label{eq:mds-clt-mid-2}
    \dfrac{V^{\pi}(S_{0})-V^{\pi}(S_{T})}{\sqrt{T}} \to 0,\quad a.s.   
\end{equation}
By combining \eqref{eq:mds-clt-mid-1} and \eqref{eq:mds-clt-mid-2} and by using Theorem~\ref{thm:slutsky}, we get
\begin{equation*}
    \lim_{T \to\infty} \dfrac{R^{\pi}_{\nu_{T}}(\omega) - \nu_{T}J^{\pi}}{\sqrt{T}} \xrightarrow[]{(d)} \mathcal{N}(0,1).    
\end{equation*}
\subsubsection{Proof of Part 3}
We verify the conditions of Theorem~\ref{thm:LIL_MDS} for the MDS $\{M_{t}^{\pi}\}_{t\geq0}$. By Lemma~\ref{lem:inequalities}, we have 
\[
\big| M_{t}^{\pi} \big| \leq K^{\pi} < \infty, \quad \forall t\geq1.
\]
As a result, MDS $\{M_{t}^{\pi}\}_{t\geq0}$ is a uniformly bounded MDS with respect to the constant $K^{\pi}$. 
On the set $\Omega_{0}^{\pi}$, we have
\[
\sum_{t=1}^{\infty}\EXP\Big[\big(M_{t}^{\pi}\big)^2\bigm|\mathcal{F}_{t-1}\Big] = \infty.
\]
As a result, by using the definition of increasing process $\{\Sigma_{t}^{\pi}\}_{t\geq0}$ and Theorem~\ref{thm:LIL_MDS}, we get
 \begin{align}\label{eq:mds-lil-mid-1}
    \liminf_{T \to \infty} \frac{\sum_{t=1}^{T}M_{t}^{\pi}}{\sqrt{2\Sigma^{\pi}_{T}\log\log \Sigma^{\pi}_{T}}} = -1, \quad
    \limsup_{T \to \infty} \frac{\sum_{t=1}^{T}M_{t}^{\pi}}{\sqrt{2\Sigma^{\pi}_{T}\log\log \Sigma^{\pi}_{T}}} = 1. 
 \end{align}
 Since by Lemma~\ref{lem:inequalities}, $V^{\pi}(S_{t})$ has bounded support for all $t\geq1$, we get 
 \begin{equation}\label{eq:mds-lil-mid-2}
     \lim_{T\to\infty}\frac{V^{\pi}(S_{0})-V^{\pi}(S_{T})}{\sqrt{2\Sigma^{\pi}_{T}\log\log \Sigma^{\pi}_{T}}} = 0 , \quad \text{a.s.} 
 \end{equation}
 By combining \eqref{eq:mds-lil-mid-1} and \eqref{eq:mds-lil-mid-2}, we get
\begin{align*}
    \liminf_{T \to \infty} \frac{R_{T}^{\pi}(\omega)-TJ^{\pi}}{\sqrt{2\Sigma^{\pi}_{T}\log\log \Sigma^{\pi}_{T}}} = -1, \quad
    \limsup_{T \to \infty} \frac{R_{T}^{\pi}(\omega)-TJ^{\pi}}{\sqrt{2\Sigma^{\pi}_{T}\log\log \Sigma^{\pi}_{T}}} = 1. 
\end{align*}
\subsection{Proof of Theorem~\ref{thm:finite_return_meta_new}}\label{sec:pf_finite_return_meta_new}
\subsubsection{Proof of Part 1}
By Lemma~\ref{lem:return_decom}, for any policy $\pi \in \Pi_{\AC}$, we can rewrite the cumulative reward $R^{\pi}_{T}(\omega)$ as follows
\begin{align*}
    R^{\pi}_{T}(\omega) = TJ^{\pi} + \sum_{t=1}^{T}M_{t}^{\pi} + V^{\pi}(S_{0})-V^{\pi}(S_{T}).    
\end{align*}
As a result, we have
\begin{align}\label{eq:mds-azuma_mid_1_new}
    \big|R^{\pi}_{T}(\omega) - TJ^{\pi} - \big(V^{\pi}(S_{0})-V^{\pi}(S_{T})\big)\big| &= \big|\sum_{t=1}^{T}M_{t}^{\pi}\big|.
\end{align}
In order to upper-bound the term $\big|\sum_{t=1}^{T}M_{t}^{\pi}\big|$, we verify the conditions of Corollary~\ref{cor:azuma}. By \eqref{eq:k} and Lemma~\ref{lem:inequalities}, we have  
\begin{equation*}
    \big|M_{t}^{\pi}\big|  \leq K^{\pi} < \infty, \quad \forall t\geq1.  
\end{equation*}
As a result, MDS $\{M^{\pi}_{t}\}_{t\geq1}$ is a uniformly bounded MDS with respect to the constant $K^{\pi}$. Therefore, Corollary~\ref{cor:azuma} implies that for any $\delta \in (0,1)$, with probability at least $1-\delta$, we have
\begin{equation}\label{eq:mds-azuma_mid_2_new}
    \Big| \sum_{t=1}^{T}M_{t}^{\pi} \Big| \leq \sqrt{2T(K^{\pi})^2\log(\frac{2}{\delta})}.
\end{equation}
By combining \eqref{eq:mds-azuma_mid_1_new} and \eqref{eq:mds-azuma_mid_2_new}, with probability at least $1-\delta$, we have
\[
\big|R^{\pi}_{T}(\omega) - TJ^{\pi} - \big(V^{\pi}(S_{0})-V^{\pi}(S_{T})\big)\big|  \leq K^{\pi}\sqrt{2T\log\frac{2}{\delta}}.
\]
\subsubsection{Proof of Part 2}
Similar to the proof of Part 1, by lemma \ref{lem:return_decom}, we have
\begin{align}\label{eq:mds_finite_lil_mid_1_new}
    \big|R^{\pi}_{T}(\omega) - TJ^{\pi} - \big(V^{\pi}(S_{0})-V^{\pi}(S_{T})\big)\big| &= \Big|\sum_{t=1}^{T}M_{t}^{\pi}\Big|   
\end{align}
Moreover, MDS $\{M_{t}^{\pi}\}_{t\geq0}$ is a uniformly bounded MDS with respect to the constant $K^{\pi}$. Therefore, Corollary~\ref{cor:finite_lil} implies that for any $\delta \in (0,1)$, for all $T\geq T_{0}(\delta) \coloneqq \Bigl\lceil \dfrac{173}{K^\pi}\log\dfrac{4}{\delta}\Big\rceil$, with probability at least $1-\delta$, we have
\begin{equation}\label{eq:mds_finite_lil_mid_2_new}
    \Big|\sum_{t=1}^{T}M_{t}^{\pi}\Big| \leq \max\Big\{K^{\pi}\sqrt{3T\Big(2\log\log\frac{3T}{2} + \log\frac{2}{\delta}\Big)}, (K^{\pi})^2\Big\}.   
\end{equation}
By combining \eqref{eq:mds_finite_lil_mid_1_new} and \eqref{eq:mds_finite_lil_mid_2_new}, with probability at least $1-\delta$, we have
\begin{equation*}
\big|R^{\pi}_{T}(\omega) - TJ^{\pi} - \big(V^{\pi}(S_{0})-V^{\pi}(S_{T})\big)\big| \leq \max\Big\{K^{\pi}\sqrt{3T\Big(2\log\log\frac{3T}{2} + \log\frac{2}{\delta}\Big)}, (K^{\pi})^2\Big\}.    
\end{equation*}
\subsection{Proof of Theorem~\ref{thm:finite_return_meta}}\label{sec:pf_finite_return_meta}
\subsubsection{Proof of Part 1}
By lemma \ref{lem:return_decom}, for any policy $\pi \in \Pi_{\AC}$, we can rewrite the cumulative reward $R^{\pi}_{T}(\omega)$ as follows
\begin{align*}
    R^{\pi}_{T}(\omega) = TJ^{\pi} + \sum_{t=1}^{T}M_{t}^{\pi} + V^{\pi}(S_{0})-V^{\pi}(S_{T}).    
\end{align*}
As a result, we have
\begin{align}\label{eq:mds-azuma_mid_1}
    \big|R^{\pi}_{T}(\omega) - TJ^{\pi}\big| &= \Big|\sum_{t=1}^{T}M_{t}^{\pi} + V^{\pi}(S_{0})-V^{\pi}(S_{T})\Big|\notag\\
    &\stackrel{(a)}{\leq} \Big|\sum_{t=1}^{T}M_{t}^{\pi}\Big| + \Big|V^{\pi}(S_{0})-V^{\pi}(S_{T})\Big|\notag\\
    &\stackrel{(b)}{\leq} \Big|\sum_{t=1}^{T}M_{t}^{\pi}\Big| + H^{\pi},    
\end{align}
where $(a)$ follows from the triangle inequality and $(b)$ follows from the definition of $H^{\pi}$.
In order to upper-bound the term $\big|\sum_{t=1}^{T}M_{t}^{\pi}\big|$, we verify the conditions of Corollary~\ref{cor:azuma}. By \eqref{eq:k} and Lemma~\ref{lem:inequalities} we have  
\begin{equation*}
    \big|M_{t}^{\pi}\big| \leq K^{\pi} < \infty, \quad \forall t\geq1.  
\end{equation*}
As a result, MDS $\{M^{\pi}_{t}\}_{t\geq1}$ is a uniformly bounded MDS with respect to the constant $K^{\pi}$. Therefore, Corollary~\ref{cor:azuma} implies that for any $\delta \in (0,1)$, with probability at least $1-\delta$, we have
\begin{equation}\label{eq:mds-azuma_mid_2}
    \Big| \sum_{t=1}^{T}M_{t}^{\pi} \Big| \leq \sqrt{2T(K^{\pi})^2\log(\frac{2}{\delta})}.
\end{equation}
By combining \eqref{eq:mds-azuma_mid_1} and \eqref{eq:mds-azuma_mid_2}, with probability at least $1-\delta$, we have
\[
|R^{\pi}_{T}(\omega) - TJ^{\pi}| \leq K^{\pi}\sqrt{2T\log\frac{2}{\delta}} + H^{\pi}.
\]
\subsubsection{Proof of Part 2}
Similar to the proof of Part 1, by lemma \ref{lem:return_decom}, we have
\begin{align}\label{eq:mds_finite_lil_mid_1}
    \big|R^{\pi}_{T}(\omega) - TJ^{\pi}\big| 
    \leq \bigg|\sum_{t=1}^{T}M_{t}^{\pi}\bigg| + H^{\pi}.   
\end{align}
Moreover, MDS $\{M_{t}^{\pi}\}_{t\geq0}$ is a uniformly bounded MDS with respect to the constant $K^{\pi}$. Therefore, Corollary~\ref{cor:finite_lil} implies that for any $\delta \in (0,1)$, for all $T\geq T_{0}(\delta) \coloneqq \Bigl\lceil \dfrac{173}{K^\pi}\log\dfrac{4}{\delta}\Big\rceil$, with probability at least $1-\delta$, we have
\begin{equation}\label{eq:mds_finite_lil_mid_2}
    \Big|\sum_{t=1}^{T}M_{t}^{\pi}\Big| \leq \max\Big\{K^{\pi}\sqrt{3T\Big(2\log\log\frac{3T}{2} + \log\frac{2}{\delta}\Big)}, (K^{\pi})^2\Big\}.   
\end{equation}
By combining \eqref{eq:mds_finite_lil_mid_1} and \eqref{eq:mds_finite_lil_mid_2}, with probability at least $1-\delta$, we have
\begin{equation*}
\big|R^{\pi}_{T}(\omega) - TJ^{\pi}\big| \leq \max\Big\{K^{\pi}\sqrt{3T\Big(2\log\log\frac{3T}{2} + \log\frac{2}{\delta}\Big)}, (K^{\pi})^2\Big\} + H^{\pi}.    
\end{equation*}

\subsection{Proof of Corollary~\ref{cor:finite_return_meta}}\label{app:pf_finite_return_meta}
\subsubsection{Proof of Part 1}
Since $\mathcal{M}$ is communicating, by Lemma~\ref{lem:inequalities}, Part 2, for any policy $\pi \in \Pi_{\AC}$, we have
\begin{equation}
    \big|M_{t}^{\pi}\big|  \leq K^{\pi} \leq D^{\pi}R_{\max}, \quad \forall t\geq1.   
\end{equation}
As a result, the MDS $\{M^{\pi}_{t}\}_{t\geq1}$ is a uniformly bounded MDS with respect to the constant $D^{\pi}R_{\max}$.
Therefore, by repeating the arguments of the proof of Theorem~\ref{thm:finite_return_meta}, Part 1, and substituting $H^{\pi}$ with $D^{\pi}R_{\max}$ in the RHS of \eqref{eq:mds-azuma_mid_1} and replacing  $K^{\pi}$ with $D^{\pi}R_{\max}$ in the RHS of \eqref{eq:mds-azuma_mid_2}, with probability at least $1-\delta$, we have:
\[
    \big|R^{\pi}_{T}(\omega) - T\!J^{\pi}\big| \leq D^{\pi}R_{\max}\sqrt{2T\log\frac{2}{\delta}} + D^{\pi}R_{\max}.   
\]
\subsubsection{Proof of Part 2}
Since $\mathcal{M}$ is communicating, by Lemma~\ref{lem:inequalities}, Part 2 for any policy $\pi \in \Pi_{\AC}$, we have
\begin{equation}
    \big| M_{t}^{\pi} \big| \leq K^{\pi} \leq D^{\pi}R_{\max}, \quad \forall t\geq1.   
\end{equation}
As a result, the MDS $\{M^{\pi}_{t}\}_{t\geq1}$ is a uniformly bounded MDS with respect to the constant $D^{\pi}R_{\max}$.
Therefore, by repeating the arguments of the proof of Theorem~\ref{thm:finite_return_meta}, Part 2, and substituting $H^{\pi}$ with $D^{\pi}R_{\max}$ in the RHS of \eqref{eq:mds_finite_lil_mid_1} and substituting  $K^{\pi}$ with $D^{\pi}R_{\max}$ in the RHS of \eqref{eq:mds_finite_lil_mid_2}, we prove the claim, i.e, for any $\delta \in (0,1)$, for all $T\geq T_{0}(\delta) \coloneqq \Bigl\lceil \dfrac{173}{D^{\pi}R
_{\max}}\log\dfrac{4}{\delta}\Big\rceil$, with probability at least $1-\delta$, we have
\[
\big|R^{\pi}_{T}(\omega) - TJ^{\pi}\big| \leq \max\Big\{D^{\pi}R_{\max}\sqrt{3T\Big(2\log\log\frac{3T}{2} + \log\frac{2}{\delta}\Big)}, (D^{\pi}R_{\max})^2\Big\} + D^{\pi}R_{\max}.
\]
\subsection{Proof of Corollary~\ref{cor:com_optimal}}\label{sec:pf_cor_com_optimal}

By~Prop.~\ref{prop:relation_of_mdps}, Part 3, if $\mathcal{M}$ is communicating, it is also weakly communicating.
In the case of weakly communicating $\mathcal{M}$, by Lemma~\ref{lem:inequalities}, Part 3, for any optimal policy $\pi^{*} \in  \Pi_{\AC}$, we have
\begin{equation}
    \big| M_{t}^{\pi^{*}} \big| = \Big|V^{*}(S_{t})-\EXP\big[V^{*}(S_{t})\bigm|S_{t-1},\pi(S_{t-1})\big]\Big| \leq K^{*} \leq DR_{\max}, \quad \forall t\geq1.   
\end{equation}
As a result, the MDS $\{M_{t}^{\pi^{*}}\}_{t\geq1}$ is uniformly bounded MDS with respect to the constant $DR_{\max}$. Therefore, by repeating the arguments of the proof of Corollary~\ref{cor:finite_return_meta}, Part 1 and Part 2 for the optimal policy $\pi^{*} \in \Pi_{\AC}$, we prove that \(\big|R^{\pi^{*}}_{T}(\omega)-TJ^{*} \big|\) satisfies the non-asymptotic concentration rates in~\eqref{eq:policy_indp_1}--\eqref{eq:policy_indp_2}, where in the RHS, $D^{\pi}$ is replaced with $D$.

\subsection{Proof of Corollary~\ref{cor:two_policies}}\label{sec:pf_cor_two_policies}
\subsubsection{Proof of Part 1}
Consider two policies $\pi_{1},\pi_{2}\in \Pi_{\AC}$. 
Then we have
\begin{align}\label{eq:mid_cor_two_1}
    \big|R^{\pi_{1}}_{T}-R^{\pi_{2}}_{T} \big| =  &\big|R^{\pi_{1}}_{T} - TJ^{\pi_{1}} + TJ^{\pi_{1}} - TJ^{\pi_{2}} + TJ^{\pi_{2}} - R^{\pi_{2}}_{T}\big| \notag \\\stackrel{(a)}{\leq} &\big|R^{\pi_{1}}_{T} - TJ^{\pi_{1}} \big| + \big|TJ^{\pi_{1}} - TJ^{\pi_{2}}\big| + \big|TJ^{\pi_{2}} - R^{\pi_{2}}_{T}\big|,
\end{align}
where $(a)$ follows from the triangle inequality. Similarly, we have 
\begin{align}\label{eq:mid_cor_two_2}
    \big|TJ^{\pi_{1}}-TJ^{\pi_{2}}\big| =  &\big|TJ^{\pi_{1}} - R^{\pi_{1}}_{T} + R^{\pi_{1}}_{T} - R^{\pi_{2}}_{T} + R^{\pi_{2}}_{T} - TJ^{\pi_{2}}\big| \notag \\\stackrel{(b)}{\leq} &\big|TJ^{\pi_{1}}-R^{\pi_{1}}_{T} \big| + \big|R^{\pi_{1}}_{T} - R^{\pi_{2}}_{T}\big| + \big|R^{\pi_{2}}_{T} -TJ^{\pi_{2}} \big|,
\end{align}
where $(b)$ follows from the triangle inequality. \eqref{eq:mid_cor_two_1}--\eqref{eq:mid_cor_two_2} imply that
\begin{align}\label{eq:mid_cor_two_3}
    \Big|\big|R^{\pi_{1}}_{T} - R^{\pi_{2}}_{T}\big| - \big|TJ^{\pi_{1}}-TJ^{\pi_{2}}\big|\Big| \leq  \big|R^{\pi_{1}}_{T} - TJ^{\pi_{1}} \big| + \big|R^{\pi_{2}}_{T} -TJ^{\pi_{2}} \big|.
\end{align}
By Theorem~\ref{thm:finite_return_meta}, we know that for any $\delta_{1} \in (0,1)$, with probability at least $1-\delta_{1}$, we have
\[
\big|R^{\pi_{1}}_{T} - TJ^{\pi_{1}}\big| \leq K^{\pi_{1}}\sqrt{2T\log\frac{2}{\delta_{1}}} + H^{\pi_{1}}.
\]
Similarly, we have that for any $\delta_{2} \in (0,1)$, with probability at least $1-\delta_{2}$, we have
\[
\big|R^{\pi_{2}}_{T} - TJ^{\pi_{2}}\big| \leq K^{\pi_{2}}\sqrt{2T\log\frac{2}{\delta_{2}}} + H^{\pi_{2}}.
\]
As a result, by applying Lemma~\ref{lem:delta} and \eqref{eq:mid_cor_two_3}, for any $\delta \in (0,1)$, with probability at least $1-\delta$, we have
\begin{align*}
    \Big|\big|R^{\pi_{1}}_{T} - R^{\pi_{2}}_{T}\big| - \big|TJ^{\pi_{1}}-TJ^{\pi_{2}}\big|\Big| &\leq \big|R^{\pi_{1}}_{T} - TJ^{\pi_{1}} \big| + \big|TJ^{\pi_{1}} - TJ^{\pi_{2}}\big| \\ &\leq K^{\pi_{1}}\sqrt{2T\log\frac{4}{\delta}} + H^{\pi_{1}} + K^{\pi_{2}}\sqrt{2T\log\frac{4}{\delta}} + H^{\pi_{2}}.
\end{align*}
\subsubsection{Proof of Part 2}
As we showed in the proof of part 1, for any two policies $\pi_{1},\pi_{2}\in \Pi_{\AC}$, we have 
\begin{align*}
    \Big|\big|R^{\pi_{1}}_{T} - R^{\pi_{2}}_{T}\big| - \big|TJ^{\pi_{1}}-TJ^{\pi_{2}}\big|\Big| \leq  \big|R^{\pi_{1}}_{T} - TJ^{\pi_{1}} \big| + \big|R^{\pi_{2}}_{T} -TJ^{\pi_{2}} \big|.
\end{align*}
By Theorem~\ref{thm:finite_return_meta}, we have that for any $\delta_{1} \in (0,1)$, for all $T \geq T^{\pi_{1}}_{0}(\delta) \coloneqq \Bigl\lceil \dfrac{173}{K^{\pi_1}}\log\dfrac{4}{\delta_{1}}\Big\rceil $, with probability at least $1-\delta_{1}$, we have
\begin{equation*}
\big|R^{\pi_{1}}_{T} - T\!J^{\pi_{1}}\big| \leq \max\Big\{K^{\pi_{1}}\sqrt{3T\Big(2\log\log\frac{3T}{2} + \log\frac{2}{\delta_{1}}\Big)}, (K^{\pi_{1}})^2\Big\} + H^{\pi_{1}}.
\end{equation*}
Similarly, we have that for any $\delta_{2} \in (0,1)$, for all $T \geq T^{\pi_{2}}_{0}(\delta) \coloneqq \Bigl\lceil \dfrac{173}{K^{\pi_2}}\log\dfrac{4}{\delta_{2}}\Big\rceil $, with probability at least $1-\delta_{2}$, we have
\begin{equation*}
\big|R^{\pi_{2}}_{T} - T\!J^{\pi_{2}}\big| \leq \max\Big\{K^{\pi_{2}}\sqrt{3T\Big(2\log\log\frac{3T}{2} + \log\frac{2}{\delta_{2}}\Big)}, (K^{\pi_{2}})^2\Big\} + H^{\pi_{2}}.
\end{equation*}
As a result, by applying Lemma~\ref{lem:delta} and \eqref{eq:mid_cor_two_3},
for all $T \geq T_{0}(\delta) \coloneqq \max\Big\{\Bigl\lceil \dfrac{173}{K^{\pi_1}}\log\dfrac{8}{\delta}\Big\rceil\\,\Bigl\lceil \dfrac{173}{K^{\pi_{2}}}\log\dfrac{8}{\delta}\Big\rceil\Big\} $, with probability at least $1-\delta$, we have
\begin{align*}
    \Big|\big|R^{\pi_{1}}_{T} - R^{\pi_{2}}_{T}\big| - \big|TJ^{\pi_{1}}-TJ^{\pi_{2}}\big|\Big| &\leq \big|R^{\pi_{1}}_{T} - TJ^{\pi_{1}} \big| + \big|TJ^{\pi_{2}} - R^{\pi_{2}}_{T}\big|\\
    &\leq \max\Big\{K^{\pi_{1}}\sqrt{3T\Big(2\log\log\frac{3T}{2} + \log\frac{4}{\delta}\Big)}, (K^{\pi_{1}})^2\Big\} + H^{\pi_{1}}\\
    &+\max\Big\{K^{\pi_{2}}\sqrt{3T\Big(2\log\log\frac{3T}{2} + \log\frac{4}{\delta}\Big)}, (K^{\pi_{2}})^2\Big\} + H^{\pi_{2}}. \\
\end{align*}

\subsection{Proof of Theorem~\ref{thm:asymp_D}}\label{sec:pf_symp_D}
By Corollary~\ref{cor:asymptotic_optimal}, for any optimal policy $\pi^{*} \in \Pi_{\AC}$, the quantity $R^{\pi^{*}}$ satisfies the asymptotic concentration rates in \eqref{eq:lln}--\eqref{eq:lil}. On the other hand, by \eqref{eq:D(T)}, for any learning policy $\mu$, we have
\[
\mathcal{D}_{T}(\omega) = R^{\pi^{*}}_{T} - TJ^{*}.
\]
As a result, by substituting $\mathcal{D}_{T}(\omega)$ in the LHS of \eqref{eq:lln}--\eqref{eq:lil}, we get that for any learning policy $\mu$, these asymptotic concentration rates also hold for the difference $\mathcal{D}_{T}(\omega)$ of cumulative regret and interim cumulative regret. 
\subsection{Proof of Theorem~\ref{thm:D_finite}}\label{sec:pf_D_finite}
By Corollary~\ref{cor:non_asymptotic_optimal}, for any optimal policy $\pi^{*} \in \Pi_{\AC}$, the quantity $|R^{\pi^{*}}_{T} - TJ^{*}|$ satisfies the asymptotic concentration rates in \eqref{eq:clt_finite}--\eqref{eq:lil_finite}. On the other hand, by \eqref{eq:D(T)}, for any learning policy $\mu$, we have
\[
\mathcal{D}_{T}(\omega) = R^{\pi^{*}}_{T} - TJ^{*}.
\]
As a result, by substituting $\mathcal{D}_{T}(\omega)$ in the LHS of \eqref{eq:clt_finite}--\eqref{eq:lil_finite}, we get that for any learning policy $\mu$, these non-asymptotic concentration rates also hold for the difference $\mathcal{D}_{T}(\omega)$ of cumulative regret and interim cumulative regret.
\subsection{Proof of Corollary~\ref{cor:D_finite}}\label{sec:pf_cor_D_finite}
By Corollary~\ref{cor:com_optimal}, for the weakly communicating $\mathcal{M}$, for any optimal policy $\pi^{*} \in \Pi_{\AC}$, the quantity $|R^{\pi^{*}}_{T} - TJ^{*}|$ satisfies the non-asymptotic concentration rates in \eqref{eq:policy_indp_1}--\eqref{eq:policy_indp_2}, where in the RHS, $D^{\pi}$ is replaced with $D$. On the other hand, by \eqref{eq:D(T)}, for any learning policy $\mu$, we have
\[
\mathcal{D}_{T}(\omega) = R^{\pi^{*}}_{T} - TJ^{*}.
\]
As a result, by substituting $\mathcal{D}_{T}(\omega)$ in the LHS of \eqref{eq:policy_indp_1}--\eqref{eq:policy_indp_2}, we get that for the weakly communicating $\mathcal{M}$, for any learning policy $\mu$, these non-asymptotic concentration rates also hold for the difference $\mathcal{D}_{T}(\omega)$ of cumulative regret and interim cumulative regret. At last by Prop.~\ref{prop:relation_of_mdps}, we have that if $\mathcal{M}$ is recurrent, unichain, or communicating it is also weakly communication. As a result, these non-asymptotic concentration bounds hold for all the cases.
\subsection{Proof of Theorem~\ref{thm:rate_equiv}}\label{sec:pf_thm_rate_equiv}
\subsubsection{Proof of Part 1}
This part of the theorem is a consequence of Theorem~\ref{thm:asymp_D}.
Recall that by definition, we have
\begin{equation}\label{eq:pf_rate_equiv_mid_1}
    \mathcal{D}_{T}(\omega) = \mathcal{R}_{T}^{\mu}(\omega) - \bar{\mathcal{R}}_{T}^{\mu}(\omega).  
\end{equation}
On the other hand, we can rewrite the law of iterated logarithm in Theorem~\ref{thm:asymp_D} using the $\tildeO(\cdot)$ notation as follows 
\begin{equation}\label{eq:pf_rate_equiv_mid_2}
    \mathcal{D}_{T}(\omega)\leq \tildeO(\sqrt{T}), \quad a.s.
\end{equation}
As a result, for any learning policy $\mu$ that satisfies 
\(
R^{\mu}_{T}(\omega) \leq \tildeO(\sqrt{T}),
\)
almost surely, 
\eqref{eq:pf_rate_equiv_mid_1}--\eqref{eq:pf_rate_equiv_mid_2} imply that $\bar{R}^{\mu}_{T}(\omega) \leq \tildeO(\sqrt{T})$. Similarly, for any learning policy $\mu$ that satisfies $\bar{R}^{\mu}_{T}(\omega) \leq \tildeO(\sqrt{T})$, almost surely, \eqref{eq:pf_rate_equiv_mid_1}--\eqref{eq:pf_rate_equiv_mid_2} imply that $R^{\mu}_{T}(\omega) \leq \tildeO(\sqrt{T})$. Therefore, statements 1 and 2 are equivalent. 
\subsubsection{Proof of Part 2}
Proof of this part is a consequence of Theorem~\ref{thm:D_finite}. By the theorem's hypothesis, for any $\delta_{1} \in (0,1)$, there exists a pair of functions $(T_{1}(\delta_{1}),h_{1}(\delta_{1},T))$, such that for all $T\geq T_{1}(\delta_{1})$, with probability at least $1-\delta_{1}$, we have
\begin{equation}\label{eq:pf_rate_equiv_mid_3}
    R^{\mu}_{T}(\omega) \leq h_{1}(\delta_{1},T),
\end{equation}
where for a fixed $\delta_{1}$, we have
\(
h_{1}(\delta_{1},T) = \tildeO(\sqrt{T}).
\)
Moreover, by Theorem~\ref{thm:D_finite}, we have that for any $\delta_{2} \in (0,1)$, there exists a pair of functions $(T_{2}(\delta_{2}),h_{2}(\delta_{2},T))$, such that for all $T\geq T_{2}(\delta_{2})$, with probability at least $1-\delta_{2}$, we have 
\begin{equation}\label{eq:pf_rate_equiv_mid_4}
\mathcal{D}_{T}(\omega) \leq h_{2}(\delta_{2},T),    
\end{equation}
where for a fixed $\delta_{2}$, we have
\(
h_{2}(\delta_{2},T) = \tildeO(\sqrt{T}).
\)
As a result, by \eqref{eq:pf_rate_equiv_mid_1}, \eqref{eq:pf_rate_equiv_mid_3}--\eqref{eq:pf_rate_equiv_mid_4}, and Lemma~\ref{lem:delta}, for any $\delta \in (0,1)$, for all $T\geq \max\big\{T_{1}(\delta/2),T_{2}(\delta/2)\big\}$, with probability at least $1-\delta$, we have
\[
\bar{R}^{\mu}_{T}(\omega) \leq h_{1}(\delta/2) + h_{2}(\delta/2).
\]
At last since for a fixed $\delta$, both $h_{1}(\delta/2)$ and $h_{2}(\delta/2)$ satisfy
\[
h_{1}(\delta/2) \leq \tildeO(\sqrt{T}), \quad \text{and}, \quad h_{2}(\delta/2) \leq \tildeO(\sqrt{T}),
\]
we get that $\bar{R}^{\mu}_{T}(\omega) \leq \tildeO(\sqrt{T})$. With repeating the similar arguments we can prove the $2^{nd}$ statement.

\section{Proof of Main Results for Discounted Reward Setup}\label{sec:pf_discount_main}
\subsection{Proof of Theorem~\ref{thm:discounted_non_asymp}}\label{app:pf_discount}
\subsubsection{Preliminary Results}
We first present a few preliminary lemmas. 
To simplify the notation, we define following martingale difference sequence. 
\begin{definition}
    Let filtration $\mathcal{F} = \{\mathcal{F}_{t}\}_{t\geq0}$ be defined as $\mathcal{F}_{t} \coloneqq \sigma(S_{0:t},A_{0:t})$. For any policy $\pi \in \Pi_{\SD}$, let $V^{\pi}_{\gamma}$ denote the corresponding discounted value function. We define the sequence $\{N_{t}^{\pi,\gamma}\}_{t\geq1}$ as follows
    \begin{equation}\label{eq:def_N}
         N_{t}^{\pi,\gamma} \coloneqq \Big[V^{\pi}_{\gamma}(S_{t})-\EXP\big[V^{\pi}_{\gamma}(S_{t})\bigm|S_{t-1},\pi(S_{t-1})\big]\Big], \quad \forall t \geq1,
    \end{equation}
    where $\{S_{t}\}_{t\geq1}$ denotes the random sequence of states encountered along the current sample path. 
\end{definition}
\begin{lemma}
    Sequence $\{\gamma^{t}N_{t}^{\pi,\gamma}\}_{t\geq1}$ is an MDS. 
\end{lemma}
\begin{proof}
By the definition of $\{\mathcal{F}_{t}\}_{t\geq0}$, we have that $S_{t-1}$ is $\mathcal{F}_{t-1}$-measurable. As a result, we have
    \begin{align*}
       \EXP\Big[\gamma^{t}N_{t}^{\pi,\gamma}\bigm|\mathcal{F}_{t-1}\Big] & = \EXP\Big[\gamma^{t}\big(V^{\pi}_{\gamma}(S_{t})-\EXP\big[V^{\pi}_{\gamma}(S_{t})\bigm|S_{t-1}, \pi(S_{t-1})\big]\big) \bigm| \mathcal{F}_{t-1}\Big]  \\ & = 
       \gamma^{t}\EXP\Big[V^{\pi}_{\gamma}(S_{t}) \bigm| \mathcal{F}_{t-1}\Big]-\gamma^{t}\EXP\Big[V^{\pi}_{\gamma}(S_{t})\bigm|S_{t-1}, \pi(S_{t-1})\Big] =0, 
    \end{align*}
    which shows that $\{\gamma_{t}N_{t}^{\pi,\gamma}\}_{t\geq0}$ is an MDS with respect to the filtration $\{\mathcal{F}_{t}\}_{t\geq0}$.    
\end{proof}
We now present a martingale decomposition for the discounted cumulative reward $R^{\pi,\gamma}_{T}(\omega)$ for any policy $\pi \in \Pi_{\SD}$.
\begin{lemma}\label{lem:discounted_decomposition}
Given any policy $\pi \in \Pi_{\SD}$, we can rewrite the discounted cumulative reward $R^{\pi,\gamma}_{T}$ as follows
\begin{equation}
    R^{\pi,\gamma}_{T}(\omega) =  \sum_{t=1}^{T}\gamma^{t}N_{t}^{\pi,\gamma} + V^{\pi}_{\gamma}(S_{0})-\gamma^{T}V^{\pi}_{\gamma}(S_{T}).
\end{equation}
\end{lemma}
\begin{proof}
    Since $\pi \in \Pi_{\SD}$, \eqref{eq:bellman} implies that along the trajectory of states $\{S_{t}\}_{t=0}^{T}$ induced by the policy $\pi$, we have
    \[
    r\big(S_{t},\pi(S_{t})\big) =  V^{\pi}_{\gamma}(S_{t}) - \gamma \EXP\Big[V^{\pi}_{\gamma}(S_{t+1})\bigm|S_{t},\pi(S_{t})\Big].
    \]
    Repeating similar steps as in the proof of Lemma~\ref{lem:return_decom}, we have
    \begin{align*}
        R^{\pi,\gamma}_{T}\!(\omega) = &\sum_{t=0}^{T-1}\gamma^{t}r(S_{t},\pi(S_{t}))\\ =
        &\sum_{t=0}^{T-1}\gamma^{t}\Big[V^{\pi}_{\gamma}\!(S_{t}) - \gamma\EXP\big[V^{\pi}_{\gamma}\!(S_{t+1})|S_{t},\pi(S_{t})\big] \Big]\\ \stackrel{(a)}{=}
        &\sum_{t=0}^{T-1}\gamma^{t}\Big[V^{\pi}_{\gamma}\!(S_{t}) - \gamma\EXP\big[V^{\pi}_{\gamma}\!(S_{t+1})|S_{t},\pi(S_{t})\big] \Big] + \gamma^{T}V^{\pi}_{\gamma}\!(S_{T})-\gamma^{T}V^{\pi}_{\gamma}\!(S_{T}) \\ \stackrel{(b)}{=} 
        &\sum_{t=0}^{T-1}\gamma^{t+1}\Big[V^{\pi}_{\gamma}\!(S_{t+1}) - \EXP\big[V^{\pi}_{\gamma}\!(S_{t+1})|S_{t},\pi(S_{t}) \big]\Big] + V^{\pi}_{\gamma}\!(S_{0})-\gamma^{T}V^{\pi}_{\gamma}\!(S_{T}) \\ \stackrel{(c)}{=} 
        &\sum_{t=0}^{T-1}\gamma^{t+1}N_{t+1}^{\pi,\gamma} + V^{\pi}_{\gamma}\!(S_{0})-\gamma^{T}V^{\pi}_{\gamma}\!(S_{T})\\ =
        &\sum_{t=1}^{T}\gamma^{t}N_{t}^{\pi,\gamma} + V^{\pi}_{\gamma}\!(S_{0})-\gamma^{T}V^{\pi}_{\gamma}\!(S_{T}),
    \end{align*}
    where $(a)$ follows from adding and subtracting the term $\gamma^{T}V^{\pi}_{\gamma}\!(S_{T})$, $(b)$ follows from re-arranging the terms in the summation, and $(c)$ follows from the definition of $\{N_{t}^{\pi,\gamma}\}_{t\geq1}$. 
\end{proof}
\subsubsection{Proof of Theorem~\ref{thm:discounted_non_asymp}}
Proof of this theorem follows from the martingale decomposition stated in Lemma~\ref{lem:discounted_decomposition} and the concentration bounds stated in Corollary~\ref{cor:azuma} and Theorem~\ref{thm:non_asym_LIL}.\\ 
\textbf{Proof of Part 1}
    By Lemma~\ref{lem:discounted_decomposition}, we have
    \begin{equation*}
    R^{\pi,\gamma}_{T}\!(\omega) =  \sum_{t=1}^{T}\gamma^{t}N_{t}^{\pi,\gamma} + V^{\pi}_{\gamma}(S_{0})-\gamma^{T}V^{\pi}_{\gamma}(S_{T}).
    \end{equation*}
    As a result, we have 
    \begin{equation}\label{eq:pf_dis_local_1}
    \Big|R^{\pi,\gamma}_{T}\!(\omega)-\big(V^{\pi}_{\gamma}(S_{0})-\gamma^{T}V^{\pi}_{\gamma}(S_{T})\big)\Big| = \big| \sum_{t=1}^{T}\gamma^{t}N_{t}^{\pi,\gamma}\big|.  
    \end{equation}
    In order to upper-bound the term $\big|\sum_{t=1}^{T}\gamma^{t}N_{t}^{\pi,\gamma}\big|$, we verify the conditions of Corollary~\ref{cor:azuma}. By \eqref{eq:k_discount} and Lemma~\ref{lem:inequalities}, we have
    \[
    \big|\gamma^{t}N_{t}^{\pi,\gamma}\big| \leq \gamma^{t}K^{\pi,\gamma} < \infty, \quad \forall t\geq1.
    \]
    As a result, MDS $\{\gamma^{t}N_{t}^{\pi,\gamma}\}_{t\geq1}$ is a sequentially bounded MDS with respect to the sequence $\{\gamma^{t}K^{\pi,\gamma}\}_{t\geq1}$. Therefore, Corollary~\ref{cor:azuma} implies that for any $\delta \in (0,1)$, with probability at least $1-\delta$, we have
    \begin{align}\label{eq:pf_dis_local_2}
        \Big|\sum_{t=1}^{T}\gamma^{t}N_{t}^{\pi,\gamma}\Big| &\leq \sqrt{2\sum_{t=1}^{T}(K^{\pi,\gamma})^{2}\gamma^{2t} \log\frac{2}{\delta}}\notag \\&= K^{\pi,\gamma} \sqrt{2\sum_{t=1}^{T} \gamma^{2t} \log\frac{2}{\delta}}\notag\\&= K^{\pi,\gamma} \sqrt{2f^{\gamma}(T)\log\frac{2}{\delta}}.
    \end{align}
    As a result, by combining \eqref{eq:pf_dis_local_1} and \eqref{eq:pf_dis_local_2}, with probability at least $1-\delta$, we have
    \begin{align}\label{eq:pf_dis_local_3}
    \Big|R^{\pi,\gamma}_{T}\!(\omega) - \big(V^{\pi}_{\gamma}(S_{0})-\gamma^{T}V^{\pi}_{\gamma}(S_{T})\big)\Big| \leq K^{\pi,\gamma} \sqrt{2f^{\gamma}(T)\log\frac{2}{\delta}}.    
    \end{align}
    
\textbf{Proof of Part 2:}
Similar to the proof of Part 1, by Lemma~\ref{lem:discounted_decomposition}, we have
    \begin{align}\label{eq:pf_dis_local_5}
        \Big|R^{\pi,\gamma}_{T}\!(\omega)-\big(V^{\pi}_{\gamma}(S_{0})-\gamma^{T}V^{\pi}_{\gamma}(S_{T})\big)\Big| = \big| \sum_{t=1}^{T}\gamma^{t}N_{t}^{\pi,\gamma}\big|.
    \end{align}
Moreover, MDS $\{\gamma^{t}N^{\pi,\gamma}_{t}\}_{t\geq 1}$ is a sequentially bounded MDS with respect to the sequence $\{\gamma^{t}K^{\pi,\gamma}\}_{t\geq 1}$.
Therefore, Theorem~\ref{thm:non_asym_LIL} implies that for any $\delta \in (0,1)$, for all $T\geq T_{0}(\delta)\coloneqq \min\big\{T \geq 1 : f^{\gamma}(T) > \dfrac{173}{K^{\pi,\gamma}}\log\dfrac{4}{\delta}\big\}$, with probability at least $1-\delta$, we have 
    \begin{align*}
        \Big| \sum_{t=1}^{T}\gamma^{t}N_{t}^{\pi,\gamma} \Big| &\leq \sqrt{3\Big(\sum_{t=1}^{T}(K^{\pi,\gamma})^{2}(\gamma^{t})^2\Big) \Big(2\log\log\big(\dfrac{3\sum_{t=1}^{T}(K^{\pi,\gamma})^2(\gamma^{t})^2}{2\big|\sum_{t=1}^{T}\gamma^{t}N_{t}^{\pi,\gamma}\big|}\big) +\log\frac{2}{\delta}\Big)}.
    \end{align*}
    Now there are two cases: either $|\sum_{t=1}^{T}\gamma^{t}N_{t}^{\pi,\gamma}| \leq (K^{\pi,\gamma})^2$ or $|\sum_{t=1}^{T}\gamma^{t}N_{t}^{\pi,\gamma}| \geq (K^{\pi,\gamma})^2$. If $|\sum_{t=1}^{T}\gamma^{t}N_{t}^{\pi,\gamma}| \geq (K^{\pi,\gamma})^2$, we get:
    \begin{align*}
        \Big|\sum_{t=1}^{T}\gamma^{t}N_{t}^{\pi,\gamma}\Big| &\leq \sqrt{3\Big(\sum_{t=1}^{T}(K^{\pi,\gamma})^{2}(\gamma^{t})^2\Big) \Big(2\log\log\big(\dfrac{3\sum_{t=1}^{T}(K^{\pi,\gamma})^2(\gamma^{t})^2}{2\big|\sum_{t=1}^{T}\gamma^{t}N_{t}^{\pi,\gamma}\big|}\big) +\log\frac{2}{\delta}\Big)} \\&\leq
        \sqrt{3\Big(\sum_{t=1}^{T}(K^{\pi,\gamma})^{2}(\gamma^{t})^2\Big) \Big(2\log\log\big(\dfrac{3\sum_{t=1}^{T}(K^{\pi,\gamma})^2(\gamma^{t})^2}{2(K^{\pi,\gamma})^2}\big) +\log\frac{2}{\delta}\Big)}\\ &\stackrel{(a)}{=}
        K^{\pi,\gamma}\sqrt{3f^{\gamma}(T) \Big(2\log\log\big(\dfrac{3}{2}f^{\gamma}(T)\big) +\log\frac{2}{\delta}\Big)},
    \end{align*}
    where $(a)$ follows from the geometric series formula and the definition of $f^{\gamma}(T)$.
    Otherwise, we have $|\sum_{t=1}^{T}\gamma^{t}N_{t}^{\pi,\gamma}| \leq (K^{\pi,\gamma})^2$. As a result, we can summarize these two cases as follows
    \begin{align}\label{eq:pf_dis_local_6}
        \Big|\sum_{t=1}^{T}\gamma^{t}N_{t}^{\pi,\gamma}\Big| \leq  
        \max\bigg\{ K^{\pi,\gamma} \sqrt{3f^{\gamma}(T)\big(2\log\log \dfrac{3}{2}f^{\gamma}(T) +\log\frac{2}{\delta}\big)},(K^{\pi,\gamma})^{2}\bigg\}.
    \end{align}
By combining \eqref{eq:pf_dis_local_5}--\eqref{eq:pf_dis_local_6}, with probability at least $1-\delta$, we have 
\begin{align}\label{eq:pf_dis_local_7}
    \Big|R^{\pi,\gamma}_{T}\!(\omega)&-\big(V^{\pi}_{\gamma}(S_{0})-\gamma^{T}V^{\pi}_{\gamma}(S_{T})\big)\Big| \notag  \\ &\leq \max\bigg\{ K^{\pi,\gamma} \sqrt{3f^{\gamma}(T)\big(2\log\log (\frac{3}{2}f^{\gamma}(T)) +\log\frac{2}{\delta}\big)},(K^{\pi,\gamma})^{2}\bigg\}.
\end{align}

\subsection{Proof of Corollary~\ref{cor:discounted_non_asymp2}}\label{app:pf_discount_2}
\textbf{Proof of Part 1:}
    By Lemma~\ref{lem:discounted_decomposition}, we have
    \begin{equation*}
    R^{\pi,\gamma}_{T}\!(\omega) =  \sum_{t=1}^{T}\gamma^{t}N_{t}^{\pi,\gamma} + V^{\pi}_{\gamma}(S_{0})-\gamma^{T}V^{\pi}_{\gamma}(S_{T}).
    \end{equation*}
    As a result, we have 
    \begin{equation}\label{eq:pf_dis_local_8}
    \Big|R^{\pi,\gamma}_{T}\!(\omega)-V^{\pi}_{\gamma}(S_{0})\Big| \stackrel{(a)}{\leq} \Big| \sum_{t=1}^{T}\gamma^{t}N_{t}^{\pi,\gamma}\Big| + \Big|\gamma^{T}V^{\pi}_{\gamma}(S_{T})\Big|,  
    \end{equation}
    where $(a)$ follows from the triangle inequality. 
    In the proof of Theorem~\ref{thm:discounted_non_asymp}, Part 1, we showed that with probability at least $1-\delta$, we have
    \begin{align}\label{eq:pf_dis_local_9}
        \Big|\sum_{t=1}^{T}\gamma^{t}N_{t}^{\pi,\gamma}\Big| \leq  K^{\pi,\gamma} \sqrt{2f^{\gamma}(T)\log\frac{2}{\delta}}.
    \end{align}
    Moreover, we have
    \begin{align}\label{eq:pf_dis_local_10}
    \gamma^{T}V^{\pi}_{\gamma}(S_{T})  & =  \gamma^{T}\EXP^{\pi}\Big[\lim_{T\to\infty} \sum_{t=0}^{T-1}\gamma^{t}r(S_{t},A_{t}) \bigm| S_{0}=S_{T}\Big] \notag \\& =  \gamma^{T}\EXP^{\pi}\Big[\lim_{T\to\infty} \sum_{t=0}^{T-1}\gamma^{t}R_{\max} \bigm| S_{0}=S_{T}\Big]\leq \frac{\gamma^{T}}{1-\gamma}R_{\max}.    
    \end{align}
    By combining \eqref{eq:pf_dis_local_8}--\eqref{eq:pf_dis_local_10}, with probability $1-\delta$, we have
    \[
    \Big|R^{\pi,\gamma}_{T}\!(\omega) - V^{\pi}_{\gamma}(S_{0})\Big| \leq K^{\pi,\gamma} \sqrt{2f^{\gamma}(T)\log\frac{2}{\delta}} + \frac{\gamma^{T}}{1-\gamma}R_{\max}.
    \]
    \textbf{Proof of Part 2:} Similar to the proof of Part 1, by Lemma~\ref{lem:discounted_decomposition}, we have
    \begin{equation}\label{eq:pf_dis_local_11}
    \Big|R^{\pi,\gamma}_{T}\!(\omega)-V^{\pi}_{\gamma}(S_{0})\Big| \leq \Big| \sum_{t=1}^{T}\gamma^{t}N_{t}^{\pi,\gamma}\Big| + \Big|\gamma^{T}V^{\pi}_{\gamma}(S_{T})\Big|.  
    \end{equation}
    Moreover, we have
    \begin{align}\label{eq:pf_dis_local_12}
    \Big|\gamma^{T}V^{\pi}_{\gamma}(S_{T}) \Big| \leq \gamma^{T}\frac{R_{\max}}{1-\gamma}.    
    \end{align}
    In addition, from proof of Theorem~\ref{thm:discounted_non_asymp}, Part 2, we have for any $\delta \in (0,1)$, for all $T\geq T_{0}(\delta)\coloneqq \min\big\{T' \geq 1 : f^{\gamma}(T') > \dfrac{173}{K^{\pi,\gamma}}\log\dfrac{4}{\delta}\big\}$, with probability at least $1-\delta$, we have
    \begin{align}\label{eq:pf_dis_local_13}
        \Big|\sum_{t=1}^{T}\gamma^{t}N_{t}^{\pi,\gamma}\Big| \leq  
        \max\bigg\{ K^{\pi,\gamma} \sqrt{3f^{\gamma}(T)\big(2\log\log \dfrac{3}{2}f^{\gamma}(T) +\log\frac{2}{\delta}\big)},(K^{\pi,\gamma})^{2}\bigg\}.
    \end{align}
    By combining \eqref{eq:pf_dis_local_11}--\eqref{eq:pf_dis_local_13}, with probability at least $1-\delta$, we have
    \begin{align}
    \Big|R^{\pi,\gamma}_{T}\!(\omega)&-V^{\pi}_{\gamma}(S_{0})\Big| \notag  \\ &\leq \max\bigg\{ K^{\pi,\gamma} \sqrt{3f^{\gamma}(T)\big(2\log\log (\frac{3}{2}f^{\gamma}(T)) +\log\frac{2}{\delta}\big)},(K^{\pi,\gamma})^{2}\bigg\} + \frac{\gamma^{T}}{1-\gamma}R_{\max}.    
    \end{align}
\subsection{Proof of Corollary~\ref{cor:two_policies_discounted}}\label{sec:pf_cor_two_policies_discounted}
\subsubsection{Proof of Part 1}
Consider two policies $\pi_{1},\pi_{2} \in \Pi_{\SD}$. Let $\{S_{t}^{\pi_{1}}\}_{t\geq0}$ and $\{S_{t}^{\pi_{2}}\}_{t\geq0}$ denote the random sequences of states encountered by following policies $\pi_{1}$ and $\pi_{2}$. We have
\begin{align}\label{eq:mid_pf_dis_two_policy_1}
    \Big|R^{\pi_{1},\gamma}_{T}-R^{\pi_{2},\gamma}_{T} \Big| \stackrel{(a)}{=}&\Big|R^{\pi_{1},\gamma}_{T} -\big[V^{\pi_{1}}_{\gamma}(S_{0}^{\pi_{1}}) - \gamma^{T}V^{\pi_{1}}_{\gamma}(S_{T}^{\pi_1}) \big] + \big[V^{\pi_{1}}_{\gamma}(S_{0}^{\pi_1}) - \gamma^{T}V^{\pi_{1}}_{\gamma}(S_{T}^{\pi_1}) \big]\notag\\ - &\big[V^{\pi_{2}}_{\gamma}(S_{0}^{\pi_2}) - \gamma^{T}V^{\pi_{2}}_{\gamma}(S_{T}^{\pi_2}) \big] + \big[V^{\pi_{2}}_{\gamma}(S_{0}^{\pi_2}) - \gamma^{T}V^{\pi_{2}}_{\gamma}(S_{T}^{\pi_2}) \big] - R^{\pi_{2},\gamma}_{T} \Big| \notag \notag\\\stackrel{(b)}{\leq} &\Big|R^{\pi_{1},\gamma}_{T} - \big[V^{\pi_{1}}_{\gamma}(S_{0}^{\pi_1}) - \gamma^{T}V^{\pi_{1}}_{\gamma}(S_{T}^{\pi_1}) \big]\Big| + \Big|\big[V^{\pi_{2}}_{\gamma}(S_{0}^{\pi_2}) - \gamma^{T}V^{\pi_{2}}_{\gamma}(S_{T}^{\pi_2}) \big] - R^{\pi_{2},\gamma}_{T} \Big|\notag \\
    + &\Big|\big[V^{\pi_{1}}_{\gamma}(S_{0}^{\pi_1}) - \gamma^{T}V^{\pi_{1}}_{\gamma}(S_{T}^{\pi_1}) \big] - \big[V^{\pi_{2}}_{\gamma}(S_{0}^{\pi_2}) - \gamma^{T}V^{\pi_{2}}_{\gamma}(S_{T}^{\pi_2}) \big] \Big|,
\end{align}
where $(a)$ follows by adding and subtracting $\big[V^{\pi_{1}}_{\gamma}(S_{0}^{\pi_1}) - \gamma^{T}V^{\pi_{1}}_{\gamma}(S_{T}^{\pi_1}) \big]$ and $\big[V^{\pi_{2}}_{\gamma}(S_{0}^{\pi_2}) - \gamma^{T}V^{\pi_{2}}_{\gamma}(S_{T}^{\pi_2}) \big]$ and $(b)$ follows from the triangle inequality. Similarly, we have 
\begin{align}\label{eq:mid_pf_dis_two_policy_2}
    &\Big|\big[V^{\pi_{1}}_{\gamma}(S_{0}^{\pi_1}) - \gamma^{T}V^{\pi_{1}}_{\gamma}(S_{T}^{\pi_1}) \big]-\big[V^{\pi_{2}}_{\gamma}(S_{0}^{\pi_2}) - \gamma^{T}V^{\pi_{2}}_{\gamma}(S_{T}^{\pi_2}) \big] \Big| \stackrel{(a)}{=}\notag\\&\Big|\big[V^{\pi_{1}}_{\gamma}(S_{0}^{\pi_1}) - \gamma^{T}V^{\pi_{1}}_{\gamma}(S_{T}^{\pi_1}) \big] -R^{\pi_{1},\gamma}_{T}  + R^{\pi_{1},\gamma}_{T} - R^{\pi_{2},\gamma}_{T} + R^{\pi_{2},\gamma}_{T}   - \big[V^{\pi_{2}}_{\gamma}(S_{0}^{\pi_2}) - \gamma^{T}V^{\pi_{2}}_{\gamma}(S_{T}^{\pi_2}) \big]\Big| \notag\\\stackrel{(b)}{\leq} &\Big|R^{\pi_{1},\gamma}_{T} - \big[V^{\pi_{1}}_{\gamma}(S_{0}^{\pi_1}) - \gamma^{T}V^{\pi_{1}}_{\gamma}(S_{T}^{\pi_1}) \big]\Big| + \Big|R^{\pi_{2},\gamma}_{T} - \big[V^{\pi_{2}}_{\gamma}(S_{0}^{\pi_2}) - \gamma^{T}V^{\pi_{2}}_{\gamma}(S_{T}^{\pi_2}) \big]\Big| \notag \\
    +&\Big| R^{\pi_{1},\gamma}_{T} - R^{\pi_{2},\gamma}_{T} \Big|,
\end{align}
where $(a)$ follows by adding and subtracting $R^{\pi_{1},\gamma}_{T}$ and $R^{\pi_{2},\gamma}_{T}$ and $(b)$ follows from the triangle inequality. \eqref{eq:mid_pf_dis_two_policy_1}--\eqref{eq:mid_pf_dis_two_policy_2} imply that 
\begin{align}\label{eq:mid_pf_dis_two_policy_3}
    &\Big|\big|R^{\pi_{1},\gamma}_{T}-R^{\pi_{2},\gamma}_{T}\big| - \big|\big[V^{\pi_{1}}_{\gamma}(S_{0}^{\pi_1}) - \gamma^{T}V^{\pi_{1}}_{\gamma}(S_{T}^{\pi_1}) \big]-\big[V^{\pi_{2}}_{\gamma}(S_{0}^{\pi_2}) - \gamma^{T}V^{\pi_{2}}_{\gamma}(S_{T}^{\pi_2}) \big] \big| \Big| \notag  \\ \leq &\Big|R^{\pi_{1},\gamma}_{T} - \big[V^{\pi_{1}}_{\gamma}(S_{0}^{\pi_1}) - \gamma^{T}V^{\pi_{1}}_{\gamma}(S_{T}^{\pi_1}) \big]\Big| + \Big|R^{\pi_{2},\gamma}_{T} - \big[V^{\pi_{2}}_{\gamma}(S_{0}^{\pi_2}) - \gamma^{T}V^{\pi_{2}}_{\gamma}(S_{T}^{\pi_2}) \big]\Big|.
\end{align}
By Theorem~\ref{thm:discounted_non_asymp}, we know that for any $\delta_{1} \in (0,1)$, with probability at least $1-\delta_{1}$, we have
\begin{equation}\label{eq:mid_pf_dis_two_policy_4}
\Big|R^{\pi_{1},\gamma}_{T} - \big[V^{\pi_{1}}_{\gamma}(S_{0}^{\pi_1}) - \gamma^{T}V^{\pi_{1}}_{\gamma}(S_{T}^{\pi_1}) \big]\Big| \leq K^{\pi_{1},\gamma} \sqrt{2f^{\gamma}(T)\log\frac{2}{\delta_{1}}}.  
\end{equation}
Similarly, we have that for any $\delta_{2} \in (0,1)$, with probability at least $1-\delta_{2}$, we have
\begin{equation}\label{eq:mid_pf_dis_two_policy_5}
\Big|R^{\pi_{2},\gamma}_{T} - \big[V^{\pi_{2}}_{\gamma}(S_{0}^{\pi_2}) - \gamma^{T}V^{\pi_{2}}_{\gamma}(S_{T}^{\pi_2}) \big]\Big| \leq K^{\pi_{2},\gamma} \sqrt{2f^{\gamma}(T)\log\frac{2}{\delta_{2}}}.    
\end{equation}
As a result, by applying Lemma~\ref{lem:delta} and \eqref{eq:mid_pf_dis_two_policy_3}--\eqref{eq:mid_pf_dis_two_policy_5}, for any $\delta \in (0,1)$, with probability at least $1-\delta$, we have
\begin{align*}
    &\phantom{\leq}\Big|\big|R^{\pi_{1},\gamma}_{T}-R^{\pi_{2},\gamma}_{T}\big| - \big|\big[V^{\pi_{1}}_{\gamma}(S_{0}^{\pi_1}) - \gamma^{T}V^{\pi_{1}}_{\gamma}(S_{T}^{\pi_1}) \big]-\big[V^{\pi_{2}}_{\gamma}(S_{0}^{\pi_2}) - \gamma^{T}V^{\pi_{2}}_{\gamma}(S_{T}^{\pi_2}) \big] \big| \Big| \\&\leq \Big|R^{\pi_{1},\gamma}_{T} - \big[V^{\pi_{1}}_{\gamma}(S_{0}^{\pi_1}) - \gamma^{T}V^{\pi_{1}}_{\gamma}(S_{T}^{\pi_1}) \big]\Big| + \Big|R^{\pi_{2},\gamma}_{T} - \big[V^{\pi_{2}}_{\gamma}(S_{0}^{\pi_2}) - \gamma^{T}V^{\pi_{2}}_{\gamma}(S_{T}^{\pi_2}) \big]\Big| \\&\leq K^{\pi_{1},\gamma}\sqrt{2f^{\gamma}(T)\log\frac{4}{\delta}} + K^{\pi_{2},\gamma}\sqrt{2f^{\gamma}(T)\log\frac{4}{\delta}}.
\end{align*}

\subsubsection{Proof of Part 2}
As we showed in the proof of part 1, for any two policies $\pi_{1},\pi_{2} \in \Pi_{\SD}$, we have
\begin{align}\label{eq:mid_pf_dis_two_policy_6}
    &\Big|\big|R^{\pi_{1},\gamma}_{T}-R^{\pi_{2},\gamma}_{T}\big| - \big|\big[V^{\pi_{1}}_{\gamma}(S_{0}^{\pi_1}) - \gamma^{T}V^{\pi_{1}}_{\gamma}(S_{T}^{\pi_1}) \big]-\big[V^{\pi_{2}}_{\gamma}(S_{0}^{\pi_2}) - \gamma^{T}V^{\pi_{2}}_{\gamma}(S_{T}^{\pi_2}) \big] \big| \Big| \notag  \\ \leq &\Big|R^{\pi_{1},\gamma}_{T} - \big[V^{\pi_{1}}_{\gamma}(S_{0}^{\pi_1}) - \gamma^{T}V^{\pi_{1}}_{\gamma}(S_{T}^{\pi_1}) \big]\Big| + \Big|R^{\pi_{2},\gamma}_{T} - \big[V^{\pi_{2}}_{\gamma}(S_{0}^{\pi_2}) - \gamma^{T}V^{\pi_{2}}_{\gamma}(S_{T}^{\pi_2}) \big]\Big|.
\end{align}
By Theorem~\ref{thm:discounted_non_asymp}, we have that for any $\delta_{1} \in (0,1)$, for all $T\geq T_{0}^{\pi_{1}}(\delta_{1})\coloneqq \min\big\{T'\geq1 : f^{\gamma}(T') > \dfrac{173}{K^{\pi_{1},\gamma}}\log\dfrac{4}{\delta_{1}}\big\}$, with probability at least $1-\delta_{1}$, we have:
\begin{align*}
\Big|R^{\pi_{1},\gamma}_{T} - \big[V^{\pi_{1}}_{\gamma}(S_{0}^{\pi_1}) &- \gamma^{T}V^{\pi_{1}}_{\gamma}(S_{T}^{\pi_1}) \big]\Big| \\&\leq \max\Big\{K^{\pi_{1},\gamma}\sqrt{3f^{\gamma}(T)\Big(2\log\log\frac{3}{2}f^{\gamma}(T) + \log\frac{2}{\delta_{1}}\Big)}, (K^{\pi_{1},\gamma})^2\Big\}.
\end{align*}
Similarly, we have that for any $\delta_{2} \in (0,1)$, for all $T\geq T_{0}^{\pi_{2}}(\delta_{2})\coloneqq \min\big\{T'\geq1 : f^{\gamma}(T') > \dfrac{173}{K^{\pi_{2},\gamma}}\log\dfrac{4}{\delta_{2}}\big\}$, with probability at least $1-\delta_{2}$, we have:
\begin{align*}
\Big|R^{\pi_{2},\gamma}_{T} - \big[V^{\pi_{2}}_{\gamma}(S_{0}^{\pi_2}) &- \gamma^{T}V^{\pi_{2}}_{\gamma}(S_{T}^{\pi_2}) \big]\Big| \\&\leq \max\Big\{K^{\pi_{2},\gamma}\sqrt{3f^{\gamma}(T)\Big(2\log\log\frac{3}{2}f^{\gamma}(T) + \log\frac{2}{\delta_{2}}\Big)}, (K^{\pi_{2},\gamma})^2\Big\}.
\end{align*}
As a result, by applying Lemma~\ref{lem:delta}, 
for all $T \geq T^\pi_{0}(\delta) \coloneqq  \max\big\{T_{0}^{\pi_{1}}(\frac{\delta}{2}),T_{0}^{\pi_{2}}(\frac{\delta}{2})\big\} $, with probability at least $1-\delta$, we have
\begin{align*}\label{eq:mid_pf_dis_two_policy_6}
    &\Big|\big|R^{\pi_{1},\gamma}_{T}-R^{\pi_{2},\gamma}_{T}\big| - \big|\big[V^{\pi_{1}}_{\gamma}(S_{0}^{\pi_1}) - \gamma^{T}V^{\pi_{1}}_{\gamma}(S_{T}^{\pi_1}) \big]-\big[V^{\pi_{2}}_{\gamma}(S_{0}^{\pi_2}) - \gamma^{T}V^{\pi_{2}}_{\gamma}(S_{T}^{\pi_2}) \big] \big| \Big| \notag  \\ \leq &\Big|R^{\pi_{1},\gamma}_{T} - \big[V^{\pi_{1}}_{\gamma}(S_{0}^{\pi_1}) - \gamma^{T}V^{\pi_{1}}_{\gamma}(S_{T}^{\pi_1}) \big]\Big| + \Big|R^{\pi_{2},\gamma}_{T} - \big[V^{\pi_{2}}_{\gamma}(S_{0}^{\pi_2}) - \gamma^{T}V^{\pi_{2}}_{\gamma}(S_{T}^{\pi_2}) \big]\Big|\\ \leq
    &\max\Big\{K^{\pi_{1},\gamma}\sqrt{3f^{\gamma}(T)\Big(2\log\log\frac{3}{2}f^{\gamma}(T) + \log\frac{4}{\delta}\Big)}, (K^{\pi_{1},\gamma})^2\Big\}\\ 
    + &\max\Big\{K^{\pi_{2},\gamma}\sqrt{3f^{\gamma}(T)\Big(2\log\log\frac{3}{2}f^{\gamma}(T) + \log\frac{4}{\delta}\Big)}, (K^{\pi_{2},\gamma})^2\Big\}.
\end{align*}
\subsection{Proof of Corollary~\ref{cor:beta_asymptotic}}\label{app:pf_cor_beta_asymptotic}
Since policy $\pi \in \Pi_{\AC}$, we know the pair $(J^{\pi},V^{\pi})$ exists  and $J^{\pi}$ is constant for all $s\in \mathcal{S}$. We first prove following preliminary lemma.
\subsubsection{Preliminary Lemma}
\begin{lemma}\label{lem:limits}
    For any policy $\pi \in \Pi_{\AC}$, as $\gamma$ goes to $1$ from below, following statements hold.
    \begin{enumerate}
        \item For any two states $s_{1}, s_{2} \in \mathcal{S}$, we have
        \begin{equation*}
        \lim_{\gamma \toup 1}\bigg[V_{\gamma}^{\pi}(s_{1}) - V_{\gamma}^{\pi}(s_{2})\bigg] =   V^{\pi}(s_{1})-V^{\pi}(s_{2}).   
        \end{equation*}
        \item For any two states $s_{1}, s_{2} \in \mathcal{S}$, we have
        \begin{equation*}
        \lim_{\gamma \toup 1}\bigg[V_{\gamma}^{\pi}(s_{1}) - \gamma^{T} V_{\gamma}^{\pi}(s_{2})\bigg] =   TJ^{\pi} + V^{\pi}(s_{1})-V^{\pi}(s_{2}).    
        \end{equation*}
        \item We have
        \begin{equation}\label{eq:limi_cor_mid_1}
            \lim_{\gamma \toup 1} f(T,\gamma) =  T.   
        \end{equation}
        \item We have
        \begin{equation}\label{eq:limi_cor_mid_2}
            \lim_{\gamma \toup 1} R^{\pi,\gamma}_{T} = R^{\pi}_{T}.   
        \end{equation}
    \end{enumerate}
\end{lemma}

\begin{proof}
    \textbf{of Part~1}: From the Laurent series expansion (\citep[Proposition~5.1.2]{bertsekas2012dynamic}, for any policy $\pi \in \Pi_{\SD}$, we have
\[
V_{\gamma}^{\pi}(s) = \frac{J^{\pi}}{1-\gamma} + V^{\pi}(s) + O(|1-\gamma|), \quad \forall s \in \mathcal{S}.
\]
As a result, we have
\begin{align*} 
    &\lim_{\gamma \toup 1}\bigg[V_{\gamma}^{\pi}(s_{1}) - V_{\gamma}^{\pi}(s_{2})\bigg] \\
    = &\lim_{\gamma \toup 1}\bigg[\frac{J^{\pi}}{1-\gamma} + V^{\pi}(s_{1}) + O(|1-\gamma|) - \Big[\frac{J^{\pi}}{1-\gamma} + V^{\pi}(s_{2}) + O(|1-\gamma|)\Big]\bigg]\\
    = & \lim_{\gamma \toup 1}\bigg[ V^{\pi}(s_{1})-V^{\pi}(s_{2})\bigg] = V^{\pi}(s_{1})-V^{\pi}(s_{2}).
\end{align*}
\textbf{Proof of Part~2}: Again from the Laurent series expansion (\citep[Proposition~5.1.2]{bertsekas2012dynamic}, for any policy $\pi \in \Pi_{\SD}$, we have
\[
V_{\gamma}^{\pi}(s) = \frac{J^{\pi}}{1-\gamma} + V^{\pi}(s) + O(|1-\gamma|), \quad \forall s \in \mathcal{S}.
\]
As a result, we have 
\begin{align*}
    &\lim_{\gamma \toup 1}\bigg[V_{\gamma}^{\pi}(s_{1}) - \gamma^{T}V_{\gamma}^{\pi}(s_{2})\bigg] \\
     = &\lim_{\gamma \toup 1}\bigg[\frac{J^{\pi}}{1-\gamma} + V^{\pi}(s_{1}) + O(|1-\gamma|) - \Big[\frac{\gamma^{T}J^{\pi}}{1-\gamma} + \gamma^{T}V^{\pi}(s_{2}) + O(\gamma^{T}|1-\gamma|)\Big]\bigg]\\
    = &\lim_{\gamma \toup 1} \bigg[\frac{(1-\gamma^{T})}{1-\gamma}J^{\pi} + V^{\pi}(s_{1})-\gamma^{T}V^{\pi}(s_{2}) \bigg]\\
    = &\phantom{.} TJ^{\pi} + V^{\pi}(s_{1})-V^{\pi}(s_{2}).
\end{align*}
\textbf{Proof of Part~3}: From the definition, we have
    \[
        \lim_{\gamma \toup 1} f(T,\gamma) = \lim_{\gamma \toup 1} \big[\frac{\gamma^{2}-\gamma^{2T+2}}{1-\gamma^2}\big] =  \lim_{\gamma \toup 1} \big[\sum_{t=1}^{T}\gamma^{2t}\big] = T .
    \]
\textbf{Proof of Part~4}: From the definition, for any finite $T\geq1$, we have
        \[
        \lim_{\gamma \toup 1}\big[R^{\pi,\gamma}_{T}\big] =  \lim_{\gamma \toup 1}\Big[\sum_{t=0}^{T-1}\gamma^{t}r(S_{t},A_{t})\Big] = \sum_{t=0}^{T-1}r(S_{t},A_{t}) = R^{\pi}_{T} .
        \]
\end{proof}
\subsubsection{Proof of Corollary~\ref{cor:beta_asymptotic}}
\textbf{Proof of Part~1}: By Lemma~\ref{lem:limits}, Part~4, for all $T\geq1$, we have 
\begin{equation}\label{eq:limi_cor_mid_3}
    \lim_{\gamma \toup 1}\big[R_{T}^{\pi,\gamma}\big]= R^{\pi}_{T}.
\end{equation}
Moreover, we have
\begin{align}\label{eq:limi_cor_mid_4}
    \lim_{\gamma \toup 1}\big[ V^{\pi}_{\gamma}(S_{0})-\gamma^{T}V^{\pi}_{\gamma}(S_{T}) \big] &= \lim_{\gamma \toup 1}\big[ V^{\pi}_{\gamma}(S_{0}) - V^{\pi}_{\gamma}(S_{T}) + V^{\pi}_{\gamma}(S_{T}) -\gamma^{T}V^{\pi}_{\gamma}(S_{T}) \big] \notag\\&\stackrel{(a)}{=}
    V^{\pi}(S_{0}) - V^{\pi}(S_{T}) + TJ^{\pi} + V^{\pi}(S_{T}) - V^{\pi}(S_{T})\notag\\
    &= TJ^{\pi} +  V^{\pi}(S_{0}) - V^{\pi}(S_{T}), 
\end{align}
where $(a)$ follows from Lemma~\ref{lem:limits}, Parts~1 and 2. 
The result of this part follows by substituting \eqref{eq:limi_cor_mid_3}--\eqref{eq:limi_cor_mid_4} on the LHS of \eqref{eq:discounted_finite1}.\\
\textbf{Proof of Part~2}: By Lemma~\ref{lem:limits}, Part~2, for all $s_{1},s_{2} \in \mathcal{S}$, we have
\begin{equation*}
    \lim_{\gamma \toup 1}\Big[V_{\gamma}^{\pi}(s_{1}) - V_{\gamma}^{\pi}(s_{2})\Big] =   V^{\pi}(s_{1})-V^{\pi}(s_{2}).   
\end{equation*}
This implies that     
\begin{equation}\label{eq:limi_cor_mid_5}
    \lim_{\gamma \toup 1}\big[K^{\pi,\gamma}\big] = K^{\pi}.   
\end{equation}
Moreover, by Lemma~\ref{lem:limits}, Part~3, we have
\begin{equation}\label{eq:limi_cor_mid_6}
    \lim_{\gamma \toup 1} f^{\gamma}(T) = T.
\end{equation}
The result of this part follows by substituting \eqref{eq:limi_cor_mid_5}--\eqref{eq:limi_cor_mid_6} on the RHS of \eqref{eq:discounted_finite1}.\\
\textbf{Proof of Part~3}: 
The result of this part follows by substituting \eqref{eq:limi_cor_mid_5}--\eqref{eq:limi_cor_mid_6} on the RHS of \eqref{eq:discounted_finite2}.
\section{Proof of Main Results for Finite-Horizon Setup}\label{sec:pf_finite_horizon}
\subsection{Proof of Theorem~\ref{thm:finite-h-non-asym}}\label{sec:pf_thm:finite-h-non-asym}
\subsubsection{Preliminary Results}
We first present a few preliminary lemmas. 
To simplify the notation, we define following martingale difference sequence. 
\begin{definition}
    Let filtration $\mathcal{F}=\{\mathcal{F}_{t}\}_{t=0}^{h}$ be defined as $\mathcal{F}_{t} \coloneqq \sigma(S_{0:t},A_{0:t})$. For any policy $\pi \in \Pi_{\FD}$, let $\{V^{\pi,h}_{t}\}_{t=0}^{h+1}$ denote the corresponding finite-horizon value function. We define the sequence $\{W_{t}^{\pi,h}\}_{t=0}^{h+1}$ as follows 
    \begin{equation}\label{eq:def_W}
        W_{t}^{\pi,h} \coloneqq \Big[V^{\pi,h}_{t}(S_{t}) - \EXP\big[V^{\pi,h}_{t}(S_{t})\bigm|S_{t-1},\pi_{t-1}(S_{t-1}) \big] \Big], \quad \forall t \in \{1,\ldots,h+1\},
    \end{equation}
    where $\{S_{t}\}_{t=0}^{h}$ denotes the random sequence of states encountered along the current sample path.
    \end{definition}
    \begin{lemma}
        Sequence $\{W_{t}^{\pi,h}\}_{t=0}^{h+1}$ is an MDS.
    \end{lemma}
    \begin{proof}
        By the definition of $\{\mathcal{F}_{t}\}_{t=0}^{h}$, we have that $S_{t-1}$ is $\mathcal{F}_{t-1}$-measurable. As a result, we have
        \begin{align*}
            \EXP\Big[W_{t}^{\pi,h} \bigm| \mathcal{F}_{t-1}\Big] &= \EXP\Big[ V^{\pi,h}_{t}(S_{t}) - \EXP\big[V^{\pi,h}_{t}(S_{t})\bigm|S_{t-1},\pi_{t-1}(S_{t-1}) \big] \bigm| \mathcal{F}_{t-1}\Big] \\
            &=\EXP\Big[V^{\pi,h}_{t}(S_{t})\bigm|\mathcal{F}_{t-1} \Big] - \EXP\Big[V^{\pi,h}_{t}(S_{t})\bigm|S_{t-1},\pi_{t-1}(S_{t-1})  \Big]=0,
        \end{align*}
        which shows that $\{W_{t}^{\pi,h}\}_{t=0}^{h+1}$ is an MDS with respect to the filtration $\{\mathcal{F}_{t}\}_{t=0}^{h}$. 
    \end{proof}
    We now present a martingale decomposition for the cumulative reward $R^{\pi,h}_{T}(\omega)$ for any policy $\pi \in \Pi_{\FD}$.\begin{lemma}\label{lem:finite_h_decomposition}
        Given any policy $\pi \in \Pi_{\FD}$, we can rewrite the cumulative reward $R^{\pi,h}_{T}$ as follows
        \begin{equation}
            R^{\pi,h}_{T}(\omega) = \sum_{t=1}^{T}W_{t}^{\pi,h} + V_{0}^{\pi,h}(S_{0}) - V_{T}^{\pi,h}(S_{T}). 
        \end{equation}
    \end{lemma}
    \begin{proof}
        \eqref{eq:finite-h-PE} implies that along the trajectory of states $\{S_{t}\}_{t=0}^{T}$ induced by the policy $\pi$, we have 
        \[
        r(S_{t},\pi(S_{t})) = V_{t}^{\pi,h}(S_{t}) - \EXP\Big[V_{t+1}^{\pi,h}(S_{t+1}) \bigm| S_{t},\pi(S_{t})\Big].
        \]
    \end{proof}
    For any $1\leq T\leq h+1$, by repeating similar steps as in the proof of Lemma~\ref{lem:return_decom}, we have 
    \begin{align*}
        R^{\pi,h}_{T} =\phantom{.} &\sum_{t=0}^{T-1}\Big[V^{\pi,h}_{t}(S_{t}) - \EXP\big[V^{\pi,h}_{t+1}(S_{t+1})\bigm|S_{t},\pi_{t}(S_{t})\big] \Big]\\ \stackrel{(a)}{=} \phantom{.}
        &\sum_{t=0}^{T-1}\Big[V^{\pi,h}_{t}(S_{t}) - \EXP\big[V^{\pi,h}_{t+1}(S_{t+1})\bigm|S_{t},\pi_{t}(S_{t})\big] \Big] + V^{\pi,h}_{T}(S_{T})-V^{\pi,h}_{T}(S_{T}) \\ \stackrel{(b)}{=} \phantom{.}
        &\sum_{t=0}^{T-1}\Big[V^{\pi,h}_{t+1}(S_{t+1}) - \EXP\big[V^{\pi,h}_{t+1}(S_{t+1})\bigm|S_{t},\pi_{t}(S_{t}) \big]\Big] + V^{\pi,h}_{0}(S_{0})-V^{\pi,h}_{T}(S_{T}) \\ \stackrel{(c)}{=} \phantom{.} 
        &\sum_{t=0}^{T-1}W^{\pi,h}_{t+1} + V^{\pi,h}_{0}(S_{0})-V^{\pi,h}_{T}(S_{T})\\ = \phantom{.} 
        &\sum_{t=1}^{T}W^{\pi,h}_{t} + V^{\pi,h}_{0}(S_{0})-V^{\pi,h}_{T}(S_{T}),
    \end{align*}
    where $(a)$ follows from adding and subtracting $V^{\pi,h}_{T}(S_{T})$, $(b)$ follows from re-arranging the terms in the summation,  and $(c)$ follows from the definition of $\{W_{t}^{\pi,h}\}_{t=0}^{h+1}$ in \eqref{eq:def_W}.
    \subsubsection{Proof of Theorem~\ref{thm:finite-h-non-asym}}
    Proof of this theorem follows from the martingale decomposition stated in Lemma~\ref{lem:finite_h_decomposition} and the concentration bounds stated in Theorem~\ref{thm:Azuma} and Theorem~\ref{thm:non_asym_LIL}. \\
    \textbf{Proof of Part 1} By Lemma~\ref{lem:finite_h_decomposition}, we have 
    \[
    R^{\pi,h}_{T}(\omega) = \sum_{t=1}^{T}W_{t}^{\pi,h} + V_{0}^{\pi,h}(S_{0}) - V_{T}^{\pi,h}(S_{T}).
    \]
    As a result, we have 
    \begin{equation}\label{eq:finite-h_mid_pf_1}
        \Big|R^{\pi,h}_{T}(\omega) - \big(V_{0}^{\pi,h}(S_{0}) - V_{T}^{\pi,h}(S_{T})\big) \Big| = \Big|\sum_{t=1}^{T}W_{t}^{\pi,h} \Big|.
    \end{equation}
    In order to upper-bound the term $\big|\sum_{t=1}^{T}W_{t}^{\pi,h} \big|$, we verify the conditions of Corollary~\ref{cor:azuma}. By \eqref{eq:def_k_h_finite}, we have 
    \[
    \big|W_{t}^{\pi,h} \big| = \big|V^{\pi,h}_{t}(S_{t}) - \EXP\big[V^{\pi,h}_{t}(S_{t})\bigm|S_{t-1},\pi_{t-1}(S_{t-1}) \big] \big|\leq K^{\pi,h}_{t} < \infty, \quad \forall t \in \{1,\ldots,T\}.
    \]
    As a result, MDS $\{W^{\pi,h}_{t}\}_{t=1}^{h+1}$ is a sequentially bounded MDS with respect to the sequence $\{K^{\pi,h}_{t}\}_{t=1}^{h+1}$. Therefore, Corollary~\ref{cor:azuma} implies that for any $\delta \in (0,1)$, with probability at least $1-\delta$, we have 
    \begin{align}\label{eq:finite-h_mid_pf_2}
        \Big|\sum_{t=1}^{T}W^{\pi,h}_{t} \Big| \leq \sqrt{2\sum_{t=1}^{T}(K^{\pi,h}_{t})^2\log\frac{2}{\delta}}\nonumber\\
        \stackrel{(a)}{=} \bar{K}^{\pi,h}_{T}\sqrt{2g^{\pi,h}(T)\log\frac{2}{\delta}},
    \end{align}
    where $(a)$ follows from \eqref{eq:def_g}. By combining \eqref{eq:finite-h_mid_pf_1} and \eqref{eq:finite-h_mid_pf_2}, with probability at least $1-\delta$, we have
    \begin{align}
        \Big|R^{\pi,h}_{T} - \big(V_{0}^{\pi,h}(S_{0}) - V_{T}^{\pi,h}(S_{T})\big) \Big| \leq \sqrt{2g^{\pi,h}(T)\log\frac{2}{\delta}}.
    \end{align}
    \textbf{Proof of Part 2:} Similar to the proof of Part 1, by Lemma~\ref{lem:finite_h_decomposition}, we have 
    \begin{equation}\label{eq:finite-h_mid_pf_3}
        \Big|R^{\pi,h}_{T} - \big(V_{0}^{\pi,h}(S_{0}) - V_{T}^{\pi,h}(S_{T})\big) \Big| = \Big|\sum_{t=1}^{T}W_{t}^{\pi,h} \Big|.
    \end{equation}
    Moreover, MDS $\{W_{t}^{\pi,h}\}_{t=1}^{h+1}$ is a sequentially bounded MDS with respect to the sequence $\{K^{\pi,h}_{t}\}_{t=1}^{h+1}$. Therefore, Theorem~\ref{thm:non_asym_LIL} implies that for any $\delta \in (0,1)$, if $g^{\pi,h}(h+1)\geq 173\log\frac{4}{\delta}$, define $T_{0}^{\pi,h}(\delta)$ to be 
        \begin{equation*}
        T_{0}^{\pi,h}(\delta) \coloneqq \min\{T'\geq1 : g^{\pi,h}(T') \geq 173\log\frac{4}{\delta}\}.    
        \end{equation*}
        Then with probability at least $1-\delta$, for all $T_{0}^{\pi,h}(\delta) \leq T \leq h+1$, we have
    \begin{align*}
        \Big| \sum_{t=1}^{T}W_{t}^{\pi,h} \Big| &\leq \sqrt{3\Big(\sum_{t=1}^{T}(K^{\pi,\gamma}_{t})^{2}\Big) \Big(2\log\log\big(\dfrac{3\sum_{t=1}^{T}(K^{\pi,\gamma}_{t})^2}{2\big|\sum_{t=1}^{T}W_{t}^{\pi,h}\big|}\big) +\log\frac{2}{\delta}\Big)}.
    \end{align*}
    Now there are two cases: either $|\sum_{t=1}^{T}W_{t}^{\pi,h}| \leq (\bar{K}^{\pi,h}_{T})^2$ or $|\sum_{t=1}^{T}W_{t}^{\pi,h}| \geq (\bar{K}^{\pi,\gamma}_{T})^2$. If $|\sum_{t=1}^{T}W_{t}^{\pi,h}| \geq (\bar{K}^{\pi,\gamma}_{T})^2$, we get:
    \begin{align*}
        \Big|\sum_{t=1}^{T}W_{t}^{\pi,h}\Big| &\leq \sqrt{3\Big(\sum_{t=1}^{T}(K^{\pi,h}_{t})^{2}\Big) \Big(2\log\log\big(\dfrac{3\sum_{t=1}^{T}(K^{\pi,h}_{t})^2}{2\big|\sum_{t=1}^{T}W_{t}^{\pi,h}\big|}\big) +\log\frac{2}{\delta}\Big)} \\&\leq
        \sqrt{3\Big(\sum_{t=1}^{T}(K^{\pi,h}_{t})^{2}\Big) \Big(2\log\log\big(\dfrac{3\sum_{t=1}^{T}(K^{\pi,h}_{t})^2}{2(\bar{K}^{\pi,h}_{T})^2}\big) +\log\frac{2}{\delta}\Big)}\\ &\stackrel{(a)}{=}
        \bar{K}^{\pi,h}_{T}\sqrt{3g^{\pi,h}(T) \Big(2\log\log\big(\dfrac{3}{2}g^{\pi,h}(T)\big) +\log\frac{2}{\delta}\Big)},
    \end{align*}
    where $(a)$ follows from the definition of $g^{\pi,h}(T)$.
    Otherwise, we have $|\sum_{t=1}^{T}W_{t}^{\pi,h}| \leq (\bar{K}^{\pi,\gamma}_{T})^2$. As a result, we can summarize these two cases as follows
    \begin{align}\label{eq:finite-h_mid_pf_4}
        \Big|\sum_{t=1}^{T}W_{t}^{\pi,h}\Big| \leq  
        \max\bigg\{ \bar{K}^{\pi,h}_{T} \sqrt{3g^{\pi,h}(T)\big(2\log\log \big(\dfrac{3}{2}g^{\pi,h}(T)\big) +\log\frac{2}{\delta}\big)},(\bar{K}^{\pi,h}_{T})^{2}\bigg\}.
    \end{align}
By combining \eqref{eq:finite-h_mid_pf_3}--\eqref{eq:finite-h_mid_pf_4}, with probability at least $1-\delta$, we have 
\begin{align}
    \Big|R^{\pi,h}_{T}(\omega) &- \big(V_{0}^{\pi,h}(S_{0}) - V_{T}^{\pi,h}(S_{T})\big) \Big| \notag  \\ &\leq \max\bigg\{ \bar{K}^{\pi,h}_{T} \sqrt{3g^{\pi,h}(T)\big(2\log\log (\frac{3}{2}g^{\pi,h}(T)) +\log\frac{2}{\delta}\big)},(\bar{K}^{\pi,h}_{T})^{2}\bigg\}.
\end{align}
    \subsection{Proof of Corollary~\ref{cor:finite-h-non-asym}}\label{pf:cor_finite-h-non-asym}
    \textbf{Proof of Part 1} By Lemma~\ref{lem:finite_h_decomposition}, we have
    \begin{equation*}
            R^{\pi,h}_{T}(\omega) = \sum_{t=1}^{T}W_{t}^{\pi,h} + V_{0}^{\pi,h}(S_{0}) - V_{T}^{\pi,h}(S_{T}). 
    \end{equation*}
    As a result, we have 
    \begin{equation}\label{eq:finite-h_mid_pf_5}
        \Big|R^{\pi,h}_{T}(\omega)-V_{0}^{\pi,h}(S_{0}) \Big| \stackrel{(a)}{\leq} \Big|\sum_{t=1}^{T}W_{t}^{\pi,h} \Big| + \Big| V_{T}^{\pi,h}(S_{T})\Big|,
    \end{equation}
    where $(a)$ follows from the triangle inequality. In the proof of Theorem~\ref{thm:finite-h-non-asym}, Part~1, we showed that with probability at least $1-\delta$, we have 
    \begin{align}\label{eq:finite-h_mid_pf_6}
        \Big|\sum_{t=1}^{T}W^{\pi,h}_{t} \Big| \leq \sqrt{2\sum_{t=1}^{T}(K^{\pi,h}_{t})^2\log\frac{2}{\delta}}\nonumber\\
        \stackrel{(b)}{\leq} \bar{K}^{\pi,h}_{T}\sqrt{2T\log\frac{2}{\delta}},
    \end{align}
    where $(b)$ follows by  $K_{t}^{\pi,h} \leq \bar{K}_{T}^{\pi,h}$, for all $t\leq T$.
    Moreover, by definition,  we have 
    \begin{equation}\label{eq:finite-h_mid_pf_7}
    V_{T}^{\pi,h}(S_{T})\leq  \bar{H}_{T}^{\pi,h}, \quad \forall t \leq T.    
    \end{equation}
    By combining \eqref{eq:finite-h_mid_pf_5}--\eqref{eq:finite-h_mid_pf_7}, with probability at least $1-\delta$, we have 
    \[
    \Big|R^{\pi,h}_{T}(\omega)-V_{0}^{\pi,h}(S_{0}) \Big| \leq \bar{K}^{\pi,h}_{T}\sqrt{2T\log\frac{2}{\delta}} + \bar{H}_{T}^{\pi,h}.
    \]
    \textbf{Proof of Part 2:} Similar to the proof of Part 1, by Lemma~\ref{lem:finite_h_decomposition}, we have 
    \begin{equation}\label{eq:finite-h_mid_pf_8}
       \Big|R^{\pi,h}_{T}(\omega)-V_{0}^{\pi,h}(S_{0}) \Big| \leq \Big|\sum_{t=1}^{T}W_{t}^{\pi,h} \Big| + \Big| V_{T}^{\pi,h}(S_{T})\Big|, 
    \end{equation}
    Moreover, we have 
    \begin{equation}\label{eq:finite-h_mid_pf_9}
    V_{T}^{\pi,h}(S_{T})\leq  \bar{H}_{T}^{\pi,h}.    
    \end{equation}
    In addition, from proof of Theorem~\ref{thm:finite-h-non-asym}, Part~2, we have for any $\delta \in (0,1)$, for all $T\geq T_{0}(\delta) \coloneqq \min \{T\geq1 : g^{\pi,h}(T) \geq 173\log\frac{4}{\delta}\}$, with probability at least $1-\delta$, we have 
    \begin{align}\label{eq:finite-h_mid_pf_10}
    \Big|\sum_{t=1}^{T}W_{t}^{\pi,h}\Big| &\leq  
    \max\bigg\{ \bar{K}^{\pi,h}_{T} \sqrt{3g^{\pi,h}(T)\big(2\log\log \big(\dfrac{3}{2}g^{\pi,h}(T)\big) +\log\frac{2}{\delta}\big)},(\bar{K}^{\pi,h}_{T})^{2}\bigg\}\nonumber \\
    &\stackrel{(c)}{\leq} \max\bigg\{ \bar{K}^{\pi,h}_{T} \sqrt{3T\big(2\log\log \big(\dfrac{3T}{2}\big) +\log\frac{2}{\delta}\big)},(\bar{K}^{\pi,h}_{T})^{2}\bigg\},
    \end{align}
    where $(c)$ follows from the fact that $g^{\pi,h}(T) \leq T$.
    By combining \eqref{eq:finite-h_mid_pf_8}--\eqref{eq:finite-h_mid_pf_10}, with probability at least $1-\delta$, we have 
    \[
    \Big|R^{\pi,h}_{T}(\omega)-V_{0}^{\pi,h}(S_{0}) \Big| \leq \max\bigg\{ \bar{K}^{\pi,h}_{T} \sqrt{3T\big(2\log\log \big(\dfrac{3T}{2}\big) +\log\frac{2}{\delta}\big)},(\bar{K}^{\pi,h}_{T})^{2}\bigg\} + \bar{H}_{T}^{\pi,h}.
    \]
    \subsection{Proof of Corollary~\ref{cor:finite-h-two-policies}}\label{app:pf_cor:finite-h-two-policies}
    \subsubsection{Proof of Part 1}
Consider two policies $\pi_{1},\pi_{2} \in \Pi_{\SD}$. Let $\{S_{t}^{\pi_{1}}\}_{t\geq0}$ and $\{S_{t}^{\pi_{2}}\}_{t\geq0}$ denote the random sequence of states encountered by following policies $\pi_{1}$ and $\pi_{2}$. We have
\begin{align}\label{eq:mid_pf_dis_two_policy_1_h_finite}
    \Big|R^{\pi_{1},h}_{T}-R^{\pi_{2},h}_{T} \Big| \stackrel{(a)}{=}&\Big|R^{\pi_{1},h}_{T} -\big[V^{\pi_{1},h}_{0}(S_{0}^{\pi_{1}}) - V^{\pi_{1},h}_{T}(S_{T}^{\pi_1}) \big] + \big[V^{\pi_{1},h}_{0}(S_{0}^{\pi_1}) - V^{\pi_{1},h}_{T}(S_{T}^{\pi_1}) \big]\notag\\ - &\big[V^{\pi_{2},h}_{0}(S_{0}^{\pi_2}) - V^{\pi_{2},h}_{T}(S_{T}^{\pi_2}) \big] + \big[V^{\pi_{2},h}_{0}(S_{0}^{\pi_2}) - V^{\pi_{2},h}_{T}(S_{T}^{\pi_2}) \big] - R^{\pi_{2},h}_{T} \Big| \notag \notag\\\stackrel{(b)}{\leq} &\Big|R^{\pi_{1},h}_{T} - \big[V^{\pi_{1},h}_{0}(S_{0}^{\pi_1}) - V^{\pi_{1},h}_{T}(S_{T}^{\pi_1}) \big]\Big| + \Big|\big[V^{\pi_{2},h}_{0}(S_{0}^{\pi_2}) - V^{\pi_{2},h}_{T}(S_{T}^{\pi_2}) \big] - R^{\pi_{2},h}_{T} \Big|\notag \\
    + &\Big|\big[V^{\pi_{1}}_{0}(S_{0}^{\pi_1}) - V^{\pi_{1}}_{T}(S_{T}^{\pi_1}) \big] - \big[V^{\pi_{2}}_{0}(S_{0}^{\pi_2}) - V^{\pi_{2}}_{T}(S_{T}^{\pi_2}) \big] \Big|,
\end{align}
where $(a)$ follows by adding and subtracting $\big[V^{\pi_{1},h}_{0}(S_{0}^{\pi_1}) - V^{\pi_{1},h}_{T}(S_{T}^{\pi_1}) \big]$ and $\big[V^{\pi_{2},h}_{0}(S_{0}^{\pi_2}) - V^{\pi_{2},h}_{T}(S_{T}^{\pi_2}) \big]$ and $(b)$ follows from the triangle inequality. Similarly, we have 
\begin{align}\label{eq:mid_pf_dis_two_policy_2_h_finite}
    &\Big|\big[V^{\pi_{1},h}_{0}(S_{0}^{\pi_1}) - V^{\pi_{1},h}_{T}(S_{T}^{\pi_1}) \big]-\big[V^{\pi_{2},h}_{0}(S_{0}^{\pi_2}) - V^{\pi_{2},h}_{T}(S_{T}^{\pi_2}) \big] \Big| \stackrel{(a)}{=}\notag\\&\Big|\big[V^{\pi_{1}}_{0}(S_{0}^{\pi_1}) - V^{\pi_{1}}_{T}(S_{T}^{\pi_1}) \big] -R^{\pi_{1},h}_{T}  + R^{\pi_{1},h}_{T} - R^{\pi_{2},h}_{T} + R^{\pi_{2},T}_{T}   - \big[V^{\pi_{2},h}_{0}(S_{0}^{\pi_2}) - V^{\pi_{2},h}_{T}(S_{T}^{\pi_2}) \big]\Big| \notag\\\stackrel{(b)}{\leq} &\Big|R^{\pi_{1},h}_{T} - \big[V^{\pi_{1}}_{0}(S_{0}^{\pi_1}) - V^{\pi_{1}}_{T}(S_{T}^{\pi_1}) \big]\Big| + \Big|R^{\pi_{2},h}_{T} - \big[V^{\pi_{2}}_{0}(S_{0}^{\pi_2}) - V^{\pi_{2}}_{T}(S_{T}^{\pi_2}) \big]\Big| \notag \\
    +&\Big| R^{\pi_{1},h}_{T} - R^{\pi_{2},h}_{T} \Big|,
\end{align}
where $(a)$ follows by adding and subtracting $R^{\pi_{1},h}_{T}$ and $R^{\pi_{2},h}_{T}$ and $(b)$ follows from the triangle inequality. \eqref{eq:mid_pf_dis_two_policy_1_h_finite}--\eqref{eq:mid_pf_dis_two_policy_2_h_finite} imply that 
\begin{align}\label{eq:mid_pf_dis_two_policy_3_h_finite}
    &\Big|\big|R^{\pi_{1},h}_{T}-R^{\pi_{2},h}_{T}\big| - \big|\big[V^{\pi_{1},h}_{0}(S_{0}^{\pi_1}) - V^{\pi_{1},h}_{T}(S_{T}^{\pi_1}) \big]-\big[V^{\pi_{2},h}_{0}(S_{0}^{\pi_2}) - V^{\pi_{2},h}_{T}(S_{T}^{\pi_2}) \big] \big| \Big| \notag  \\ \leq &\Big|R^{\pi_{1},h}_{T} - \big[V^{\pi_{1},h}_{0}(S_{0}^{\pi_1}) - V^{\pi_{1},h}_{T}(S_{T}^{\pi_1}) \big]\Big| + \Big|R^{\pi_{2},h}_{T} - \big[V^{\pi_{2},h}_{0}(S_{0}^{\pi_2}) - V^{\pi_{2},h}_{\gamma}(S_{T}^{\pi_2}) \big]\Big|.
\end{align}
By Theorem~\ref{thm:finite-h-non-asym}, we know that for any $\delta_{1} \in (0,1)$, with probability at least $1-\delta_{1}$, we have
\begin{equation}\label{eq:mid_pf_dis_two_policy_4_h_finite}
    \Big|R^{\pi_{1},h}_{T} - \big(V_{0}^{\pi_{1},h}(S_{0}) - V_{T}^{\pi_{1},h}(S_{T})\big) \Big| \leq \bar{K}_{T}^{\pi_{1},h}\sqrt{2g^{\pi_{1},h}(T)\log\frac{2}{\delta_{1}}}.
\end{equation}
Similarly, we have that for any $\delta_{2} \in (0,1)$, with probability at least $1-\delta_{2}$, we have
\begin{equation}\label{eq:mid_pf_dis_two_policy_5_h_finite}
    \Big|R^{\pi_{2},h}_{T} - \big(V_{0}^{\pi_{2},h}(S_{0}) - V_{T}^{\pi_{2},h}(S_{T})\big) \Big| \leq \bar{K}_{T}^{\pi_{2},h}\sqrt{2g^{\pi_{2},h}(T)\log\frac{2}{\delta_{2}}}.
\end{equation}
As a result, by applying Lemma~\ref{lem:delta} and \eqref{eq:mid_pf_dis_two_policy_3_h_finite}--\eqref{eq:mid_pf_dis_two_policy_5_h_finite}, for any $\delta \in (0,1)$, with probability at least $1-\delta$, we have
\begin{align*}
    &\phantom{\leq}\Big|\big|R^{\pi_{1},h}_{T}-R^{\pi_{2},h}_{T}\big| - \big|\big[V^{\pi_{1},h}_{0}(S_{0}^{\pi_1}) - V^{\pi_{1},h}_{T}(S_{T}^{\pi_1}) \big]-\big[V^{\pi_{2},h}_{0}(S_{0}^{\pi_2}) - V^{\pi_{2},h}_{T}(S_{T}^{\pi_2}) \big] \big| \Big| \\&\leq \Big|R^{\pi_{1},h}_{T} - \big[V^{\pi_{1},h}_{0}(S_{0}^{\pi_1}) - V^{\pi_{1},h}_{T}(S_{T}^{\pi_1}) \big]\Big| + \Big|R^{\pi_{2},h}_{T} - \big[V^{\pi_{2},h}_{0}(S_{0}^{\pi_2}) - V^{\pi_{2},h}_{T}(S_{T}^{\pi_2}) \big]\Big| \\&\leq \bar{K}_{T}^{\pi_{1},h}\sqrt{2g^{\pi_{1},h}(T)\log\frac{4}{\delta}} + \bar{K}_{T}^{\pi_{2},h}\sqrt{2g^{\pi_{2},h}(T)\log\frac{4}{\delta}}.
\end{align*}
\subsection{Proof of Part 2}
As we showed in the proof of part~1, for any two policies $\pi_{1},\pi_{2} \in \Pi_{\FD}$, we have
\begin{align}
    &\Big|\big|R^{\pi_{1},h}_{T}-R^{\pi_{2},h}_{T}\big| - \big|\big[V^{\pi_{1},h}_{0}(S_{0}^{\pi_1}) - V^{\pi_{1},h}_{T}(S_{T}^{\pi_1}) \big]-\big[V^{\pi_{2},h}_{0}(S_{0}^{\pi_2}) - V^{\pi_{2},h}_{T}(S_{T}^{\pi_2}) \big] \big| \Big| \notag  \\ \leq &\Big|R^{\pi_{1},h}_{T} - \big[V^{\pi_{1},h}_{0}(S_{0}^{\pi_1}) - V^{\pi_{1},h}_{T}(S_{T}^{\pi_1}) \big]\Big| + \Big|R^{\pi_{2},h}_{T} - \big[V^{\pi_{2},h}_{0}(S_{0}^{\pi_2}) - V^{\pi_{2},h}_{\gamma}(S_{T}^{\pi_2}) \big]\Big|.
\end{align}
By Corollary~\ref{cor:finite-h-two-policies}, for any $\delta_{1} \in (0,1)$, if $g^{\pi_{1},h}(h)\geq 173\log\frac{4}{\delta_{1}}$, let 
    \begin{equation}
    T_{0}^{\pi,h}(\delta_{1}) \coloneqq \min\Big\{T'\geq1 : g^{\pi,h}(T') \geq 173\log\frac{4}{\delta_{1}}\Big\}.    
    \end{equation}
Then with probability at least $1-\delta_{1}$, for all $T_{0}^{\pi_{1},h}(\delta_{1}) \leq T \leq h+1$, we have 
\begin{align}
    \Big|R^{\pi_{1},h}_{T} &- \big(V_{0}^{\pi_{1},h}(S_{0}) - V_{T}^{\pi_{1},h}(S_{T})\big) \Big| \notag  \\ &\leq \max\bigg\{ \bar{K}^{\pi_{1},h}_{T} \sqrt{3g^{\pi_{1},h}(T)\big(2\log\log (\frac{3}{2}g^{\pi_{1},h}(T)) +\log\frac{2}{\delta_{1}}\big)},(\bar{K}^{\pi_{1},h}_{T})^{2}\bigg\}.
\end{align}
Similarly, for any $\delta_{2} \in (0,1)$, if $g^{\pi_{2},h}(h)\geq 173\log\frac{4}{\delta_{2}}$, with probability at least $1-\delta_{2}$, for all $T_{0}^{\pi,h}(\delta_{2}) \leq T \leq h+1$,  we have
\begin{align}
    \Big|R^{\pi_{2},h}_{T} &- \big(V_{0}^{\pi_{2},h}(S_{0}) - V_{T}^{\pi_{2},h}(S_{T})\big) \Big| \notag  \\ &\leq \max\bigg\{ \bar{K}^{\pi_{2},h}_{T} \sqrt{3g^{\pi_{2},h}(T)\big(2\log\log (\frac{3}{2}g^{\pi_{2},h}(T)) +\log\frac{2}{\delta_{2}}\big)},(\bar{K}^{\pi_{2},h}_{T})^{2}\bigg\}.
\end{align}
As a result, by applying Lemma~\ref{lem:delta}, for any $\delta \in (0,1)$, if $\min\big\{g^{\pi_{1},h}(h),g^{\pi_{2},h}(h)\big\} \geq 173\log\frac{8}{\delta}$, let 
\[
T_{0}(\delta) \coloneqq \max \{T_{0}^{\pi_{1},h}(\frac{8}{\delta}),T_{0}^{\pi_{1},h}(\frac{8}{\delta})\}.
\]
Then, with probability at least $1-\delta$, for all $T_{0}(\delta) \leq T \leq h+1$, we have
\begin{align}
    &\Big|\big|R^{\pi_{1},h}_{T}-R^{\pi_{2},h}_{T}\big| - \big|\big[V^{\pi_{1},h}_{0}(S_{0}^{\pi_1}) - V^{\pi_{1},h}_{T}(S_{T}^{\pi_1}) \big]-\big[V^{\pi_{2},h}_{0}(S_{0}^{\pi_2}) - V^{\pi_{2},h}_{T}(S_{T}^{\pi_2}) \big] \big| \Big| \nonumber \\\leq &\Big|R^{\pi_{1},h}_{T} - \big[V^{\pi_{1},h}_{0}(S_{0}^{\pi_1}) - V^{\pi_{1},h}_{T}(S_{T}^{\pi_1}) \big]\Big| + \Big|R^{\pi_{2},h}_{T} - \big[V^{\pi_{2},h}_{0}(S_{0}^{\pi_2}) - V^{\pi_{2},h}_{T}(S_{T}^{\pi_2}) \big]\Big| \nonumber \\\leq 
    &\max\bigg\{ \bar{K}^{\pi_{1},h}_{T} \sqrt{3g^{\pi_{1},h}(T)\big(2\log\log (\frac{3}{2}g^{\pi_{1},h}(T)) +\log\frac{4}{\delta}\big)},(\bar{K}^{\pi_{1},h}_{T})^{2}\bigg\} \nonumber\\
    + &\max\bigg\{ \bar{K}^{\pi_{2},h}_{T} \sqrt{3g^{\pi_{2},h}(T)\big(2\log\log (\frac{3}{2}g^{\pi_{2},h}(T)) +\log\frac{4}{\delta}\big)},(\bar{K}^{\pi_{2},h}_{T})^{2}\bigg\}.
\end{align}

\section{Proof of Main Results for Random Reward Setup}
\subsection{Preliminary Results}
\subsubsection{Reward Martingale Decomposition}
To simplify the notation, we define following martingale difference sequence. 
\begin{definition}
    Let filtration $\mathcal{F} =\{\mathcal{F}_{t}\}_{t\geq0}$ be defined as $\mathcal{F}_{t} \coloneqq \sigma(S_{0:t},A_{0:t}).$ For any policy $\pi \in \tilde{\Pi}_{\AC}$, we define the sequence $\{\tilde{W}^{\pi}_{t}\}_{t\geq1}$ as follows
    \begin{equation}\label{eq:def_tilde_w}
    \tilde{W}_{t}^{\pi} \coloneqq \tilde{r}(S_{t-1},A_{t-1},E_{t-1}) - \EXP\big[\tilde{r}(S_{t-1},A_{t-1},E_{t-1})\big|S_{t-1},A_{t-1}\!=\!\pi(S_{t-1})\big], \quad \forall t\geq1,   
    \end{equation}
    where $\{S_{t}\}_{t\geq0}$ denotes the random sequence of states encountered along the current sample path. 
\end{definition}
The sequence $\{\tilde{W}^{\pi}_{t}\}_{t\geq1}$ is an MDS with respect to the filtration $\mathcal{F} = \{\mathcal{F}_{t}\}_{t\geq0}$ since
\begin{align}\label{eq:tilde_W_martingale}
    \EXP\Big[\tilde{W}_{t}^{\pi}\big| \mathcal{F}_{t-1}\Big]&=\EXP\Big[\tilde{r}(S_{t-1},A_{t-1},E_{t-1}) - \EXP\big[\tilde{r}(S_{t-1},A_{t-1},E_{t-1})\big|\mathcal{F}_{t-1}\big] \Big| \mathcal{F}_{t-1}\Big] \notag \\
    &=\EXP\Big[\tilde{r}(S_{t-1},A_{t-1},E_{t-1})\big|S_{t-1},A_{t-1} \!=\! \pi(S_{t-1})\big]\notag \\
    &-\EXP\Big[\tilde{r}(S_{t-1},A_{t-1},E_{t-1})\big|S_{t-1},A_{t-1} \!=\! \pi(S_{t-1})\big] = 0.
\end{align}
Recall the definitions of processes $\tilde{R}^{\pi}_{T}$ and $R^{\pi}_{T}$
\begin{align}\label{eq:recal_R_tilde_R}
    \tilde{R}^{\pi}_{T} &= \sum_{t=0}^{T-1}\tilde{r}(S_{t},A_{t},E_{t}),\\
    R^{\pi}_{T} &= \sum_{t=0}^{T-1}\EXP\big[\tilde{r}(S_{t},A_{t},E_{t})\big|S_{t},A_{t}=\pi(S_{t})\big].
\end{align}
By the definition of $\{\tilde{W}_{t}^{\pi}\}_{t\geq1}$ in \eqref{eq:def_tilde_w} we have
\begin{equation}\label{eq:mid_proof_ref_w}
    \tilde{R}^{\pi}_{T}   = R^{\pi}_{T} +  \sum_{t=1}^{T} \tilde{W}^{\pi}_{t},
\end{equation}    
where $\{\tilde{W}^{\pi}_{t}\}_{t\geq1}$ is an MDS with respect to the filtration $\mathcal{F} = \{\mathcal{F}_{t}\}_{t\geq0}$ by \eqref{eq:tilde_W_martingale}.
The following theorem establishes the concentration of the process $R^{\pi}_{T}$ around the quantity $T\tilde{J}^{\pi} - \big(\tilde{V}^{\pi}(S_{T})-\tilde{V}^{\pi}(S_{0})\big)$.
\begin{theorem}\label{thm:non-asymp_exp_stoch_reward}
    For any policy $\pi \in \tilde{\Pi}_{\AC}$, the following upper-bounds hold:
    \begin{enumerate}
        \item For any $\delta \in (0,1)$, with probability at least $1-\delta$, we have:
        \begin{equation}
            \Big|R^{\pi}_{T} - T\tilde{J}^{\pi} - \big(\tilde{V}^{\pi}(S_{0}) - \tilde{V}^{\pi}(S_{T}) \big) \Big| \leq \tilde{K}^{\pi}\sqrt{2T\log\frac{2}{\delta}}.
        \end{equation}
        \item For any $\delta \in (0,1)$, for all $T \geq T_{0}(\delta) \coloneqq \Bigl\lceil \dfrac{173}{\tilde{K}^\pi}\log\dfrac{4}{\delta}\Big\rceil$, with probability at least $1-\delta$, we have 
        \begin{align}
            \Big| R^{\pi}_{T} - T\tilde{J}^{\pi} - \big(\tilde{V}^{\pi}(S_{0}) - \tilde{V}^{\pi}(S_{T}) \big) \Big| 
            \leq \max \Big\{\tilde{K}^{\pi}\sqrt{3T\Big(2\log\log\dfrac{3T}{2} + \log\dfrac{2}{\delta}\Big)},\big(\tilde{K}^{\pi}\big)^2 \Big\}.
        \end{align}
    \end{enumerate}
\end{theorem}
\begin{proof}
    This theorem follows by applying the result of Theorem~\ref{thm:finite_return_meta_new} to the reduced model of $\tilde{\mathcal{M}}$ defined in Def.~\ref{def:reduced_model}.
    Let $\mathcal{M} = (P,r)$ denote the reduced model of $\tilde{\mathcal{M}} = (P,\tilde{r})$ with $r(s,a)$ defined in \eqref{eq:r_reduction}. We can rewrite the process $R^{\pi}_{T}$ as the cumulative reward process associated with the reduced model $\mathcal{M} = (P,r)$, i.e., 
    \[
    R^{\pi}_{T} =  \sum_{t=0}^{T-1}\EXP\big[\tilde{r}(S_{t},A_{t},E_{t})\big|S_{t},A_{t}=\pi(S_{t})\big] = \sum_{t=0}^{T-1}r(S_{t},A_{t}) .
    \]
    The result of this theorem follows by applying the result of Theorem~\ref{thm:finite_return_meta_new} on the cumulative reward process $R^{\pi}_{T}$ and recalling that quantities $\tilde{J}^{\pi}$, $\tilde{V}^{\pi}$, and $\tilde{K}^{\pi}$ are defined based on the reduced model $M$. 
\end{proof}

\subsection{Proof of Theorem~\ref{thm:non-asymp_stoch_reward}}\label{sec:pf_thm:non-asymp_stoch_reward}
\subsubsection{Proof of Part 1}
For any policy $\pi \in \tilde{\Pi}_{\AC}$, we have 
\begin{align}\label{eq:mid_proof_sum_random_reward}
    \big|\tilde{R}^{\pi}_{T} - T\tilde{J}^{\pi} - \big(\tilde{V}^{\pi}(S_{0})-\tilde{V}^{\pi}(S_{T})\big) \big| = \big| R^{\pi}_{T} - T\tilde{J}^{\pi} - \big(\tilde{V}^{\pi}(S_{0})-\tilde{V}^{\pi}(S_{T})\big) + \tilde{R}^{\pi}_{T} - R^{\pi}_{T}\big|.
\end{align}
By~\eqref{eq:mid_proof_ref_w}, we have
\[
\Big|\tilde{R}^{\pi}_{T} - R^{\pi}_{T} \Big| = \Big|\sum_{t=1}^{T}\tilde{W}_{t}^{\pi}\Big|.
\]
Moreover by~\eqref{eq:k_r_tilde}, $\{\tilde{W}_{t}^{\pi}\}_{t\geq0}$ is a uniformly bounded MDS with respect to the constant $\tilde{K}_{r}^{\pi}$. Therefore, Corollary~\ref{cor:azuma} implies that for any $\delta_{1} \in (0,1)$, with probability at least $1-\delta_{1}$, we have
\begin{equation}\label{eq:mid_proof_random_reward_1}
     \Big|\sum_{t=1}^{T}\tilde{W}_{t}^{\pi}\Big| \leq \tilde{K}^{\pi}_{r}\sqrt{2T\log\frac{2}{\delta_{1}}}.    
\end{equation}
Moreover, Theorem~\ref{thm:non-asymp_exp_stoch_reward}, Part 1, implies that for any $\delta_{2} \in (0,1)$, with probability at least $1-\delta_{2}$, we have 
\begin{equation}\label{eq:mid_proof_random_reward_2}
    \Big|R^{\pi}_{T} - T\tilde{J}^{\pi} - \Big(\tilde{V}^{\pi}(S_{0}) - \tilde{V}^{\pi}(S_{T}) \Big) \Big| \leq \tilde{K}^{\pi}\sqrt{2T\log\frac{2}{\delta_{2}}}.
\end{equation}
As a result, by combining \eqref{eq:mid_proof_sum_random_reward}, \eqref{eq:mid_proof_random_reward_1}, and \eqref{eq:mid_proof_random_reward_2} and applying Lemma~\ref{lem:delta}, for any $\delta \in (0,1)$, with probability at least $1-\delta$, we have 
\begin{align}
    \Big|R^{\pi}_{T} - T\tilde{J}^{\pi} &- \Big(\tilde{V}^{\pi}(S_{0}) - \tilde{V}^{\pi}(S_{T}) \Big) \Big| \\&\leq \tilde{K}^{\pi}\sqrt{2T\log\frac{4}{\delta}} + \tilde{K}^{\pi}_{r}\sqrt{2T\log\frac{4}{\delta}}.
\end{align}

\subsubsection{Proof of Part 2}
Similar to the proof of Part~1, for any policy $\pi \in \tilde{\Pi}_{\AC}$, we have 
\begin{align}\label{eq:mid_proof_sum_random_reward2}
    \big|\tilde{R}^{\pi}_{T} - T\tilde{J}^{\pi} - \big(\tilde{V}^{\pi}(S_{0})-\tilde{V}^{\pi}(S_{T})\big) \big| = \big| R^{\pi}_{T} - T\tilde{J}^{\pi} - \big(\tilde{V}^{\pi}(S_{0})-\tilde{V}^{\pi}(S_{T})\big) + \tilde{R}^{\pi}_{T} - R^{\pi}_{T}\big|.
\end{align}
By \eqref{eq:mid_proof_ref_w}, we have
\[
\Big|\tilde{R}^{\pi}_{T} - R^{\pi}_{T} \Big| = \Big|\sum_{t=1}^{T}\tilde{W}_{t}^{\pi}\Big|.
\]
Moreover by~\eqref{eq:k_r_tilde}, $\{\tilde{W}_{t}^{\pi}\}_{t\geq0}$ is a uniformly bounded MDS with respect to the constant $\tilde{K}_{r}^{\pi}$. Therefore, Corollary~\ref{cor:finite_lil} implies that for any $\delta_{1} \in (0,1)$, for all $T \geq T_{0}(\delta_{1}) \coloneqq \Bigl\lceil \dfrac{173}{\tilde{K}^\pi_{r}}\log\dfrac{4}{\delta_{1}}\Big\rceil$, with probability at least $1-\delta_{1}$, we have 
\begin{equation}\label{eq:mid_proof_random_reward_3}
    \Big|\sum_{t=1}^{T}\tilde{W}_{t}^{\pi}\Big| \leq \max \Big\{\tilde{K}^{\pi}_{r}\sqrt{3T\Big(2\log\log\dfrac{3T}{2} + \log\dfrac{2}{\delta_{1}}\Big)},\big(\tilde{K}^{\pi}_{r}\big)^2 \Big\}.
\end{equation}
Moreover, Theorem~\ref{thm:non-asymp_exp_stoch_reward}, Part 2, implies that for any $\delta_{2} \in (0,1)$, for all $T \geq T_{0}(\delta_{2}) \coloneqq \Bigl\lceil \dfrac{173}{\tilde{K}^\pi}\log\dfrac{4}{\delta_{2}}\Big\rceil$, with probability at least $1-\delta_{2}$, we have 
\begin{equation}\label{eq:mid_proof_random_reward_4}
    \Big| R^{\pi}_{T} - T\tilde{J}^{\pi} - \big(\tilde{V}^{\pi}(S_{0}) - \tilde{V}^{\pi}(S_{T}) \big) \Big| \leq \max \Big\{\tilde{K}^{\pi}\sqrt{3T\Big(2\log\log\dfrac{3T}{2} + \log\dfrac{2}{\delta_{2}}\Big)},\big(\tilde{K}^{\pi}\big)^2 \Big\}.
\end{equation}
As a result, by combining \eqref{eq:mid_proof_sum_random_reward2}, \eqref{eq:mid_proof_random_reward_3}, and \eqref{eq:mid_proof_random_reward_4} and applying Lemma~\ref{lem:delta}, for any $\delta \in (0,1)$, for all $T\geq T_{0}(\delta) \coloneqq \max\Big\{\Bigl\lceil \dfrac{173}{\tilde{K}^{\pi}}\log\dfrac{8}{\delta}\Big\rceil , \Bigl\lceil\dfrac{173}{\tilde{K}^{\pi}_{r}}\log\dfrac{8}{\delta}\Big\rceil \Big\}$, with probability at least $1-\delta$, we have
\begin{align}
    \Big| \tilde{R}^{\pi}_{T} - T\tilde{J}^{\pi} &- \big(\tilde{V}^{\pi}(S_{0}) - \tilde{V}^{\pi}(S_{T}) \big) \Big| 
    \\&\leq 
    \max\Big\{\tilde{K}^{\pi}\sqrt{3T\Big(2\log\log\frac{3T}{2} + \log\frac{4}{\delta}\Big)}, (\tilde{K}^{\pi})^2\Big\}\\
    &+\max\Big\{\tilde{K}^{\pi}_{r}\sqrt{3T\Big(2\log\log\frac{3T}{2} + \log\frac{4}{\delta}\Big)}, (\tilde{K}^{\pi}_{r})^2\Big\}.
\end{align}

\section{Miscellaneous Theorems}
\subsection{Slutsky's Theorem}
\begin{theorem}[{see~\citep[Theorem~7.7.1]{ash2000probability}}]\label{thm:slutsky}
    If $X_{t} \xrightarrow[]{(d)} X$ and $Y_{t} \xrightarrow[]{(d)} c$, where $c\in \reals$ (equivalently $Y_{t} \xrightarrow[]{(P)} c$) then we have 
    \begin{enumerate}
        \item $X_{t} + Y_{t} \xrightarrow[]{(d)} X + c $.
        \item $X_{t}Y_{t} \xrightarrow[]{(d)} cX $.
        \item $\dfrac{X_{t}}{Y_{t}} \xrightarrow[]{(d)} \dfrac{X}{c}$, if $c\not=0$.
    \end{enumerate}
\end{theorem}
\begin{remark}
    Since convergence in the almost-sure sense implies convergence in probability, same results hold when $Y_{t} \xrightarrow[]{(a.s.)} c$.
\end{remark}

\end{document}